\documentclass{article}
\usepackage[latin9]{inputenc}
\usepackage{color}
\usepackage{float}
\usepackage{wrapfig}
\usepackage{amsthm}
\usepackage{amsmath}
\usepackage{amssymb}
\usepackage{graphicx}
\usepackage{esint}
\usepackage[numbers]{natbib}

\makeatletter

\providecommand{\tabularnewline}{\\}
\newcommand{\lyxdot}{.}

\floatstyle{ruled}
\newfloat{algorithm}{tbp}{loa}
\providecommand{\algorithmname}{Algorithm}
\floatname{algorithm}{\protect\algorithmname}

\theoremstyle{plain}
\newtheorem{thm}{\protect\theoremname}
\theoremstyle{definition}
\newtheorem{defn}[thm]{\protect\definitionname}
\theoremstyle{plain}
\newtheorem{lem}[thm]{\protect\lemmaname}
\ifx\proof\undefined
\newenvironment{proof}[1][\protect\proofname]{\par
\normalfont\topsep6\p@\@plus6\p@\relax
\trivlist
\itemindent\parindent
\item[\hskip\labelsep\scshape #1]\ignorespaces
}{%
\endtrivlist\@endpefalse
}
\providecommand{\proofname}{Proof}
\fi
\theoremstyle{plain}
\newtheorem{cor}[thm]{\protect\corollaryname}

\usepackage{subfigure}

\usepackage{algorithm}\usepackage{algorithmic}

\usepackage{nips14submit_e,times}
\usepackage{url}
\nipsfinalcopy

\usepackage{amsthm}% allow blackboard bold (aka N,R,Q sets)
\usepackage{psfrag}
\usepackage{afterpage}

\author{
Xianghang Liu\\
NICTA, The University of New South Wales\\
\texttt{xianghang.liu@nicta.com.au}
\And
Justin Domke\\
NICTA, The Australian National University\\
\texttt{justin.domke@nicta.com.au}
}

\title{Projecting Markov Random Field Parameters for Fast Mixing}

\makeatother

\providecommand{\corollaryname}{Corollary}
\providecommand{\definitionname}{Definition}
\providecommand{\lemmaname}{Lemma}
\providecommand{\theoremname}{Theorem}

\begin{document}
\maketitle
\begin{abstract}
Markov chain Monte Carlo (MCMC) algorithms are simple and extremely
powerful techniques to sample from almost arbitrary distributions.
The flaw in practice is that it can take a large and/or unknown amount
of time to converge to the stationary distribution. This paper gives
sufficient conditions to guarantee that univariate Gibbs sampling
on Markov Random Fields (MRFs) will be fast mixing, in a precise sense.
Further, an algorithm is given to project onto this set of fast-mixing
parameters in the Euclidean norm. Following recent work, we give an
example use of this to project in various divergence measures, comparing
univariate marginals obtained by sampling after projection to common
variational methods and Gibbs sampling on the original parameters.
\end{abstract}
\label{submission}

\section{Introduction}

Exact inference in Markov Random Fields (MRFs) is generally intractable,
motivating approximate algorithms. There are two main classes of approximate
inference algorithms: variational methods and Markov chain Monte Carlo
(MCMC) algorithms \citep{Wainwright2008GraphicalModelsExponential}.

Among variational methods, mean-field approximations \citep{Koller2009ProbabilisticGraphicalModels:}
are based on a ``tractable'' family of distributions, such as the
fully-factorized distributions. Inference finds a distribution in
the tractable set to minimize the KL-divergence from the true distribution.
Other methods, such as loopy belief propagation (LBP), generalized
belief propagation \citep{Yedidia2005ConstructingFreeEnergy} and
expectation propagation \citep{Minka2001ExpectationPropagationApproximate}
use a less restricted family of target distributions, but approximate
the KL-divergence. Variational methods are typically fast, and often
produce high-quality approximations. However, when the variational
approximations are poor, estimates can be correspondingly worse.

MCMC strategies, such as Gibbs sampling, simulate a Markov chain whose
stationary distribution is the target distribution. Inference queries
are then answered by the samples drawn from the Markov chain. In principle,
MCMC will be arbitrarily accurate if run long enough. The principal
difficulty is that the time for the Markov chain to converge to its
stationary distribution, or the ``mixing time'', can be exponential
in the number of variables.

This paper is inspired by a recent hybrid approach for Ising models
\citep{Domke2013ProjectingIsingModel}. This approach minimizes the
divergence from the true distribution to one in a tractable family.
However, the tractable family is a ``fast mixing'' family where
Gibbs sampling is guaranteed to quickly converge to the stationary
distribution. They observe that an Ising model will be fast mixing
if the spectral norm of a matrix containing the absolute values of
all interactions strengths is controlled. An algorithm projects onto
this fast mixing parameter set in the Euclidean norm, and projected
gradient descent (PGD) can minimize various divergence measures. This
often leads to inference results that are better than either simple
variational methods or univariate Gibbs sampling (with a limited time
budget). However, this approach is limited to Ising models, and scales
poorly in the size of the model, due to the difficulty of projecting
onto the spectral norm.

The principal contributions of this paper are, first, a set of sufficient
conditions to guarantee that univariate Gibbs sampling on an MRF will
be fast-mixing (Section \ref{sec:Dependency-for-MRF}), and an algorithm
to project onto this set in the Euclidean norm (Section \ref{sec:Euclidean-Projection}).
A secondary contribution of this paper is considering an alternative
matrix norm (the induced $\infty$-norm) that is somewhat looser than
the spectral norm, but more computationally efficient. Following previous
work \citep{Domke2013ProjectingIsingModel}, these ideas are experimentally
validated via a projected gradient descent algorithm to minimize other
divergences, and looking at the accuracy of the resulting marginals.
The ability to project onto a fast-mixing parameter set may also be
of independent interest. For example, it might be used during maximum
likelihood learning to ensure that the gradients estimated through
sampling are more accurate.

\section{Notation}

We consider discrete pairwise MRFs with $n$ variables, where the
i-th variable takes values in $\{1,...,L_{i}\}$,$ $ $\mathcal{E}$
is the set of edges, and $\theta$ are the potentials on each edge.
Each edge in $\mathcal{E}$ is an ordered pair $(i,j)$ with $i\leq j$.
The parameters are a set of matrices $\theta:=\{\theta^{ij}|\theta^{ij}\in\mathcal{R}^{L_{i}\times L_{j}},\forall(i,j)\in\mathcal{E}\}$.
When $i>j$, and $(j,i)\in\mathcal{E}$, we let $\theta^{ij}$ denote
the transpose of $\theta^{ji}$. The corresponding distribution is
\begin{equation}
p(x;\theta)=\exp\left(\sum_{(i,j)\in\mathcal{E}}\theta^{ij}(x_{i},x_{j})-A(\theta)\right),\label{eq_mrf1}
\end{equation}
where $A(\theta):=\log\sum_{x}\exp\left(\sum_{(i,j)\in\mathcal{E}}\theta^{ij}(x_{i},x_{j})\right)$
is the log-partition function, and $\theta^{ij}(x_{i},x_{j})$ denotes
the entry in the $x_{i}$-th row and $x_{j}$-th column of $\theta^{ij}$.
It is easy to show that any parametrization of a pairwise MRF can
be converted into this form. ``Self-edges'' $(i,i)$ can be included
in $\mathcal{E}$ if one wishes to explicitly represent univariate
terms.

It is sometimes convenient to work with the exponential family representation
\begin{equation}
p(x;\theta)=\exp\{f(x)\cdot\theta-A(\theta)\},
\end{equation}
where $f(x)$ is the sufficient statistics for configuration $x$.
If these are indicator functions for all configurations of all pairs
in $\mathcal{E}$, then the two representations are equivalent.

\section{Background Theory on Rapid Mixing\label{sec:Background-Theory}}

\label{sec_cond}This section reviews background on mixing times that
will be used later in the paper.
\begin{defn}
Given two finite distributions $p$ and $q$, the\textbf{ total variation
distance} $\Vert\cdot\Vert_{TV}$ is defined as $\Vert p(X)-q(X)\Vert_{TV}=\frac{1}{2}\sum_{x}\vert p(X=x)-q(X=x)\vert$.
\end{defn}
Next, one must define a measure of how fast a Markov chain converges
to the stationary distribution. Let the state of the Markov chain
after $t$ iterations be $X^{t}$. Given a constant $\epsilon$, this
is done by finding some number of iterations $\tau(\epsilon)$ such
that the induced distribution $p(X^{t}|X^{0}=x)$ will always have
a distance of less than $\epsilon$ from the stationary distribution,
irrespective of the starting state $x$.
\begin{defn}
Let $\{X^{t}\}$ be the sequence of random variables corresponding
to running Gibbs sampling on a distribution $p$. The \textbf{mixing
time} $\tau(\epsilon)$ is defined as $\tau(\epsilon)=\min\{t:d(t)<\epsilon\}$,
where $d(t)=\max_{x}\Vert\mathbb{P}(X^{t}|X^{0}=x)-p(X)\Vert_{TV}$
is the maximum distance at time $t$ when considering all possible
starting states $x$. 
\end{defn}
Now, we are interested in when Gibbs sampling on a distribution $p$
can be shown to have a fast mixing time. The central property we use
is the dependency of one variable on another, defined informally as
how much the conditional distribution over $X_{i}$ can be changed
when all variables other than $X_{j}$ are the same. 
\begin{defn}
\label{R-def}Given a distribution $p$, the dependency matrix $R$
is defined by 
\[
R_{ij}=\max_{x,x':x_{-j}=x'_{-j}}\Vert p(X_{i}|x_{-i})-p(X_{i}|x'_{-i})\Vert_{TV}.
\]

\end{defn}
Here, the constraint $x_{-j}=x'_{-j}$ indicates that all variables
in $x$ and $x'$ are identical except $x_{j}$. The central result
on rapid mixing is given by the following Theorem, due to Dyer et
al. \citep{Dyer2009Matrixnormsand}, generalizing the work of Hayes
\citep{Hayes2006simpleconditionimplying}. Informally, it states that
if $\Vert R\Vert<1$ for \textit{any} sub-multiplicative norm $\Vert\cdot\Vert$,
then mixing will take on the order of $n\ln n$ iterations, where
$n$ is the number of variables.
\begin{thm}
\citep[Lemma 17]{Dyer2009Matrixnormsand} If $\Vert\cdot\Vert$ is
any sub-multiplicative matrix norm and $||R||<1$, the mixing time
of univariate Gibbs sampling on a system with $n$ variables with
random updates is bounded by $\tau(\epsilon)\leq\frac{n}{1-\Vert R\Vert}\ln\left(\frac{\Vert1_{n}\Vert\,\Vert1_{n}^{T}\Vert}{\epsilon}\right).$
\end{thm}
Here, $\Vert1_{n}\Vert$ denotes the same matrix norm applied to a
matrix of ones of size $n\times1$, and similarly for $1_{n}^{T}$.
In particular, if $\Vert\cdot\Vert$ induced by a vector p-norm, then
$\Vert1_{n}\Vert\,\Vert1_{n}^{T}\Vert=n$.

Since this result is true for a variety of norms, it is natural to
ask, for a given matrix $R$, which norm will give the strongest result.
It can be shown that for symmetric matrices (such as the dependency
matrix), the spectral norm $\Vert\cdot\Vert_{2}$ is always superior.
\begin{thm}
\citep[Lemma 13]{Dyer2009Matrixnormsand} If $A$ is a symmetric matrix
and $\Vert\cdot\Vert$ is any sub-multiplicative norm, then $\Vert A\Vert_{2}\leq\Vert A\Vert$. 
\end{thm}
Unfortunately, as will be discussed below, the spectral norm can be
more computationally expensive than other norms. As such, we will
also consider the use of the $\infty$-norm $\Vert\cdot\Vert_{\infty}$.
This leads to additional looseness in the bound in general, but is
limited in some cases. In particular if $R=rG$ where $G$ is the
adjacency matrix for some regular graph with degree $d$, then for
all induced p-norms, $\Vert R\Vert=rd$, since $\Vert R\Vert=\max_{x\not=0}\Vert Rx\Vert/\Vert x\vert=r\max_{x\not=0}\Vert Gx\Vert/\Vert x\Vert=r\Vert Go\Vert/\Vert o\Vert=rd,$
where $o$ is a vector of ones. Thus, the extra looseness from using,
say, $\Vert\cdot\Vert_{\infty}$ instead of $\Vert\cdot\Vert_{2}$
will tend to be minimal when the graph is close to regular, and the
dependency is close to a constant value. For irregular graphs with
highly variable dependency, the looseness can be much larger.

\section{Dependency for Markov Random Fields\label{sec:Dependency-for-MRF}}

In order to establish that Gibbs sampling on a given MRF will be fast
mixing, it is necessary to compute (a bound on) the dependency matrix
$R$, as done in the following result. The proof of this result is
fairly long, and so it is postponed to the Appendix. Note that it
follows from several bounds on the dependency that are tighter, but
less computationally convenient.
\begin{thm}
\label{thm:R_upper_bound}The dependency matrix for a pairwise Markov
random field is bounded by 
\[
R_{ij}(\theta)\leq\max_{a,b}\frac{1}{2}\Vert\theta_{\cdot a}^{ij}-\theta_{\cdot b}^{ij}\Vert_{\infty}.
\]

\end{thm}
Here, $\theta_{\cdot a}^{ij}$ indicates the $a-$th column of $\theta^{ij}$.
Note that the MRF can include univariate terms as self-edges with
no impact on the dependency bound, regardless of the strength of the
univariate terms. It can be seen easily that from the definition of
$R$ (Definition \ref{R-def}), for any $i$ the entry $R_{ii}$ for
self-edges $(i,i)$ should always be zero. One can, without loss of
generality, set each column of $\theta^{ii}$ to be the same, meaning
that $R_{ii}=0$ in the above bound.

\section{Euclidean Projection Operator\label{sec:Euclidean-Projection}}

The Euclidean distance between two MRFs parameterized respectively
by $\psi$ and $\theta$ is $\Vert\theta-\psi\Vert^{2}:=\sum_{(i,j)\in{\mathcal{E}}}\Vert\theta^{ij}-\psi^{ij}\Vert_{F}^{2}$.
This section considers projecting a given vector $\psi$ onto the
fast mixing set or, formally, finding a vector $\theta$ with minimum
Euclidean distance to $\psi$, subject to the constraint that a norm
$\Vert\cdot\Vert_{*}$ applied to the bound on the dependency matrix
$R$ is less than some constant $c$. Euclidean projection is considered
because, first, it is a straightforward measure of the closeness between
two parameters and, second, it is the building block of the projected
gradient descent for projection in other distance measures. To begin
with, we do not specify the matrix norm $\Vert\cdot\Vert_{*}$, as
it could be any sub-multiplicative norm (Section \ref{sec_cond}).

Thus, in principle, we would like to find $\theta$ to solve
\begin{equation}
\text{proj}_{c}(\psi):=\underset{\theta:\Vert R(\theta)\Vert_{*}\leq c}{\text{argmin }}\Vert\theta-\psi\Vert^{2}.\label{eq:proj0}
\end{equation}
Unfortunately, while convex, this optimization turns out to be somewhat
expensive to solve, due to a lack of smoothness Instead, we introduce
a matrix $Z$, and constrain that $Z_{ij}\geq R_{ij}(\theta)$, where
$R_{ij}(\theta)$ is the bound on dependency in Thm \ref{thm:R_upper_bound}
(as an equality). We add an extra quadratic term $\alpha\Vert Z-Y\Vert_{F}^{2}$
to the objective, where $Y$ is an arbitrarily given matrix and $\alpha>0$
is trade-off between the smoothness and the closeness to original
problem (\ref{eq:proj0}). The smoothed projection operator is
\begin{equation}
\text{proj}_{\mathcal{C}}(\psi,Y):=\underset{(\theta,Z)\in\mathcal{C}}{\text{argmin }}\Vert\theta-\psi\Vert^{2}+\alpha\Vert Z-Y\Vert_{F}^{2},\,\,\,\,\mathcal{C}=\{(\theta,Z):Z_{ij}\geq R_{ij}(\theta),\Vert Z\Vert_{*}\leq c\}.\label{eq:proj1}
\end{equation}

If $\alpha=0$, this yields a solution that is identical to that of
Eq. \ref{eq:proj0}. However, when $\alpha=0$, the objective in Eq.
\ref{eq:proj1} is not strongly convex as a function of $Z$, which
results in a dual function which is non-smooth, meaning it must be
solved with a method like subgradient descent, with a slow convergence
rate. In general, of course, the optimal point of Eq. \ref{eq:proj1}
is different to that of Eq. \ref{eq:proj0}. However, the main usage
of the Euclidean projection operator is the projection step in the
projected gradient descent algorithm for divergence minimization.
In these tasks the smoothed projection operator can be directly used
in the place of the non-smoothed one without changing the final result.
In situations when the exact Euclidean projection is required, it
can be done by initializing $Y_{1}$ arbitrarily and repeating $(\theta_{k+1},Y_{k+1})\leftarrow\text{proj}_{\mathcal{C}}(\psi,Y_{k})$,
for $k=1,2,\dots$ until convergence.

\subsection{Dual Representation\label{sub:Dual-Representation}}
\begin{thm}
\label{thm:dual-rep}Eq. \ref{eq:proj1} has the dual representation
\begin{equation}
\begin{aligned} & \underset{\sigma,\phi,\Delta,\Gamma}{\text{maximize}} &  & g(\sigma,\phi,\Delta,\Gamma)\\
 & \text{subject to} &  & \sigma_{ij}(a,b,c)\geq0,\phi_{ij}(a,b,c)\geq0, &  & \forall(i,j)\in\mathcal{E},a,b,c
\end{aligned}
,\label{eq:the-dual}
\end{equation}
where 
\begin{align*}
g(\sigma,\phi,\Delta,\Gamma) & =\min_{Z}h_{1}(Z;\sigma,\phi,\Delta,\Gamma)+\min_{\theta}h_{2}(\theta;\sigma,\phi)\\
h_{1}(Z;\sigma,\phi,\Delta,\Gamma) & =-\text{tr}(Z\Lambda^{T})+\text{I}(\Vert Z\Vert_{*}\leq c)+\alpha\Vert Z-Y\Vert_{F}^{2}\\
h_{2}(\theta;\sigma,\phi) & =\Vert\theta-\psi\Vert^{2}+\frac{1}{2}\sum_{i,j\in\mathcal{E}}\sum_{a,b,c}\big(\sigma_{ij}(a,b,c)-\phi_{ij}(a,b,c)\big)(\theta_{c,a}^{ij}-\theta_{c,b}^{ij}),
\end{align*}
in which $\Lambda_{ij}:=\Delta_{ij}D_{ij}+\hat{\Gamma}_{ij}+\sum_{a,b,c}\sigma_{ij}(a,b,c)+\phi_{ij}(a,b,c)$,
where $\hat{\Gamma}_{ij}:=\left\{ \begin{array}{ll}
\Gamma_{ij} & \text{if \ensuremath{(i,j)\in\mathcal{E}}}\\
-\Gamma_{ij} & \text{if \ensuremath{(j,i)\in\mathcal{E}}}
\end{array}\right.$, and $D$ is an indicator matrix with $D_{ij}=0$ if $(i,j)\in\mathcal{E}\text{ or }(j,i)\in\mathcal{E}$,
and $D_{ij}=1$ otherwise. The dual variables $\sigma_{ij}$ and $\phi_{ij}$
are arrays of size $L_{j}\times L_{i}\times L_{i}$ for all pairs
$(i,j)\in\mathcal{E}$ while $\Delta$ and $ $\textup{$\Gamma$ are
of size $n\times n$.}
\end{thm}
The proof of this is in the Appendix. Here, $\text{I}(\cdot)$ is
the indicator function with $\text{I}(x)=0$ when $x$ is true and
$\text{I}(x)=\infty$ otherwise.

Being a smooth optimization problem with simple bound constraints,
Eq. \ref{eq:the-dual} can be solved with LBFGS-B \citep{Byrd1995limitedmemoryalgorithm}.
For a gradient-based method like this to be practical, it must be
possible to quickly evaluate $g$ and its gradient. This is complicated
by the fact that $g$ is defined in terms of the minimization of $h_{1}$
with respect to $Z$ and $h_{2}$ with respect to $\theta$. We discuss
how to solve these problems now. We first consider the minimization
of $h_{2}$. This is a quadratic function of $\theta$ and can be
solved analytically via the condition that $\frac{\partial}{\partial\theta}h_{2}(\theta;\sigma,\phi)=0$.
The closed form solution is 
\[
\theta_{c,a}^{ij}=\psi_{c,a}^{ij}-\frac{1}{4}\left[\sum_{b}\sigma_{ij}(a,b,c)-\sum_{b}\sigma_{ij}(b,a,c)-\sum_{b}\phi_{ij}(a,b,c)+\sum_{b}\phi_{ij}(b,a,c)\right]
\]
$\forall(i,j)\in\mathcal{E},1\leq a,c\leq m.$. The time complexity
is linear in the size of $\psi$.

Minimizing $h_{1}$ is more involved. We assume to start that there
exists an algorithm to quickly project a matrix onto the set $\{Z:\Vert Z\Vert_{*}\leq c\}$,
i.e. to solve the optimization problem of 
\begin{equation}
\min_{\Vert Z\Vert_{*}\leq c}\Vert Z-A\Vert_{F}^{2}.\label{eq:_ballproj}
\end{equation}
 Then, we observe that $\arg\min_{Z}h_{1}$ is equal to
\begin{align*}
\arg\min_{Z}-\text{tr}(Z\Lambda^{T})+I(\Vert Z\Vert_{*}\leq c)+\alpha\Vert Z-Y\Vert_{F}^{2} & =\arg\min_{\Vert Z\Vert_{*}\leq c}\Vert Z-(Y+\frac{1}{2\alpha}\Lambda)\Vert_{F}^{2}.
\end{align*}
For different norms $\Vert\cdot\Vert_{*}$, the projection algorithm
will be different and can have a large impact on efficiency. We will
discuss in the followings sections the choices of $\Vert\cdot\Vert_{*}$
and an algorithm for the $\infty$-norm. 

Finally, once $h_{1}$ and $h_{2}$ have been solved, the gradient
of $g$ is (by Danskin's theorem \citep{Bertsekas2004NonlinearProgramming})
\begin{align*}
\frac{\partial g}{\partial\Delta_{ij}}= & -D_{ij}\hat{Z}_{ij}, &  & \frac{\partial g}{\partial\Gamma_{ij}} & = & \hat{Z}_{ji}-\hat{Z}_{ij},\\
\frac{\partial g}{\partial\sigma_{ij}(a,b,c)}= & \frac{1}{2}(\hat{\theta}_{c,a}^{ij}-\hat{\theta}_{c,b}^{ij})-\hat{Z}_{ij}, &  & \frac{\partial g}{\partial\phi_{ij}(a,b,c)} & = & -\partial_{\sigma_{ij}(a,b,c)}g,
\end{align*}
where $\hat{Z}$ and $\hat{\theta}$ represent the solutions to the
subproblems.

\subsection{Spectral Norm}

When $\Vert\cdot\Vert_{*}$ is set to the spectral norm, i.e. the
largest singular value of a matrix, the projection in Eq. \ref{eq:_ballproj}
can be performed by thresholding the singular values of $A$ \citep{Domke2013ProjectingIsingModel}.
Theoretically, using spectral norm will give a tighter bound on $Z$
than other norms (Section \ref{sec:Background-Theory}). However,
computing a full singular value decomposition can be impractically
slow for a graph with a large number of variables.

\subsection{$\infty$-norm}

Here, we consider setting $\Vert\cdot\Vert_{*}$ to the $\infty$-norm,
$\Vert A\Vert_{\infty}=\max_{i}\sum_{j}\vert A_{ij}\vert$, which
measures the maximum $l_{1}$ norm of the rows of $A$. This norm
has several computational advantages. Firstly, to project a matrix
onto a $\infty$-norm ball $\{A:\Vert A_{\infty}\Vert\leq c\}$, we
can simply project each row $a_{i}$ of the matrix onto the $l_{1}$-norm
ball $\{a:\Vert a\Vert_{1}\leq c\}$. Duchi et al. \citep{Duchi2008Efficientprojectionsonto}
provide a method linear in the number of nonzeros in $a$ and logarithmic
in the length of $a$. Thus, if $Z$ is an $n\times n$, matrix, Eq.
\ref{eq:_ballproj} for the $\infty$-norm can be solved in time $n^{2}$
and, for sufficiently sparse matrices, in time $n\log n$.

A second advantage of the $\infty$-norm is that (unlike the spectral
norm) projection in Eq. \ref{eq:_ballproj} preserves the sparsity
of the matrix. Thus, one can disregard the matrix $D$ and dual variables
$\Delta$ when solving the optimization in Theorem \ref{thm:dual-rep}.
This means that $Z$ itself can be represented sparsely, i.e. we only
need variables for those $(i,j)\in\mathcal{E}$. These simplifications
significantly improve the efficiency of projection, with some tradeoff
in accuracy.

\section{Projection in Divergences}

In this section, we want to find a distribution $p(x;\theta)$ in
the fast mixing family closest to a target distribution $p(x;\psi)$
in some divergence $D(\psi,\theta)$. The choice of divergence depends
on convenience of projection, the approximate family and the inference
task. We will first present a general algorithmic framework based
on projected gradient descent (Algorithm \ref{alg_pgd}), and then
discuss the details of several previously proposed divergences \citep{Minka2005Divergencemeasuresand,Domke2013ProjectingIsingModel}.

\subsection{General algorithm framework for divergence minimization}

The problem of projection in divergences is formulated as 
\begin{equation}
\min_{\theta\in\mathcal{\bar{C}}}D(\psi,\theta),\label{eq_gdproj}
\end{equation}
$D(\cdot,\cdot)$ is some divergence measure, and $\mathcal{\bar{C}}:=\{\theta:\exists Z,s.t.(\theta,Z)\in C\}$,
where $C$ is the feasible set in Eq. \ref{eq:proj1}. Our general
strategy for this is to use projected gradient descent to solve the
optimization
\begin{equation}
\min_{(\theta,Z)\in\mathcal{C}}D(\psi,\theta),\label{eq:gdproj1}
\end{equation}
using the joint operator to project onto $\mathcal{C}$ described
in Section \ref{sec:Euclidean-Projection}.

\begin{algorithm}[tb]
\protect\caption{Projected gradient descent for divergence projection}

\label{alg_pgd} \begin{algorithmic}\STATE Initialize ($\theta_{1}$,
$Z_{1}$), $k\leftarrow1$. \REPEAT\STATE $\theta'\leftarrow\theta_{k}-\lambda\nabla_{\theta}D(\psi,\theta_{k})$
\STATE $(\theta_{k+1},Z_{k+1})\leftarrow\text{proj}_{\mathcal{C}}(\theta',Z_{k})$
\STATE $k\leftarrow k+1$ \UNTIL{$convergence$} \end{algorithmic} 
\end{algorithm}

For different divergences, the only difference in projection algorithm
is the evaluation of the gradient $\nabla_{\theta}D(\psi,\theta)$.
It is clear that if $(\theta^{*},Z^{*})$ is the solution of Eq. \ref{eq:gdproj1},
then $\theta^{*}$ is the solution of \ref{eq_gdproj}.

\begin{wrapfigure}{O}{0.53\columnwidth}%
\begin{centering}
\textbf{\vspace{-5bp}
\includegraphics[width=0.5\textwidth]{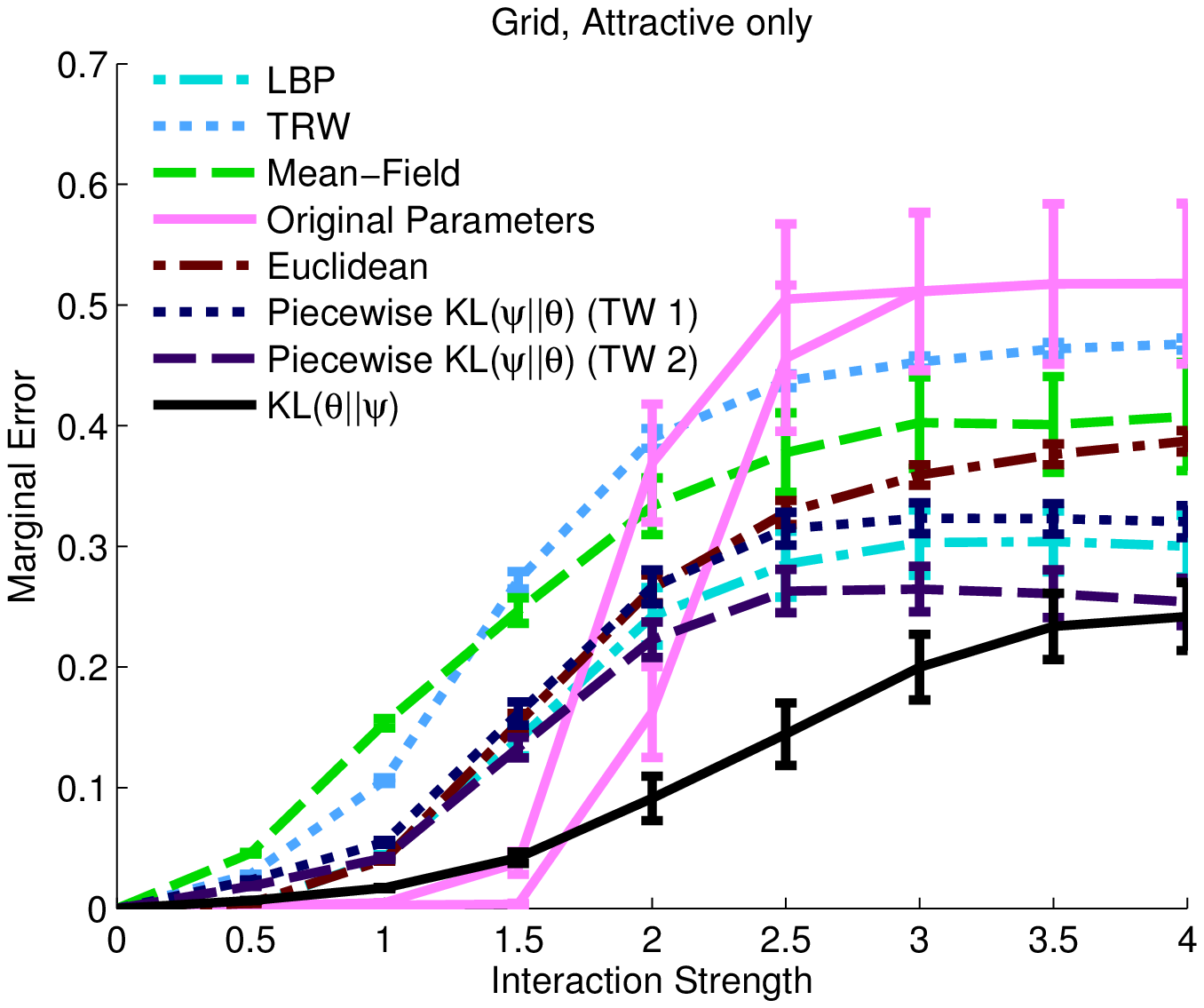}} 
\par\end{centering}

\begin{centering}
\includegraphics[width=0.5\textwidth]{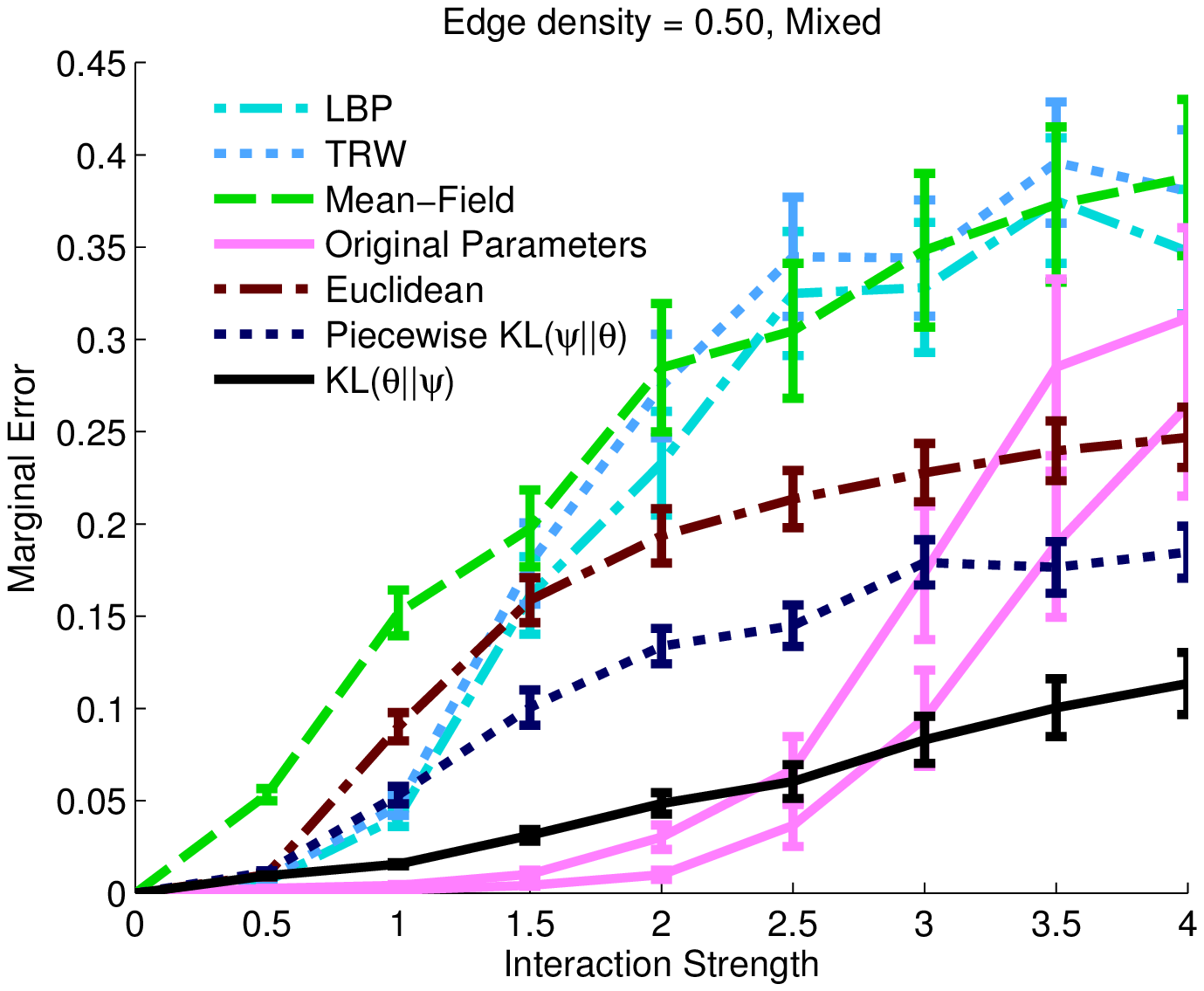}
\par\end{centering}

\protect\caption{Mean univariate marginal error on $16\times16$ grids (top) with attractive
interactions and median-density random graphs (bottom) with mixed
interactions, comparing 30k iterations of Gibbs sampling after projection
(onto the $l_{\infty}$ norm) to variational methods. The original
parameters also show a lower curve with $10^{6}$ samples.\label{fig:ising-strength-vs-error}}
\vspace{-70bp}
\end{wrapfigure}%

\subsection{Divergences}

In this section, we will discuss the different choices of divergences
and corresponding projection algorithms.

\subsubsection{KL-divergence\label{sub:KL}}

\begin{wrapfigure}{O}{0.5\columnwidth}%
\begin{centering}
\vspace{-40pt}
\includegraphics[width=0.48\textwidth]{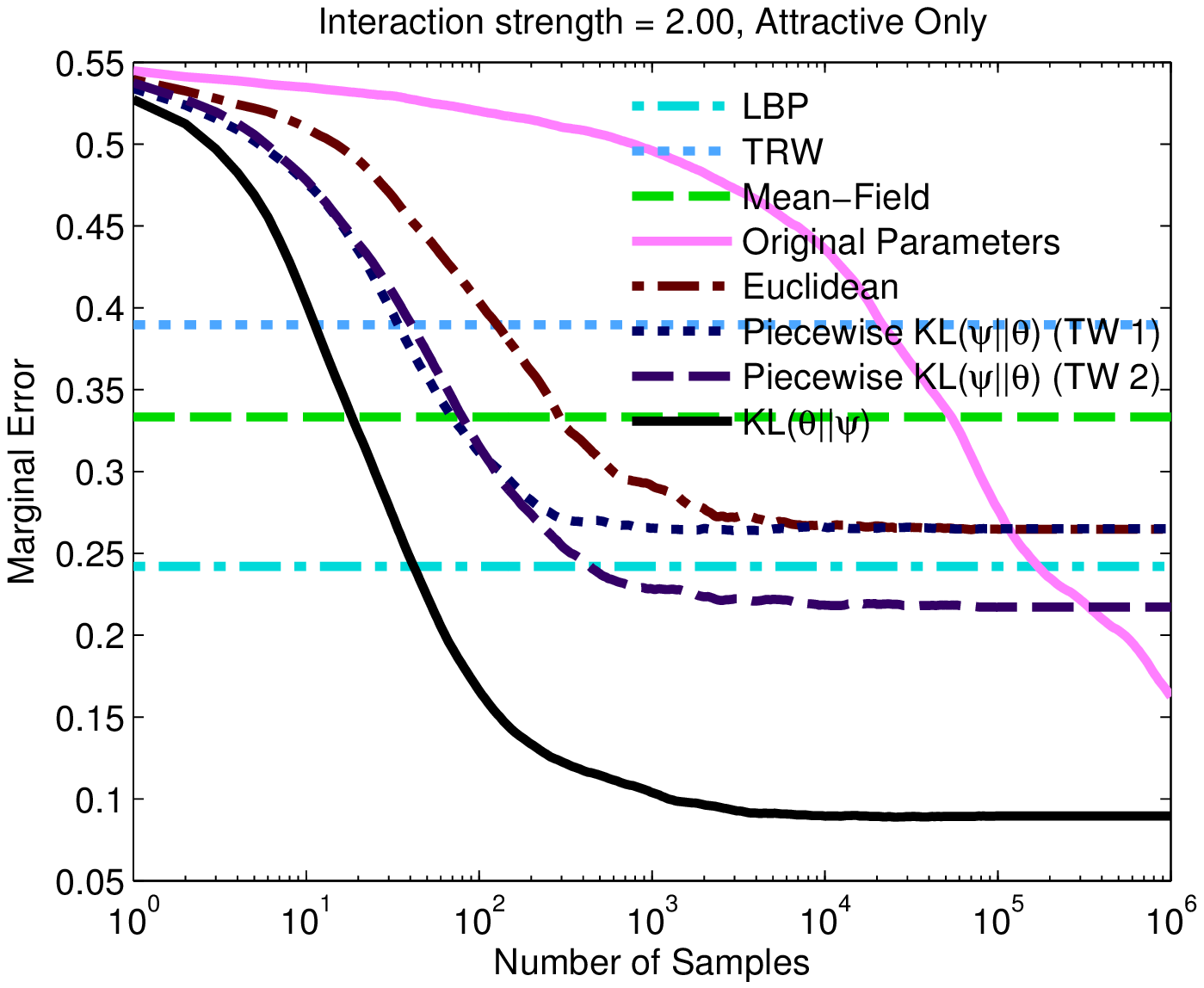} 
\par\end{centering}

\begin{centering}
\includegraphics[width=0.48\textwidth]{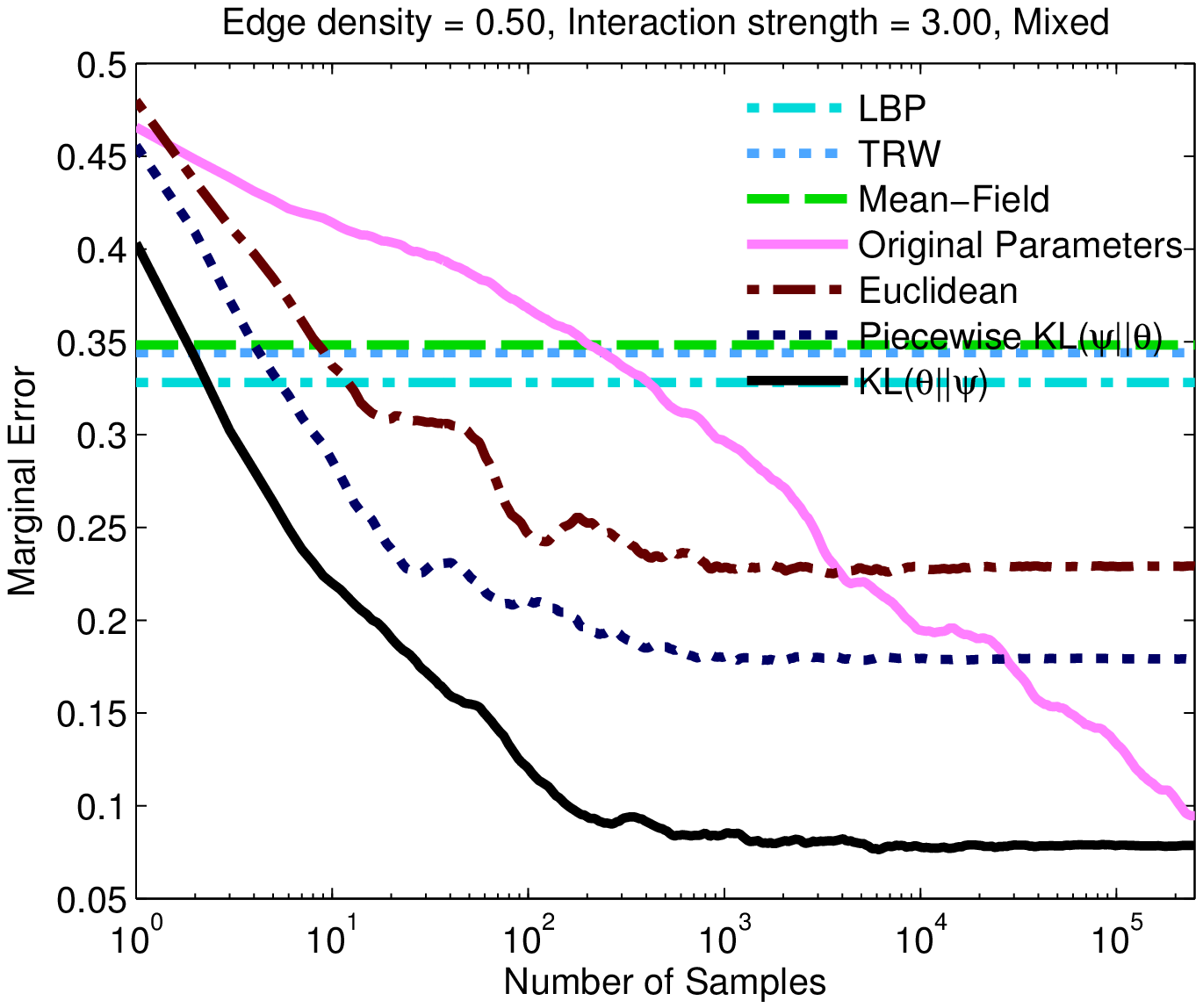}
\par\end{centering}

\protect\caption{Examples of the accuracy of obtained marginals vs. the number of samples.
Top: Grid graphs. Bottom: Median-Density Random graphs.\label{fig:ising-time-vs-error}}
\vspace{-10pt}
\end{wrapfigure}%
The KL-divergence $\text{KL}(\psi\Vert\theta):=\sum_{x}p(x;\psi)\log\frac{p(x;\psi)}{p(x;\theta)}$
is arguably the optimal divergence for marginal inference because
it strives to preserve the marginals of $p(x;\theta)$ and $p(x;\psi)$.
However, projection in KL-divergence is intractable here because the
evaluation of the gradient $\nabla_{\theta}\text{KL}(\psi\Vert\theta)$
requires the marginals of distribution $\psi$.

\subsubsection{Piecewise KL-divergence\label{sub:PKL}}

One tractable surrogate of $\text{KL}(\psi\Vert\theta)$ is the piecewise
KL-divergence \citep{Domke2013ProjectingIsingModel} defined over
some tractable subgraphs. Here, $D(\psi,\theta):=\max_{T\in\mathcal{T}}\text{KL}(\psi_{T}\Vert\theta_{T}),$
where $\mathcal{T}$ is a set of low-treewidth subgraphs. The gradient
can be evaluated as $\nabla_{\theta}D(\psi,\theta)=\nabla_{\theta}\text{KL}(\psi_{T^{*}}\Vert\theta_{T^{*}})$
where $T^{*}=\arg\max_{T\in\mathcal{T}}\text{KL}(\psi_{T}\Vert\theta_{T})$.
For any $T$ in $\mathcal{T}$, $\text{KL}(\psi_{T}\Vert\theta_{T})$
and its gradient can be evaluated by the junction-tree algorithm.

\subsubsection{Reversed KL-divergence\label{sub:RKL}}

The ``reversed'' KL-divergence $\text{KL}(\theta\Vert\psi)$ is
minimized by mean-field methods. In general $\text{KL}(\theta\Vert\psi)$
is inferior to $\text{KL}(\psi\Vert\theta)$ for marginal inference
since it tends to underestimate the support of the distribution \citep{Minka2005Divergencemeasuresand}.
Still, it often works well in practice. $\nabla_{\theta}\text{KL}(\theta\Vert\psi)$
can computed as % \begin{equation}
$\nabla_{\theta}\text{KL}(\theta\Vert\psi)=\sum_{x}p(x;\theta)(\theta-\psi)\cdot f(x)\big(f(x)-\mu(\theta)\big)$
% \end{equation}
, which can be approximated by samples generated from $p(x;\theta)$
\citep{Domke2013ProjectingIsingModel}. In implementation, we maintain
a ``pool'' of samples, each of which is updated by a single Gibbs
step after each iteration of Algorithm \ref{alg_pgd}.

\section{Experiments}

The experiments below take two stages: first, the parameters are projected
(in some divergence) and then we compare the accuracy of sampling
with the resulting marginals. We focus on this second aspect. However,
we provide a comparison of the computation time for various projection
algorithms in Table \ref{tab_timing}, and when comparing the accuracy
of sampling with a given amount of time, provide two curves for sampling
with the original parameters, where one curve has an extra amount
of sampling effort roughly approximating the time to perform projection
in the reversed KL divergence.

\subsection{Synthetic MRFs}

Our first experiment follows that of \citep{Hazan2008ConvergentMessagePassing,Domke2013ProjectingIsingModel}
in evaluating the accuracy of approximation methods in marginal inference.
In the experiments, we approximate randomly generated MRF models with
rapid-mixing distributions using the projection algorithms described
previously. Then, the marginals of fast mixing approximate distributions
are estimated by running a Gibbs chain on each distribution. These
are compared against exact marginals as computed by the junction tree
algorithm. We use the mean absolute difference of the marginals $\vert p(X_{i}=1)-q(X_{i}=1)\vert$
as the accuracy measure. We compare to Naive mean-field (MF), Gibbs
sampling on original parameters (Gibbs), and Loopy belief propagation
(LBP). Many other methods have been compared against a similar benchmark
\citep{Globerson2007Approximateinferenceusing,Hazan2008ConvergentMessagePassing}.

While our methods are for general MRFs, we test on Ising potentials
because this is a standard benchmark. Two graph topologies are used:
two-dimensional $16\times16$ grids and 10 node random graphs, where
each edge is independently present with probability $p_{e}\in\{0.3,0.5,0.7\}$.
Node parameters $\theta_{i}$ are uniform from $[-d_{n},d_{n}]$ with
fixed field strength $d_{n}=1.0$. Edge parameters $\theta_{ij}$
are uniform from $[-d_{e},d_{e}]$ or $[0,d_{e}]$ to obtain mixed
or attractive interactions respectively, with interaction strengths
$d_{e}\in\{0,0.5,\dots,4\}$. Figure \ref{fig:ising-strength-vs-error}
shows the average marginal error at different interaction strengths.
Error bars show the standard error normalized by the number of samples,
which can be interpreted as a $68.27\%$ confidence interval. We also
include time-accuracy comparisons in Figure \ref{fig:ising-time-vs-error}.
All results are averaged over 50 random trials. We run Gibbs long
enough ( $10^{6}$ samples) to get a fair comparison in terms of running
time. 

Except where otherwise stated, parameters are projected onto the ball
$\{\theta:\Vert R(\theta)\Vert_{\infty}\leq c\}$, where $c=2.5$
is larger than the value of $c=1$ suggested by the proofs above.
Better results are obtained by using this larger constraint set, presumably
because of looseness in the bound. For piecewise projection, grids
use simple vertical and horizontal chains of treewidth either one
or two. For random graphs, we randomly generate spanning trees until
all edges are covered. Gradient descent uses a fixed step size of
$\lambda=0.1$. A Gibbs step is one ``systematic-scan'' pass over
all variables between. The reversed KL divergence maintains a pool
of 500 samples, each of which is updated by a single Gibbs step in
each iteration.

We wish to compare the trade-off between computation time and accuracy
represented by the choice between the use of the $\infty$ and spectral
norms. We measure the running time on $16\times16$ grids in Table
\ref{tab_timing}, and compare the accuracy in Figure \ref{fig:SP_vs_Inf}.

\begin{wrapfigure}{O}{0.5\columnwidth}%
\begin{centering}
\vspace{-20pt}
\includegraphics[bb=0bp 5bp 401bp 333bp,clip,width=0.48\columnwidth]{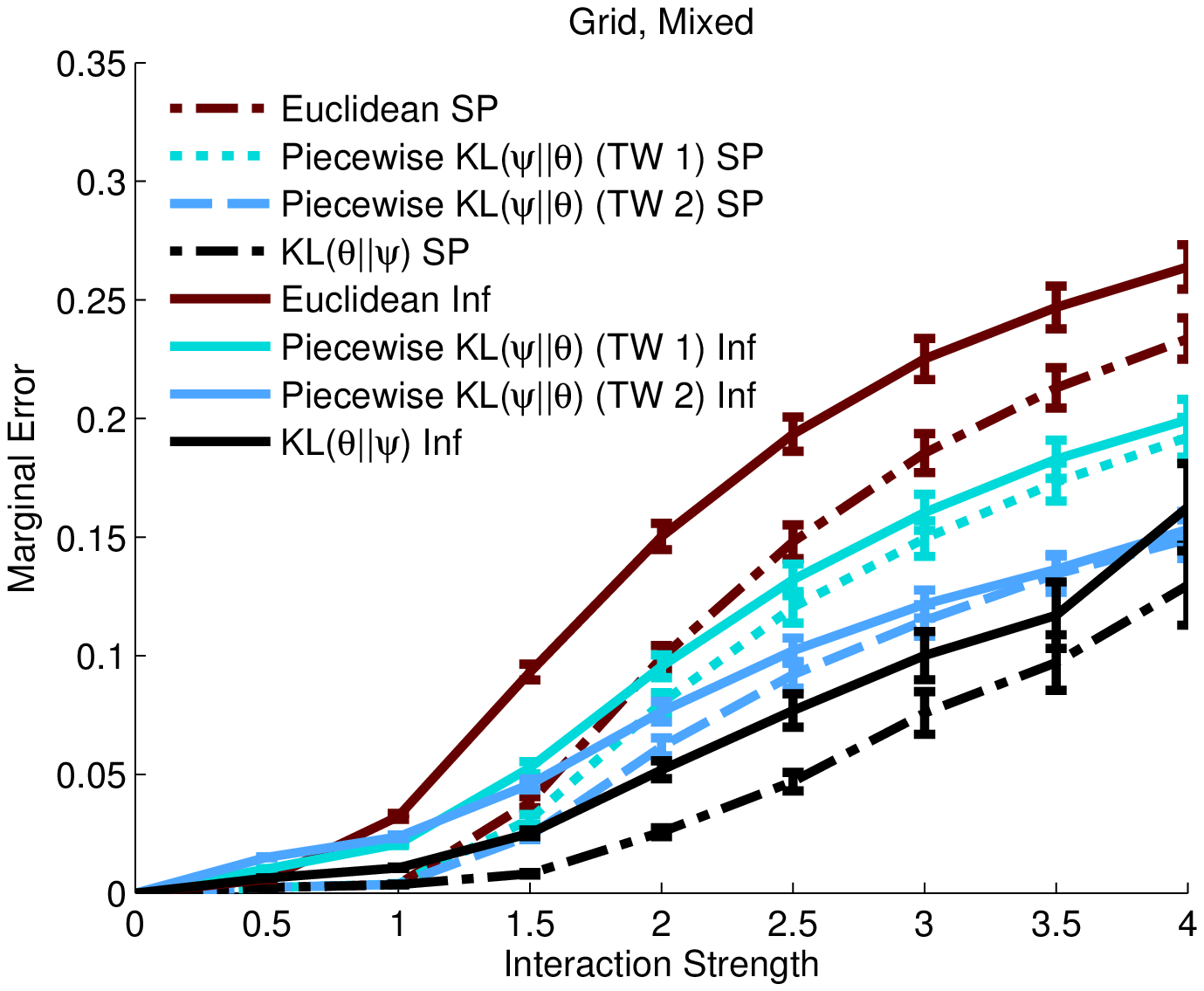}
\par\end{centering}

\protect\caption{The marginal error using $\infty$-norm projection (solid lines) and
spectral-norm projection (dotted lines) on 16x16 Ising grids.\label{fig:SP_vs_Inf}\vspace{-15pt}
}
\end{wrapfigure}%
The appendix contains results for a three-state Potts model on an
$8\times8$ grid, as a test of the multivariate setting. Here, the
intractable divergence $KL(\psi\Vert\theta)$ is included for reference,
with the projection computed with the help of the junction tree algorithm
for inference.
\begin{table*}[t]
\setlength{\tabcolsep}{3.0pt} \protect\caption{Running times on $16\times16$ grids with attractive interactions.
Euclidean projection converges in around 5 LBFGS-B iterations. Piecewise
projection (with a treewidth of 1) and reversed KL projection use
60 gradient descent steps. All results use a single core of a Intel
i7 860 processor.}

\begin{centering}
\begin{tabular}{|c|cc|cc|cc|cc|}
\hline 
 & \multicolumn{2}{c|}{Gibbs} & \multicolumn{2}{c|}{{\small{}Euclidean}} & \multicolumn{2}{c|}{{\small{}Piecewise}} & \multicolumn{2}{c|}{{\small{}Reversed-KL}}\tabularnewline
\cline{2-9} 
 & {\small{}30k Steps} & {\small{}$10^{6}$ Steps} & $l_{\infty}$ norm & $l_{2}$ norm & $l_{\infty}$ norm & $l_{2}$ norm & $l_{\infty}$ norm & $l_{2}$ norm\tabularnewline
\hline 
$d_{e}=1.5$ & 0.67s & 22.42s & 1.50s & 25.63s & 12.87s & 45.26s & 13.13s & 66.81s\tabularnewline
$d_{e}=3.0$ & 0.67s & 22.42s & 3.26s & 164.34s & 20.73s & 211.08s & 20.12s & 254.25s\tabularnewline
\hline 
\end{tabular}
\par\end{centering}

\label{tab_timing}
\end{table*}

\subsection{Berkeley binary image denoising}

This experiment evaluates various methods for denoising binary images
from the Berkeley segmentation dataset downscaled from $300\times200$
to $120\times80$. The images are binarized by setting $Y_{i}=1$
if pixel $i$ is above the average gray scale in the image, and $Y_{i}=-1$.
The noisy image $X$ is created by setting: $X_{i}=\frac{Y_{i}+1}{2}_{i}(1-t_{i}^{1.25})+\frac{1-Y_{i}}{2}t_{i}^{1.25}$,
in which $t_{i}$ is sampled uniformly from $[0,1]$. For inference
purposes, the conditional distribution $Y$ is modeled as $P(Y\vert X)\propto\exp\left(\beta\sum_{ij}Y_{i}Y_{j}+\frac{\alpha}{2}\sum_{i}(2X_{i}-1)Y_{i}\right)$,
where the pairwise strength $\beta>0$ encourages smoothness. On this
attractive-only Ising potential, the Swendsen-Wang method {\small{}\citep{Swendsen1987Nonuniversalcriticaldynamics}}
mixes rapidly, and so we use the resulting samples to estimate the
ground truth. The parameters $\alpha$ and $\beta$ are heuristically
chosen to be $0.5$ and $0.7$ respectively.

\begin{wrapfigure}{O}{0.5\columnwidth}%
\begin{centering}
\vspace{-20pt}
\includegraphics[width=0.48\textwidth]{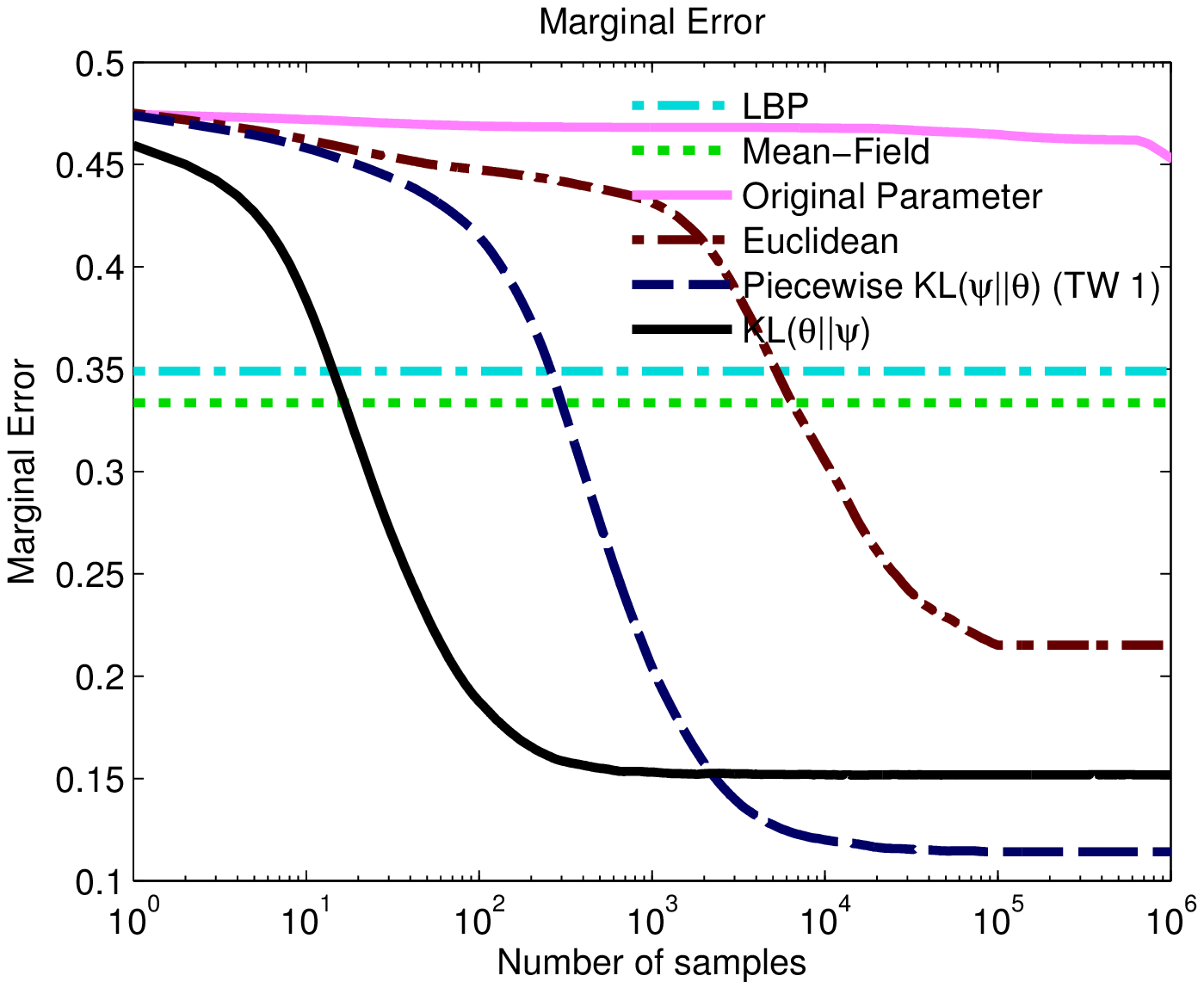}
\par\end{centering}

\centering{} \protect\caption{Average marginal error  on the Berkeley segmentation dataset. \label{fig:Average-marginal-denoising}}
\vspace{-30pt}
\end{wrapfigure}%
Figure \ref{fig:Average-marginal-denoising} shows the decrease of
average marginal error. To compare running time, Euclidean and $K(\theta\Vert\psi)$
projection cost approximately the same as sampling $10^{5}$ and $4.8\times10^{5}$
samples respectively. Gibbs sampling on the original parameter converges
very slowly. Sampling the approximate distributions from our projection
algorithms converge quickly in less than $10^{4}$ samples.

\section{Conclusions}

We derived sufficient conditions on the parameters of an MRF to ensure
fast-mixing of univariate Gibbs sampling, along with an algorithm
to project onto this set in the Euclidean norm. As an example use,
we explored the accuracy of samples obtained by projecting parameters
and then sampling, which is competitive with simple variational methods
as well as traditional Gibbs sampling. Other possible applications
of fast-mixing parameter sets include constraining parameters during
learning.

\subsection*{Acknowledgments}

NICTA is funded by the Australian Government through the Department
of Communications and the Australian Research Council through the
ICT Centre of Excellence Program.

\newpage{}

\clearpage{}

\bibliographystyle{plain}
\bibliography{/Users/jdomke/Dropbox/Papers/Bibliography/justindomke}

\begin{thebibliography}{10}

\bibitem{Bertsekas2004NonlinearProgramming}
Dimitri Bertsekas.
\newblock {\em Nonlinear Programming}.
\newblock Athena Scientific, 2004.

\bibitem{Byrd1995limitedmemoryalgorithm}
Richard~H. Byrd, Peihuang Lu, Jorge Nocedal, and Ciyou Zhu.
\newblock A limited memory algorithm for bound constrained optimization.
\newblock {\em SIAM J. Sci. Comput.}, 16(5):1190--1208, 1995.

\bibitem{Domke2013ProjectingIsingModel}
Justin Domke and Xianghang Liu.
\newblock Projecting {I}sing model parameters for fast mixing.
\newblock In {\em NIPS}, 2013.

\bibitem{Duchi2008Efficientprojectionsonto}
John~C. Duchi, Shai Shalev-Shwartz, Yoram Singer, and Tushar Chandra.
\newblock Efficient projections onto the $l_1$-ball for learning in high
  dimensions.
\newblock In {\em ICML}, 2008.

\bibitem{Dyer2009Matrixnormsand}
Martin~E. Dyer, Leslie~Ann Goldberg, and Mark Jerrum.
\newblock Matrix norms and rapid mixing for spin systems.
\newblock {\em Ann. Appl. Probab.}, 19:71--107, 2009.

\bibitem{Globerson2007Approximateinferenceusing}
Amir Globerson and Tommi Jaakkola.
\newblock Approximate inference using conditional entropy decompositions.
\newblock In {\em UAI}, 2007.

\bibitem{Hayes2006simpleconditionimplying}
Thomas~P. Hayes.
\newblock A simple condition implying rapid mixing of single-site dynamics on
  spin systems.
\newblock In {\em FOCS}, pages 39--46, 2006.

\bibitem{Hazan2008ConvergentMessagePassing}
Tamir Hazan and Amnon Shashua.
\newblock Convergent message-passing algorithms for inference over general
  graphs with convex free energies.
\newblock In {\em UAI}, pages 264--273, 2008.

\bibitem{Koller2009ProbabilisticGraphicalModels:}
D.~Koller and N.~Friedman.
\newblock {\em Probabilistic Graphical Models: Principles and Techniques}.
\newblock MIT Press, 2009.

\bibitem{Minka2001ExpectationPropagationApproximate}
Thomas Minka.
\newblock Expectation propagation for approximate bayesian inference.
\newblock In {\em UAI}, 2001.

\bibitem{Minka2005Divergencemeasuresand}
Thomas Minka.
\newblock Divergence measures and message passing.
\newblock Technical report, 2005.

\bibitem{Swendsen1987Nonuniversalcriticaldynamics}
Robert~H. Swendsen and Jian-Sheng Wang.
\newblock Nonuniversal critical dynamics in monte carlo simulations.
\newblock {\em Phys. Rev. Lett.}, 58:86--88, Jan 1987.

\bibitem{Wainwright2008GraphicalModelsExponential}
Martin Wainwright and Michael Jordan.
\newblock Graphical models, exponential families, and variational inference.
\newblock {\em Found. Trends Mach. Learn.}, 1(1-2):1--305, 2008.

\bibitem{Yedidia2005ConstructingFreeEnergy}
Jonathan Yedidia, William Freeman, and Yair Weiss.
\newblock Constructing free energy approximations and generalized belief
  propagation algorithms.
\newblock {\em IEEE Transactions on Information Theory}, 51:2282--2312, 2005.

\end{thebibliography}

\newpage{}

\section{Appendix}

\subsection{Proof of MRF Dependency Bound}

This section gives a proof of the bound on the dependency matrix stated
in Section \ref{sec:Dependency-for-MRF} above.

To start with, we observe the conditional distribution of a single
variable $x_{i}$ when all others are fixed, which is easy to calculate.
\begin{lem}
The conditional probability of one variable given all others is 
\[
p(X_{i}=\cdot\vert X_{-i}=x_{-i})=\text{sig}\left(\sum_{k\in N(i)}\theta_{\cdot x_{k}}^{ik}\right),
\]
where $\text{sig}$ is the ``multivariate sigmoid'' defined as $(v)=\exp(v)/1^{T}\exp(v)$,
and $N(i)$ is the set of indices that are in a pair with $i$.
\end{lem}
Now, to compute the influence matrix, we must consider what configuration
of all the variables other than $x_{i}$ and $x_{j}$ will allow a
change in $x_{j}$ to induce the greatest change in $x_{i}$ (Definition
\ref{R-def}). 
\begin{lem}
The dependency matrix is given by 
\begin{align*}
R_{ij} & =\max_{x,y:x_{-j}=y_{-j}}\frac{1}{2}\Vert\text{sig}(\theta_{\cdot x_{j}}^{ij}+s)-\text{sig}(\theta_{\cdot y_{j}}^{ij}+s)\Vert_{1}\\
s & =\sum_{k\in N(i)\backslash j}\theta_{\cdot x_{k}}^{ik}
\end{align*}
 \end{lem}
\begin{proof}
Using the previous Lemma inside the definition of the dependency matrix
(Definition \ref{R-def}) gives that

\begin{align*}
R_{ij} & =\max_{x,y:x_{-j}=y_{-j}}\Vert p(X_{i}=\cdot|x_{-i})-p(X_{i}=\cdot|y_{-i})\Vert_{TV}\\
 & =\max_{x,y:x_{-j}=y_{-j}}\frac{1}{2}\Vert\text{sig}(\sum_{k\in N(i)}\theta_{\cdot x_{k}}^{ik})-\text{sig}(\sum_{k\in N(i)}\theta_{\cdot y_{k}}^{ik})\Vert_{1}.
\end{align*}
Substituting the definition of $s$ inside each of the $\text{sig}()$
terms gives the result. 

While the previous Lemma bounds the dependency, it is not in a very
convenient form. Hence, the rest of this section will apply a series
of relaxations to obtain more convenient upper-bounds. The first of
these is obtained by letting $s$ be an arbitrary vector, rather than
determined by $\theta$ and $x$. \end{proof}
\begin{lem}
The dependency matrix for an MRF is bounded by 
\begin{align*}
R_{ij} & \leq\max_{x_{j},y_{j}}\max_{s}\frac{1}{2}\Vert\text{sig}(\theta_{\cdot x_{j}}^{ij}+s)-\text{sig}(\theta_{\cdot y_{j}}^{ij}+s)\Vert_{1}.
\end{align*}

\end{lem}
The following Lemma will be needed in what follows.
\begin{lem}
For vectors $x,y,s$,
\[
\max_{s}\Vert\text{sig}(x+s)-\text{sig}(y+s)\Vert_{1}=2\vert2a-1\vert,
\]
where $a=\sigma\left(\frac{1}{2}\text{range}(y-x)\right)$. Here,
$\text{range}(z)$ is defined as $\max_{i}z_{i}-\min z_{i}$. 
\end{lem}
Now, applying this Lemma to the previous result on the dependency
matrix gives the following Theorem.
\begin{thm}
The dependency matrix for an MRF is bounded by
\[
R_{ij}\leq\frac{1}{4}\max_{a,b}\vert\text{range}(\theta_{\cdot a}^{ij}-\theta_{\cdot b}^{ij})\vert.
\]
\end{thm}
\begin{proof}
The previous result gives us the bound
\[
R_{ij}\leq\max_{a,b}\vert2\sigma(\frac{1}{2}\text{range}(\theta_{\cdot a}^{ij}-\theta_{\cdot b}^{ij})-1\vert.
\]
 Using the easily-proven fact that $\vert2\sigma(\frac{1}{2}x)-1\vert\leq\frac{1}{4}\vert x\vert$
gives the result. \end{proof}
\begin{cor}
The dependency matrix for an MRF is bounded by
\[
R_{ij}\leq\max_{a,b}\frac{1}{4}\Vert\theta_{\cdot a}^{ij}-\theta_{\cdot b}^{ij}\Vert_{1},\quad R_{ij}\leq\max_{a,b}\frac{1}{2}\Vert\theta_{\cdot a}^{ij}-\theta_{\cdot b}^{ij}\Vert_{\infty}.
\]
\end{cor}
\begin{proof}
This follows immediately from the observations that $\vert\text{range}(x)\vert\leq\Vert x\Vert_{1}$
and that $\vert\text{range}(x)\vert\leq2\Vert x\Vert_{\infty}$.
\end{proof}

\subsection{Proof of Dual Representation for Euclidean Projection Operator}

This section gives a proof of the main result of Section \ref{sub:Dual-Representation},
as stated below.
\begin{thm}
The projection operator

\begin{equation}
\text{proj}_{\mathcal{C}}(\psi,Y):=\underset{(\theta,Z)\in\mathcal{C}}{\text{argmin }}\Vert\theta-\psi\Vert^{2}+\alpha\Vert Z-Y\Vert_{F}^{2},\,\,\,\,\mathcal{C}=\{(\theta,Z):Z_{ij}\geq R_{ij}(\theta),\Vert Z\Vert_{*}\leq c\}\label{eq:projection-appendix}
\end{equation}
has the dual representation of
\begin{equation}
\begin{aligned} & \underset{\sigma,\phi,\Delta,\Gamma}{\text{maximize}} &  & g(\sigma,\phi,\Delta,\Gamma)\\
 & \text{subject to} &  & \sigma_{ij}(a,b,c)\geq0,\phi_{ij}(a,b,c)\geq0, &  & \forall(i,j)\in\mathcal{E},a,b,c
\end{aligned}
,\label{eq:the-dual-appendix}
\end{equation}
where 
\begin{align*}
g(\sigma,\phi,\Delta,\Gamma) & =\min_{Z}h_{1}(Z;\sigma,\phi,\Delta,\Gamma)+\min_{\theta}h_{2}(\theta;\sigma,\phi)\\
h_{1}(Z;\sigma,\phi,\Delta,\Gamma) & =-\text{tr}(Z\Lambda^{T})+\text{I}(\Vert Z\Vert_{*}\leq c)+\alpha\Vert Z-Y\Vert_{F}^{2}\\
h_{2}(\theta;\sigma,\phi) & =\Vert\theta-\psi\Vert^{2}+\frac{1}{2}\sum_{i,j\in\mathcal{E}}\sum_{a,b,c}\big(\sigma_{ij}(a,b,c)-\phi_{ij}(a,b,c)\big)(\theta_{c,a}^{ij}-\theta_{c,b}^{ij}),
\end{align*}
in which $\Lambda_{ij}:=\Delta_{ij}D_{ij}+\hat{\Gamma}_{ij}+\sum_{a,b,c}\sigma_{ij}(a,b,c)+\phi_{ij}(a,b,c)$,
where $\hat{\Gamma}_{ij}:=\left\{ \begin{array}{ll}
\Gamma_{ij} & \text{if \ensuremath{(i,j)\in\mathcal{E}}}\\
-\Gamma_{ij} & \text{if \ensuremath{(j,i)\in\mathcal{E}}}
\end{array}\right.$, and $D$ is an indicator matrix with $D_{ij}=0$ if $(i,j)\in\mathcal{E}\text{ or }(j,i)\in\mathcal{E}$,
and $D_{ij}=1$ otherwise. The dual variables $\sigma_{ij}$ and $\phi_{ij}$
are arrays of size $L_{j}\times L_{i}\times L_{i}$ for all pairs
$(i,j)\in\mathcal{E}$ while $\Delta$ and $ $\textup{$\Gamma$ are
of size $n\times n$.}\end{thm}
\begin{proof}
Firstly, we observe that the minimization in Eq. \ref{eq:projection-appendix}
is equivalent to
\begin{equation}
\begin{aligned} & \underset{\theta,Z}{\text{minimize}} &  & \Vert\theta-\psi\Vert^{2}+\alpha\Vert Z-Y\Vert_{F}^{2}\\
 & \text{subject to} &  & \Vert Z\Vert_{*}\leq c\\
 &  &  & Z_{ij}=Z_{ji},\quad\forall(i,j)\in{\mathcal{E}}\\
 &  &  & Z_{ij}\geq\max_{1\leq a,b\leq m}\frac{1}{2}\Vert\theta_{.a}^{ij}-\theta_{.b}^{ij}\Vert_{\infty},\forall(i,j)\in{\mathcal{E}}\\
 &  &  & D_{ij}Z_{ij}=0,\quad1\leq i,j\leq n.
\end{aligned}
\label{eq_proj-appendix}
\end{equation}

\end{proof}
Now, consider the Lagrangian of this problem, 
\begin{align*}
L & (\theta,Z,\sigma,\phi,\Delta,\Gamma):=\Vert\theta-\psi\Vert^{2}+\alpha\Vert Z-Y\Vert_{F}^{2}+\text{I}(\Vert Z\Vert_{*}\leq c)\\
 & -\sum_{(i,j)\in\mathcal{E}}\sum_{a,b,c}\sigma_{ij}(a,b,c)\big(Z_{ij}-\frac{1}{2}(\theta_{c,a}^{ij}-\theta_{c,b}^{ij})\big)-\sum_{(i,j)\in\mathcal{E}}\sum_{a,b,c}\phi_{ij}(a,b,c)\big(Z_{ij}+\frac{1}{2}(\theta_{c,a}^{ij}-\theta_{c,b}^{ij})\big)\\
 & -\sum_{i,j}\Delta_{ij}D_{ij}Z_{ij}-\sum_{(i,j)\in\mathcal{E}}\Gamma_{ij}(Z_{ij}-Z_{ji}).
\end{align*}
Here, $\Gamma$, $\Delta$, $\sigma_{ij}$ and $\phi_{ij},1\leq i,j\leq n$
are dual variables and $\sum_{i,j}$ denotes $\sum_{1\leq i,j\leq n}$
for simplicity of notation. Here, note that $L$ is independent of
$\Gamma_{ij},\sigma_{ij}$ and $\phi_{ij}$ for $(i,j)\not\in\mathcal{E}$.
For convenience, one can simply set these to zero.

It is straightforward to verify that the problem in Eq. \ref{eq_proj-appendix}
is convex and Slater's conditions hold. Thus, by strong duality we
have the the solution of Eq. \ref{eq_proj-appendix} is equal to 
\[
\min_{\theta,Z}\max_{\sigma\geq0,\phi\geq0,\Delta,\Gamma}L(\theta,Z,\sigma,\phi,\Delta,\Gamma)=\max_{\sigma\geq0,\phi\geq0,\Delta,\Gamma}g(\sigma,\phi,\Delta,\Gamma),
\]
where we define the dual function 
\[
g(\sigma,\phi,\Delta,\Gamma)=\min_{\theta,Z}L(\theta,Z,\sigma,\phi,\Delta,\Gamma).
\]

Finally, by a simple manipulation of terms, we can see that\textcolor{green}{{}
}
\begin{align*}
g(\sigma,\phi,\Delta,\Gamma) & =\min_{Z}h_{1}(Z;\sigma,\phi,\Delta,\Gamma)+\min_{\theta}h_{2}(\theta;\sigma,\phi)\\
h_{1}(Z;\sigma,\phi,\Delta,\Gamma) & =-\text{tr}(Z\Lambda^{T})+\text{I}(||Z||_{*}\leq c)+\alpha||Z-Y||_{F}^{2}\\
h_{2}(\theta;\sigma,\phi) & =\Vert\theta-\psi\Vert^{2}+\frac{1}{2}\sum_{i,j\in\mathcal{E}}\sum_{a,b,c}\big(\sigma_{ij}(a,b,c)-\phi_{ij}(a,b,c)\big)(\theta_{c,a}^{ij}-\theta_{c,b}^{ij}).
\end{align*}

\subsection{Additional Experimental Results}

The rest of the appendix contains extra experimental results that
could not fit in the main paper.

\begin{figure}[bh]
\begin{centering}
\textbf{\includegraphics[width=0.4\textwidth]{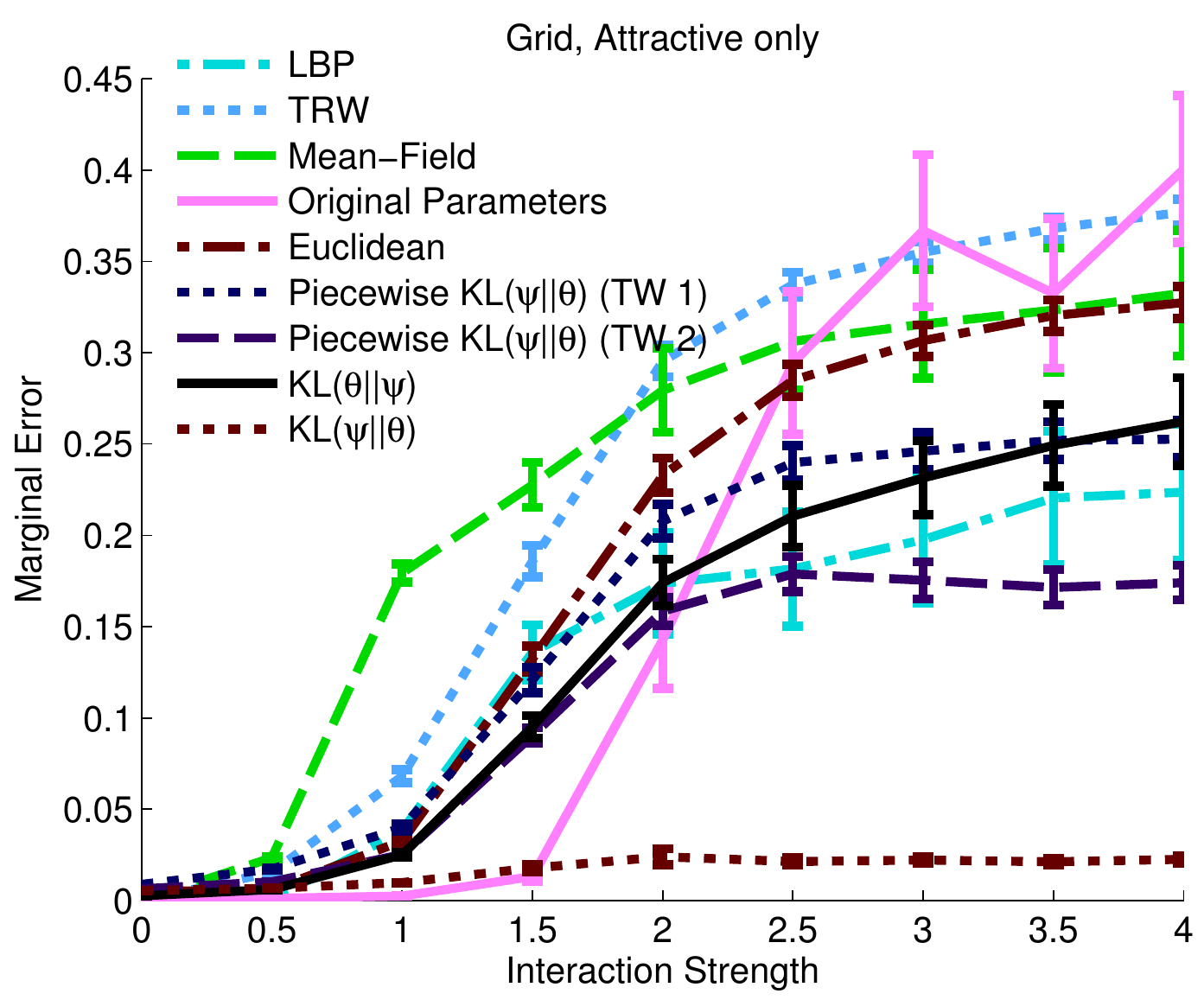}\includegraphics[width=0.4\textwidth]{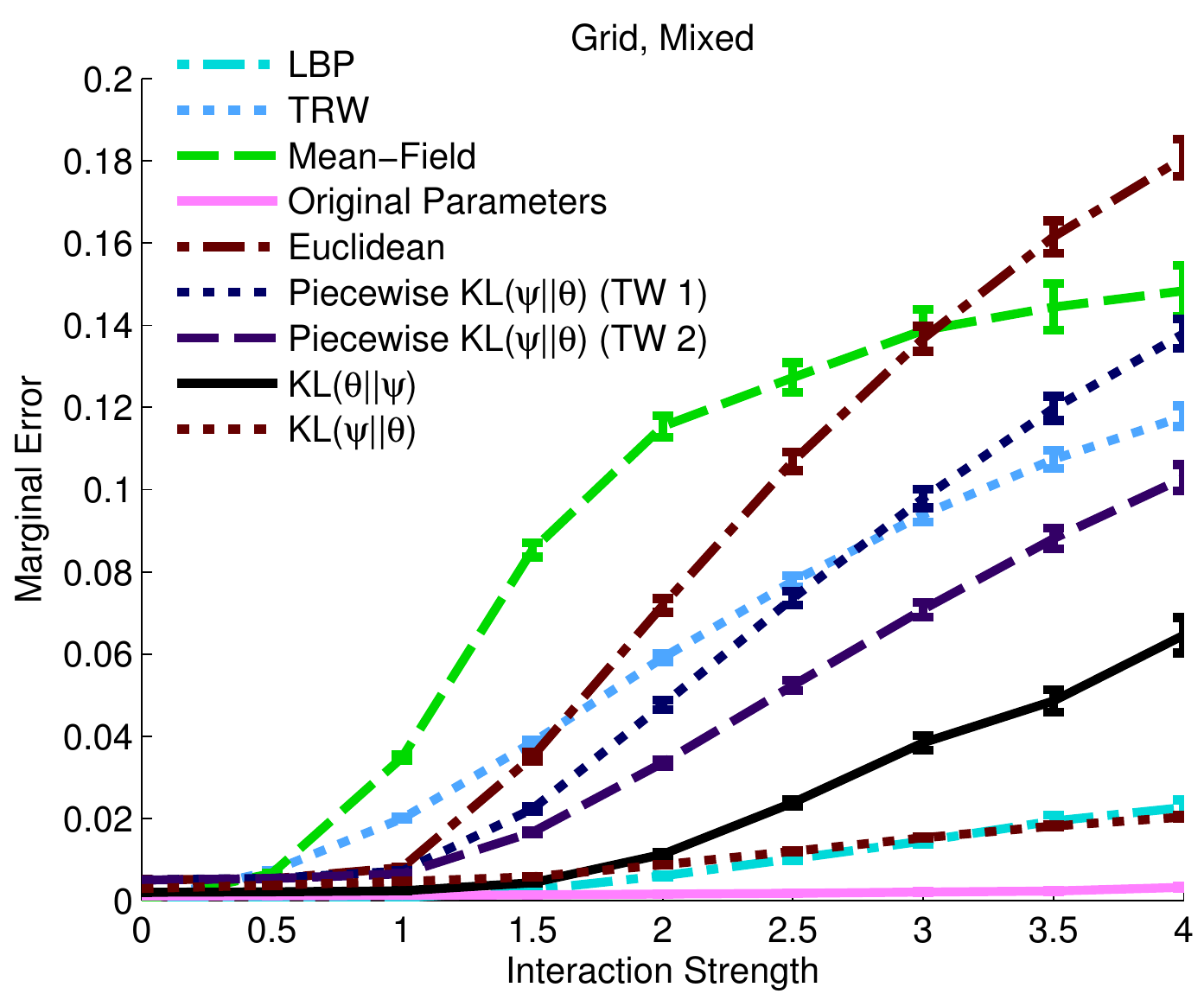}}
\par\end{centering}

\centering{}\protect\caption{Marginal error vs. interaction strength for 3-state Potts models on
grids. Here, the intractable divergence $KL(\psi\Vert\theta)$ is
included for reference. With attractive interactions, the best-performing
tractable algorithm uses the piecewise divergence, while with mixed
interactions, loopy BP and simply sampling using the original parameters
both perform extremely well.}
\end{figure}

\newpage{}

\begin{figure}
\begin{centering}
\textbf{\includegraphics[width=0.4\textwidth]{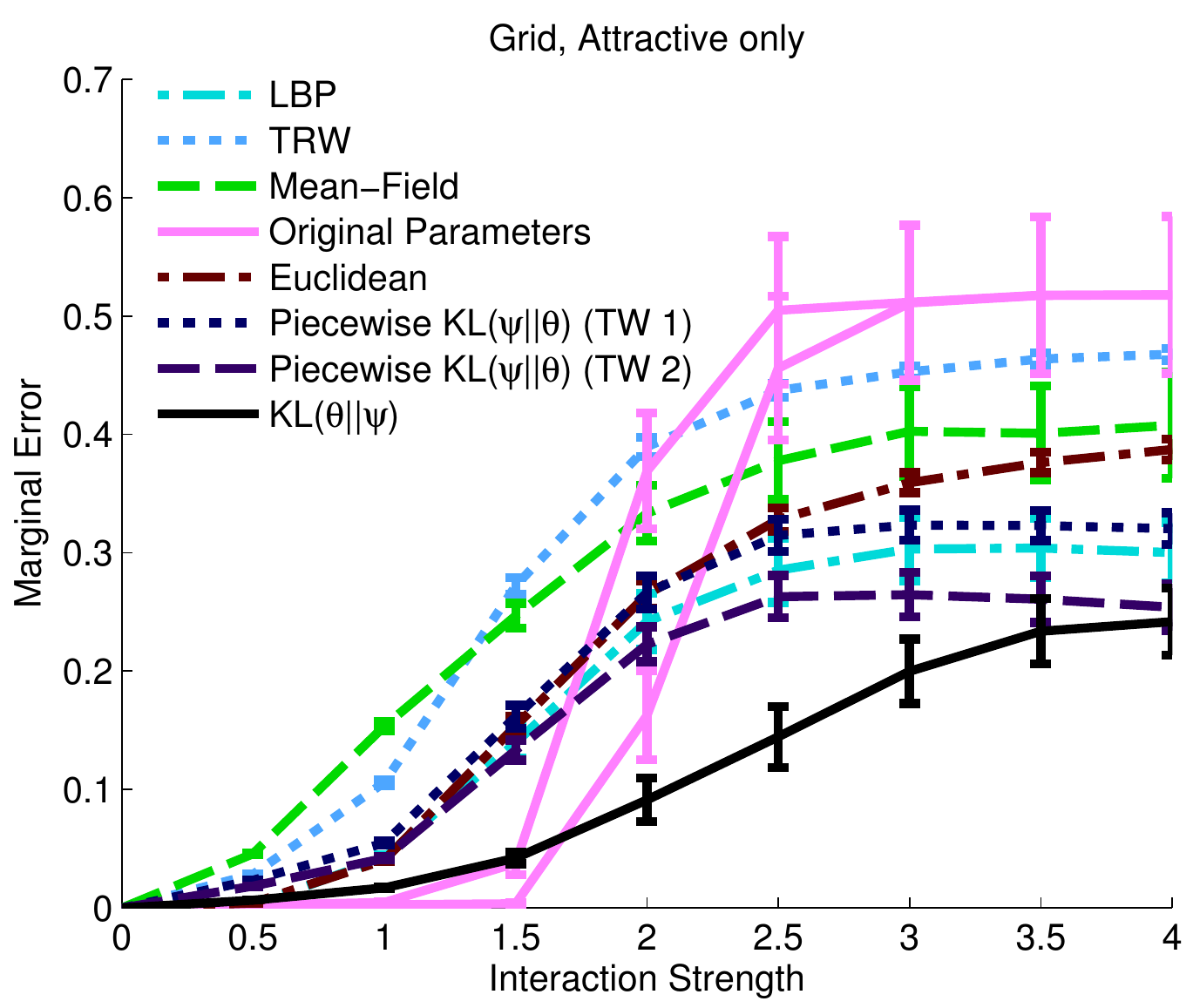}\includegraphics[width=0.4\textwidth]{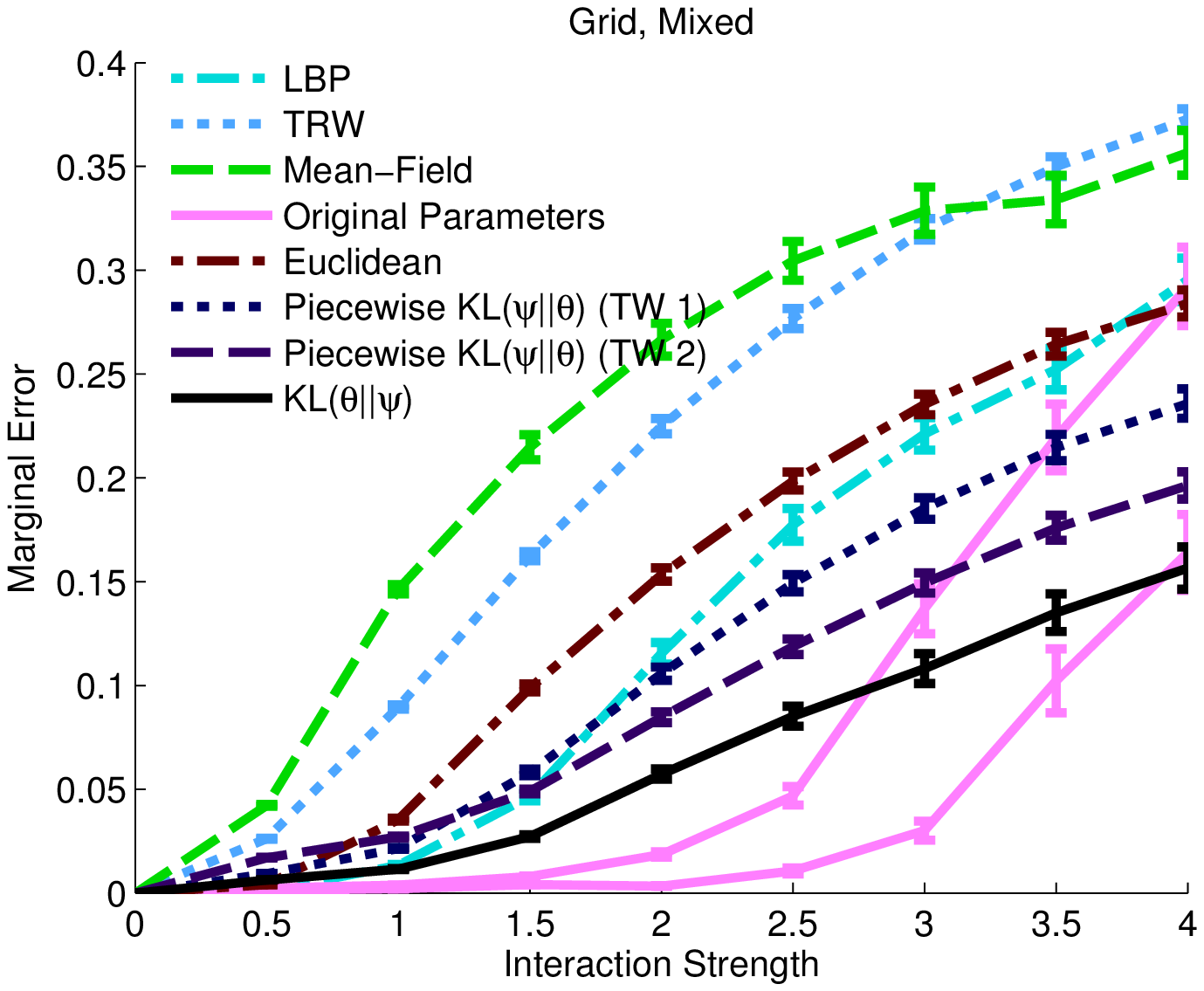}}
\par\end{centering}

\centering{}\protect\caption{Marginal error vs. interaction strength for Ising models on grids}
\end{figure}

\newpage{}

\begin{figure}
\centering \includegraphics[scale=0.4]{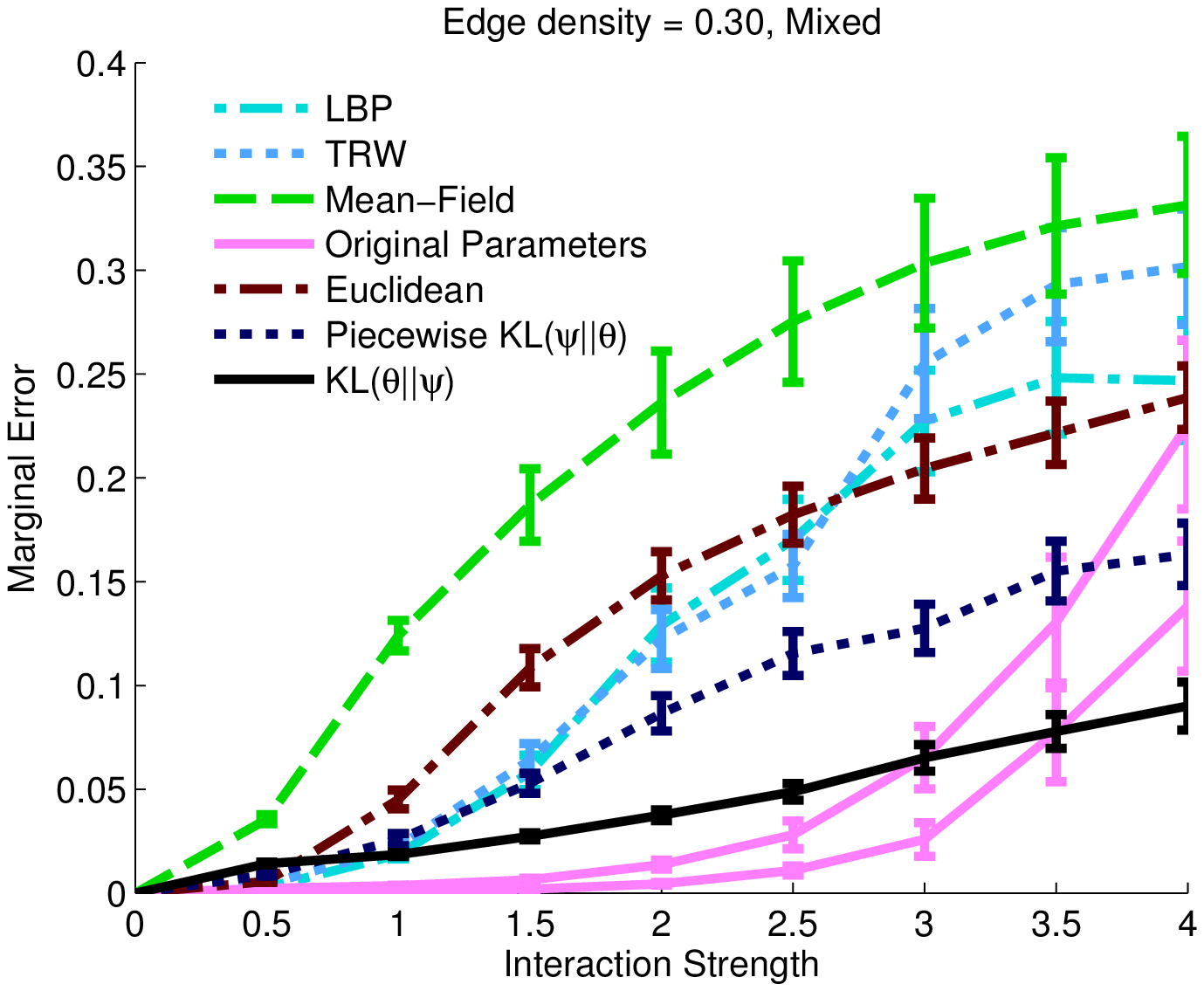}
\includegraphics[scale=0.4]{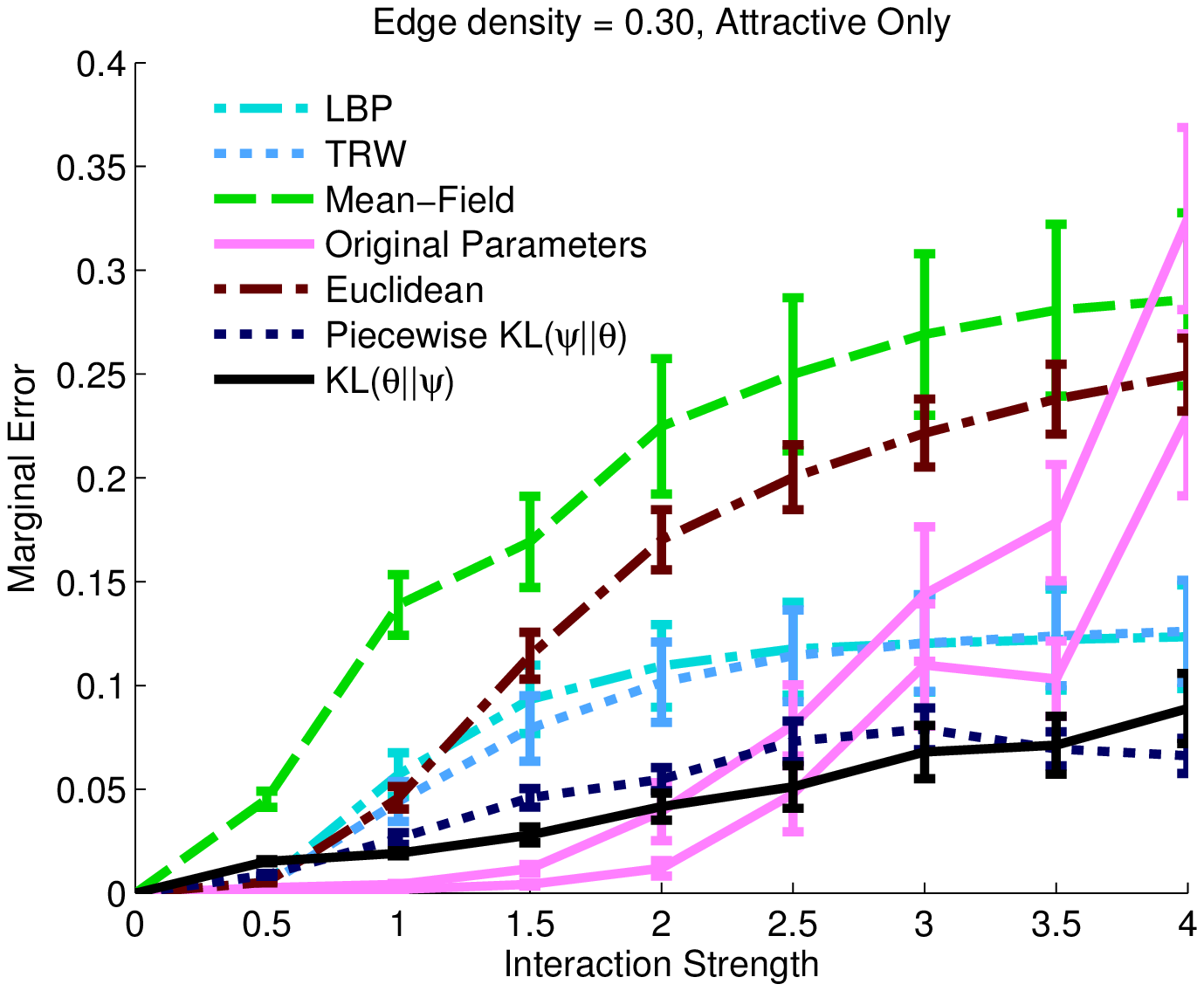}

\includegraphics[scale=0.4]{figures/strength_0\lyxdot 50} \includegraphics[scale=0.4]{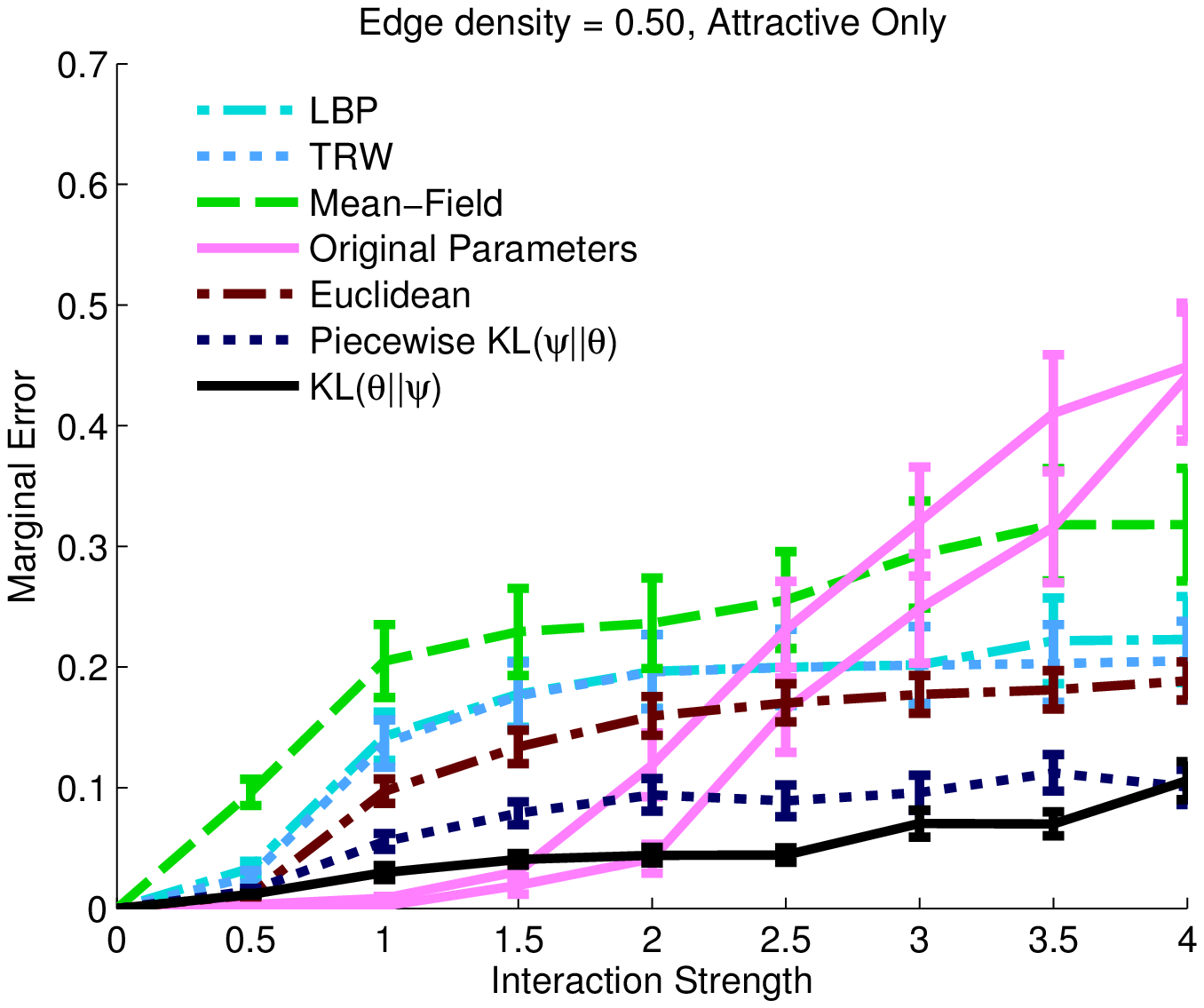}

\includegraphics[scale=0.4]{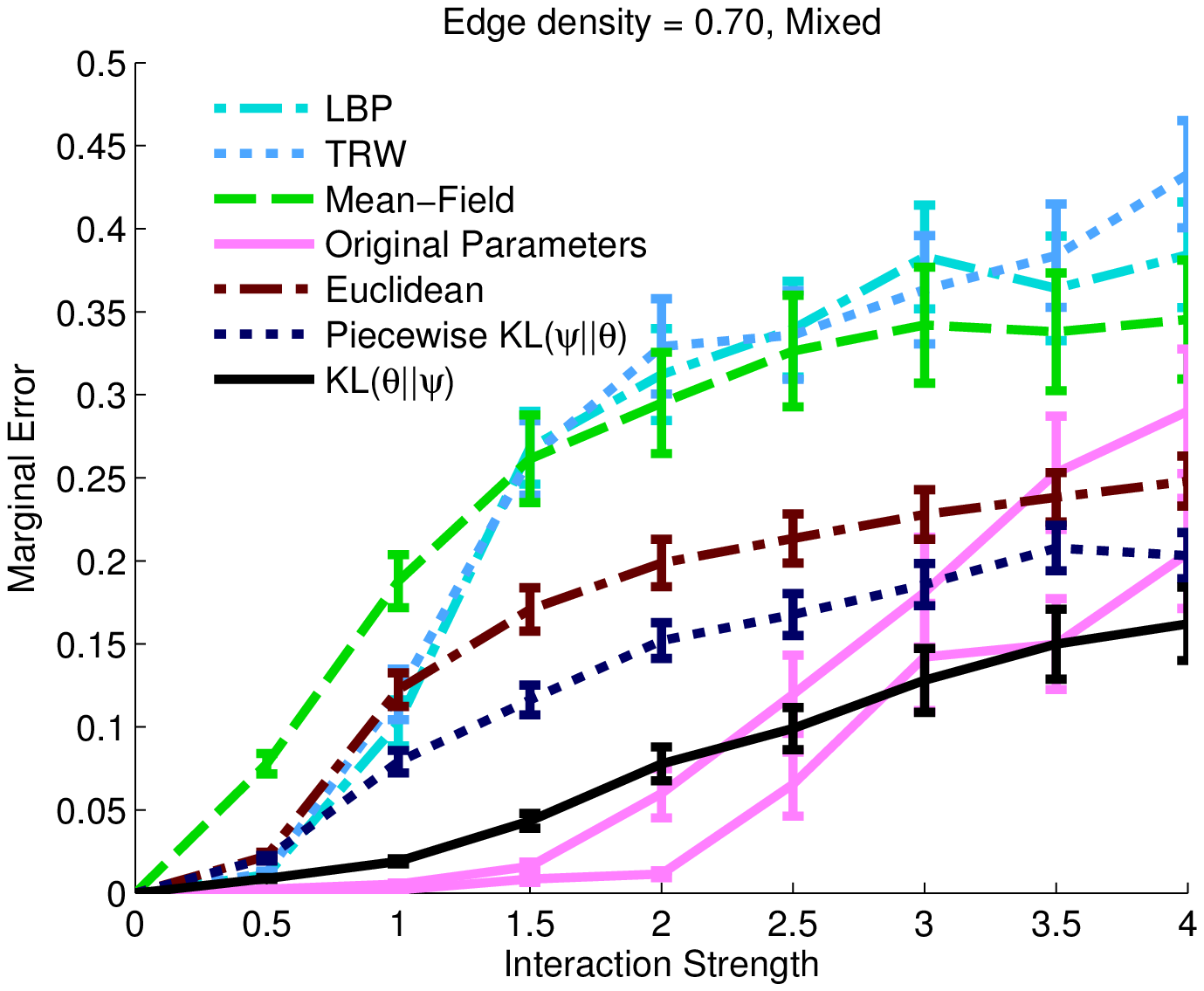} \includegraphics[scale=0.4]{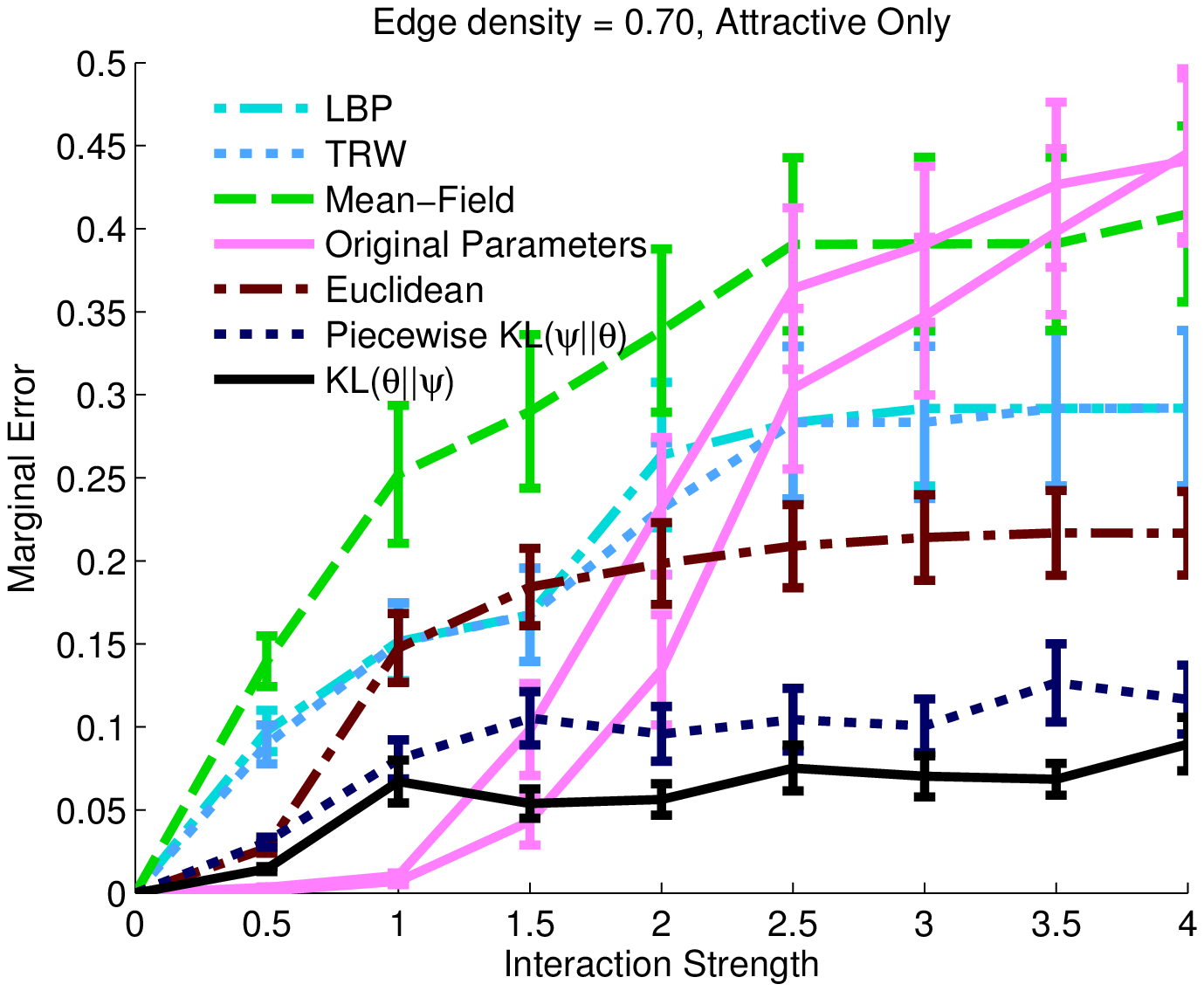}

\protect\caption{Marginal error v.s. interaction strength for Ising models on random
graphs}
\end{figure}

\newpage{}
\begin{figure}
\centering \includegraphics[scale=0.4]{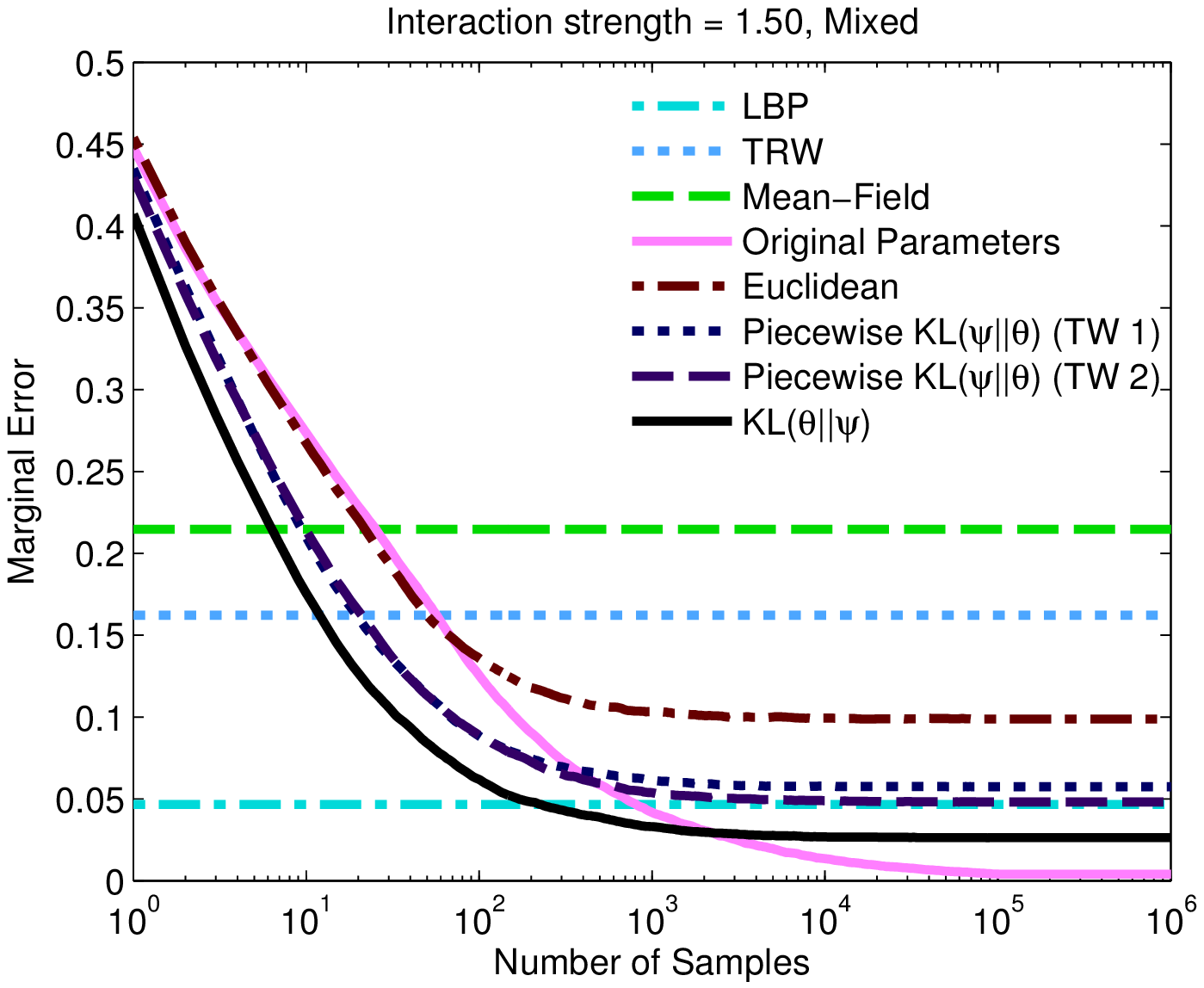}
\includegraphics[scale=0.4]{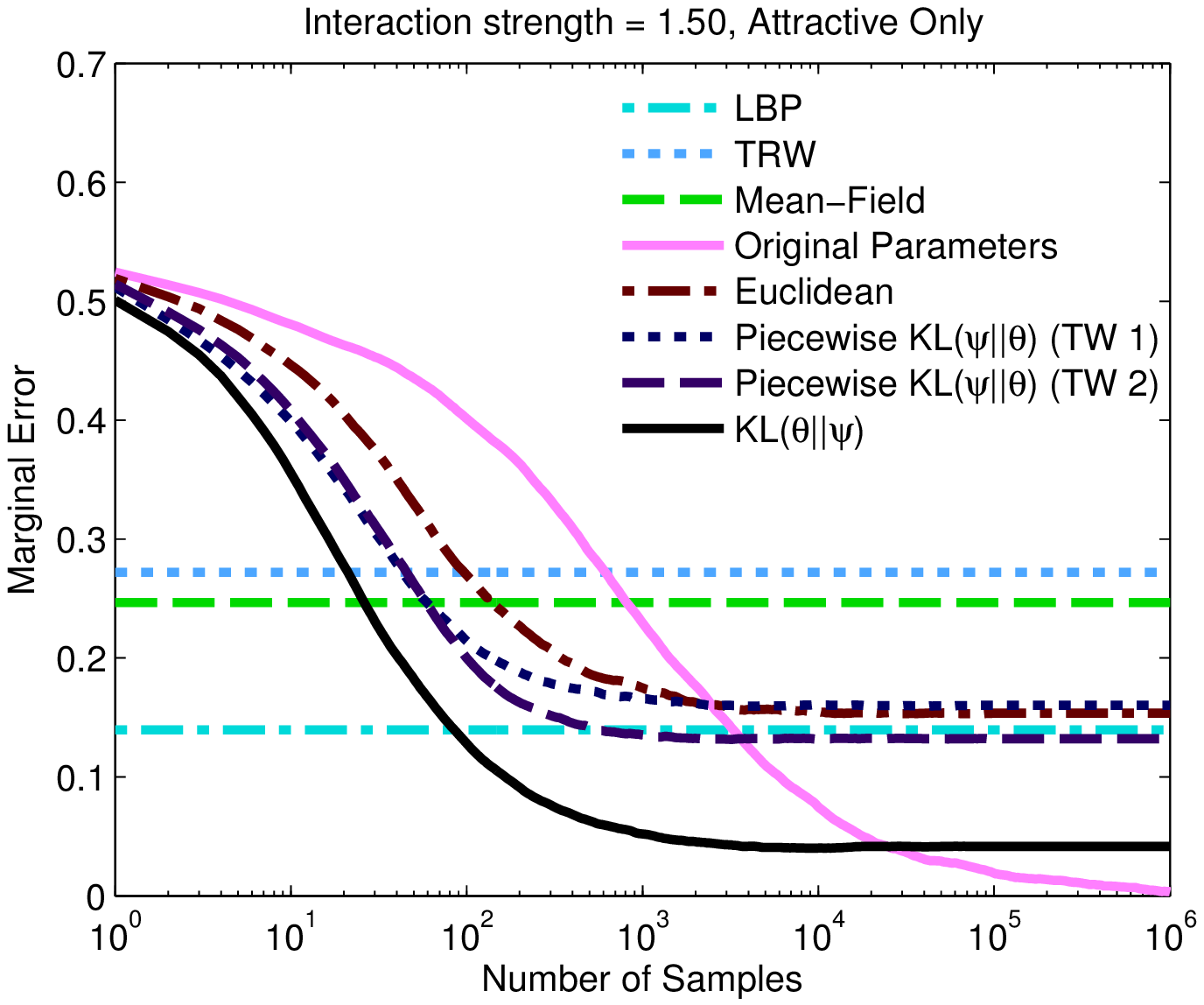}

\includegraphics[scale=0.4]{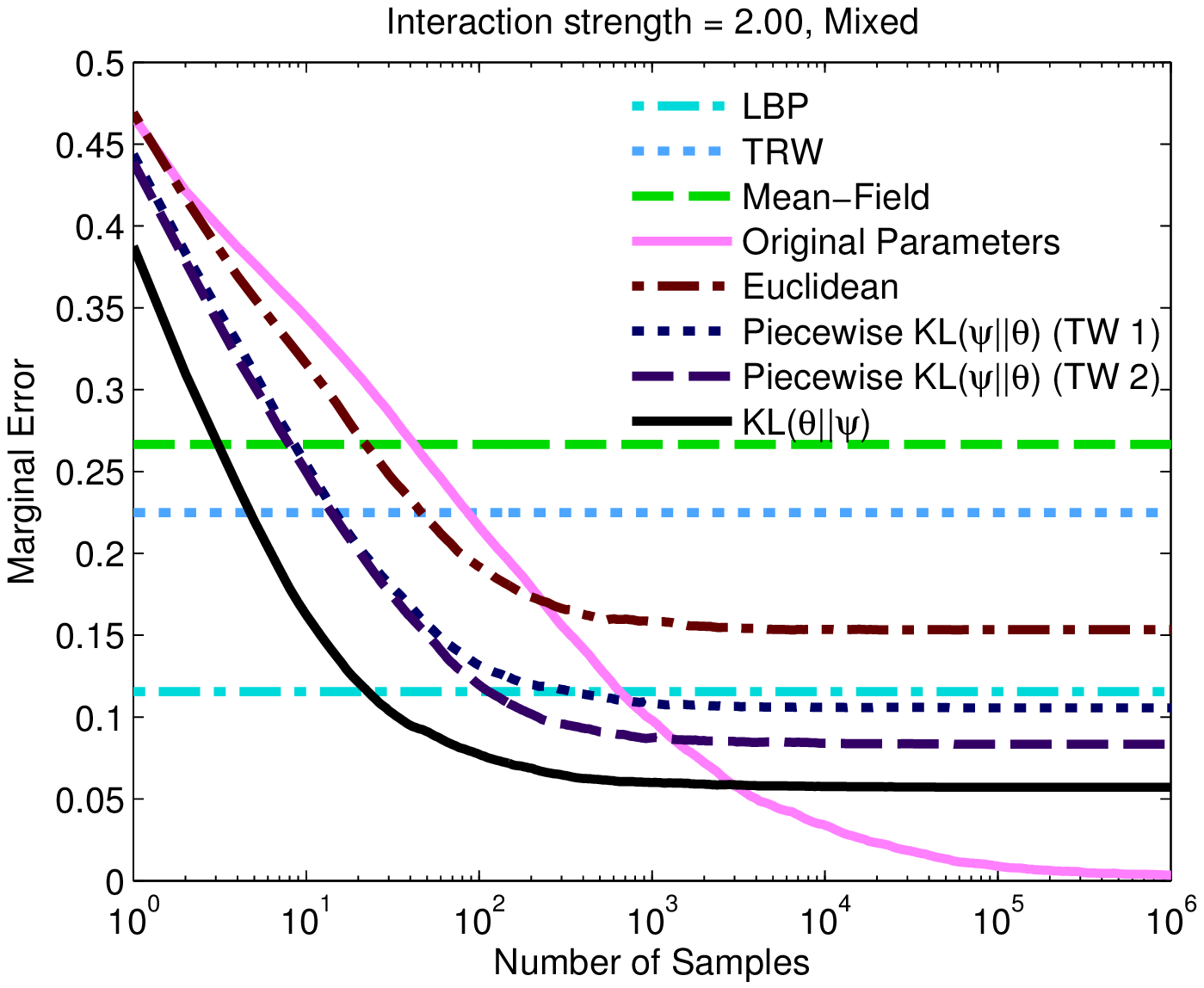} \includegraphics[scale=0.4]{figures/time_a_2\lyxdot 00}

\includegraphics[scale=0.4]{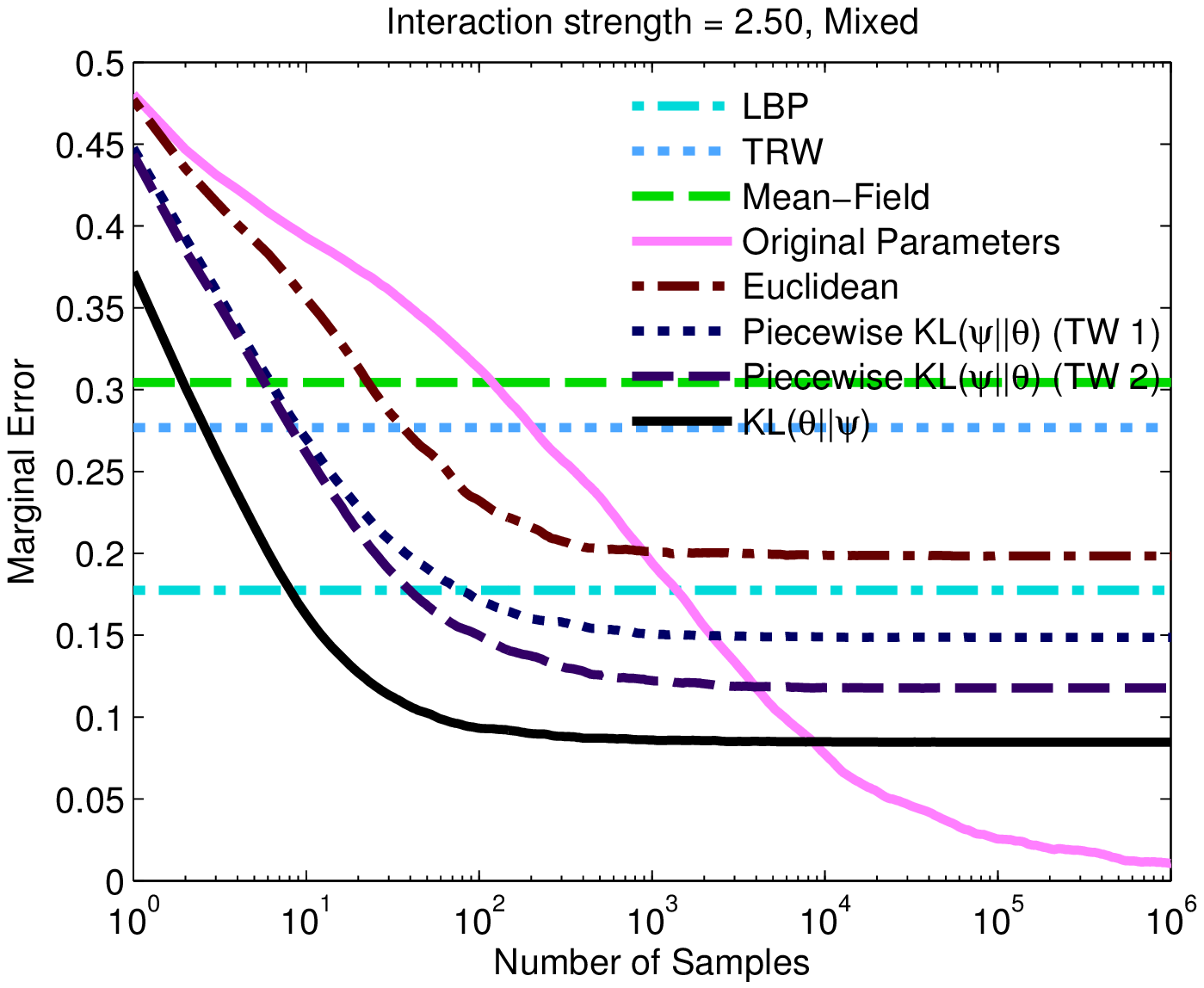} \includegraphics[scale=0.4]{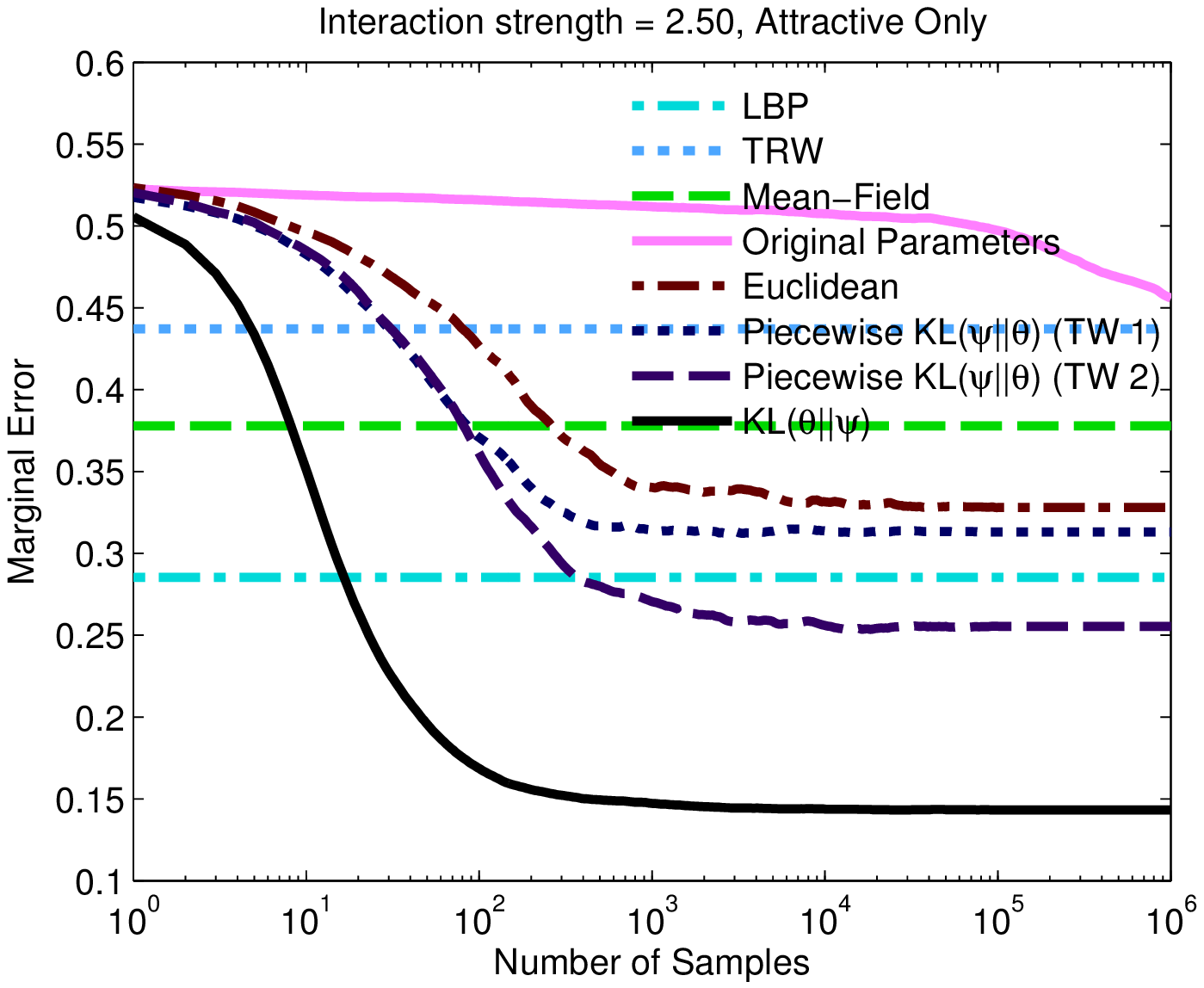}

\includegraphics[scale=0.4]{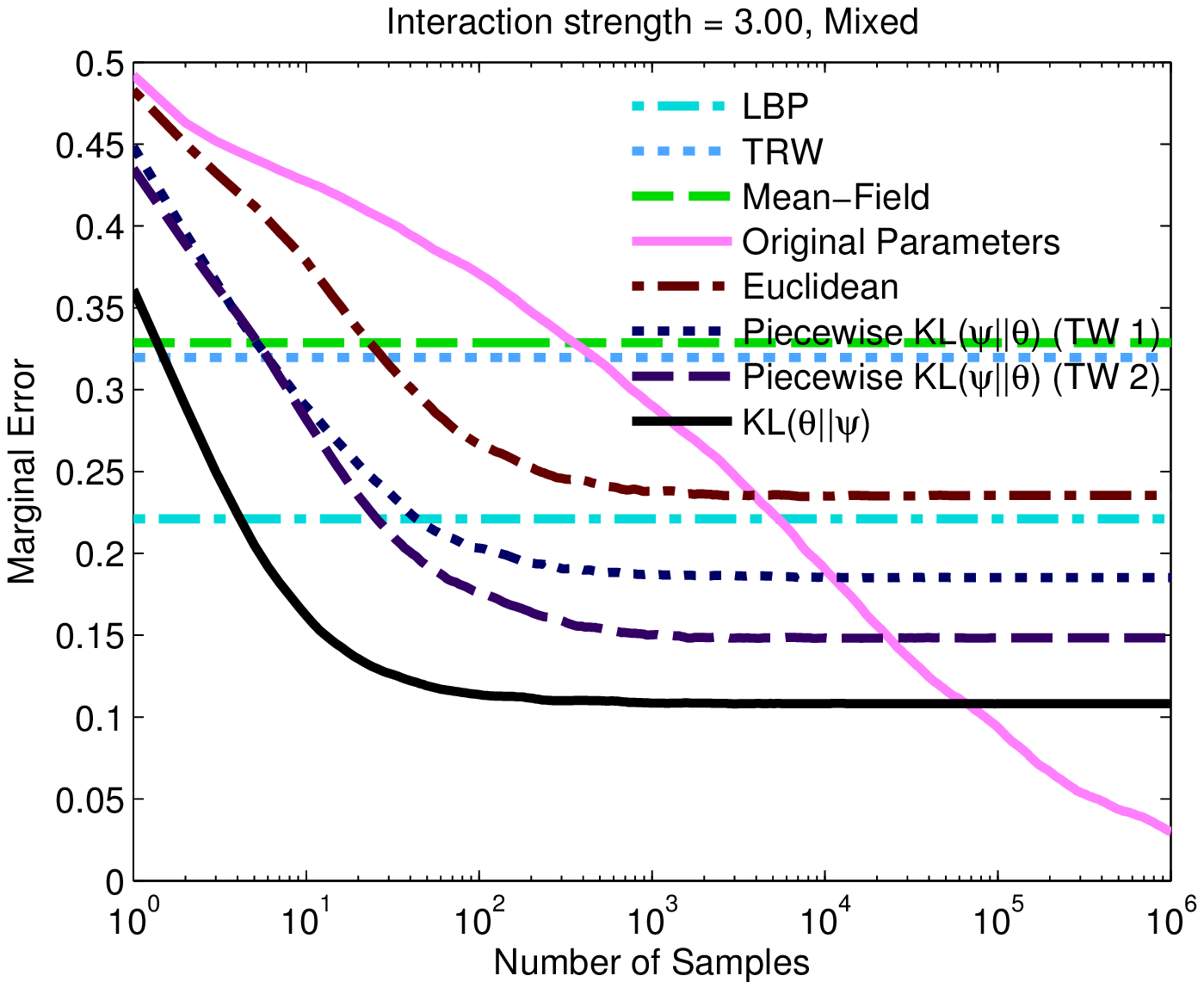} \includegraphics[scale=0.4]{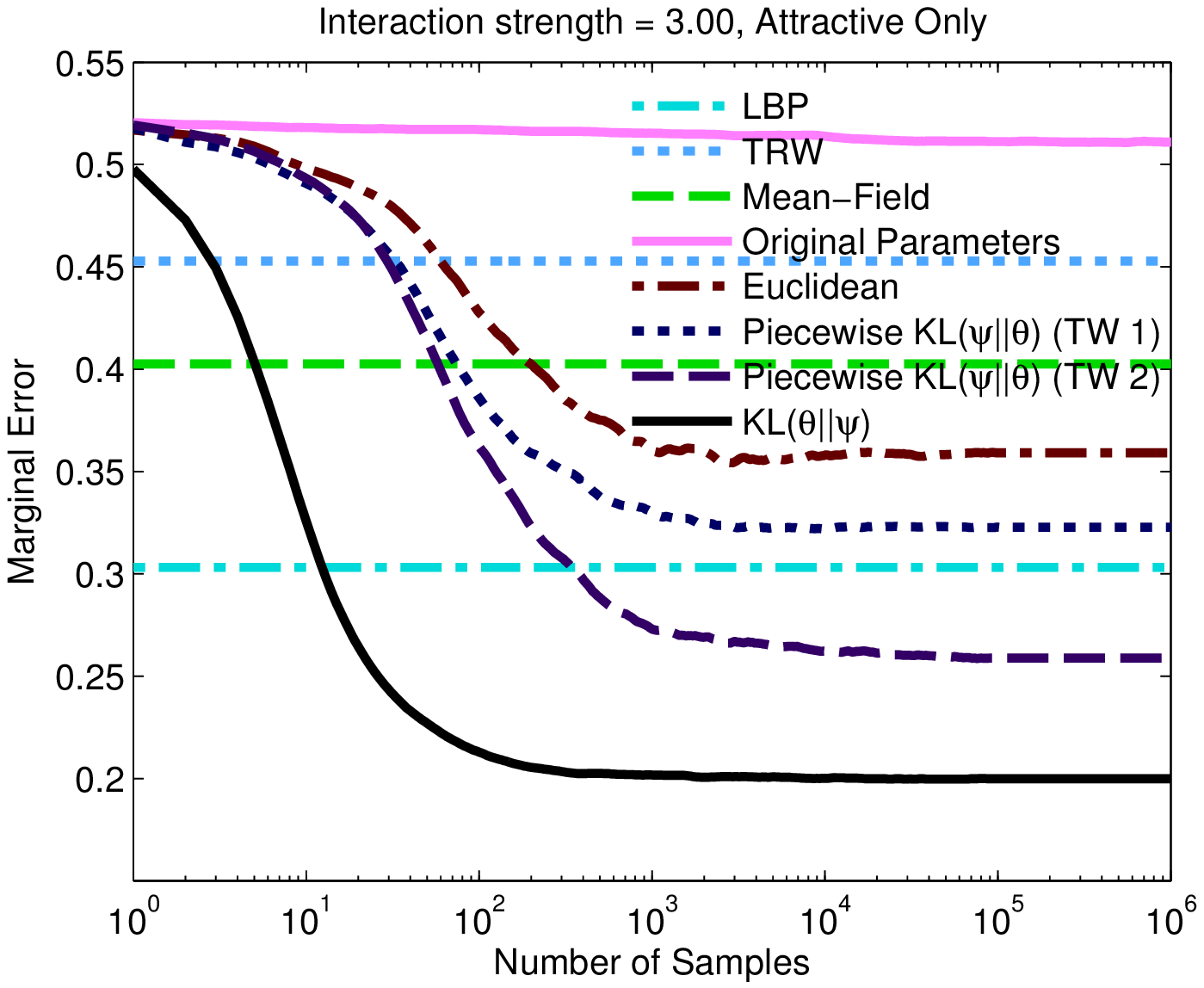}

\includegraphics[scale=0.4]{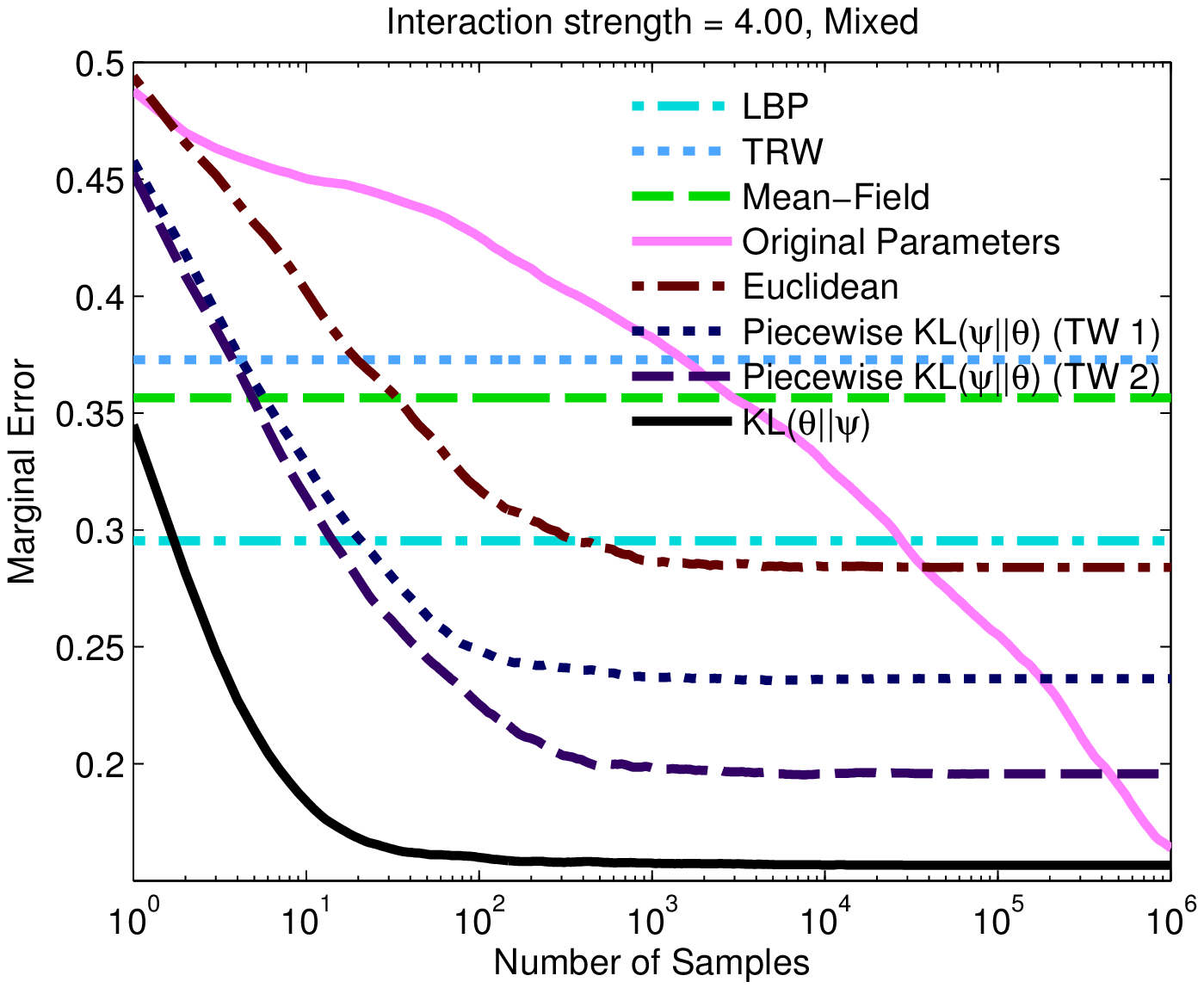} \includegraphics[scale=0.4]{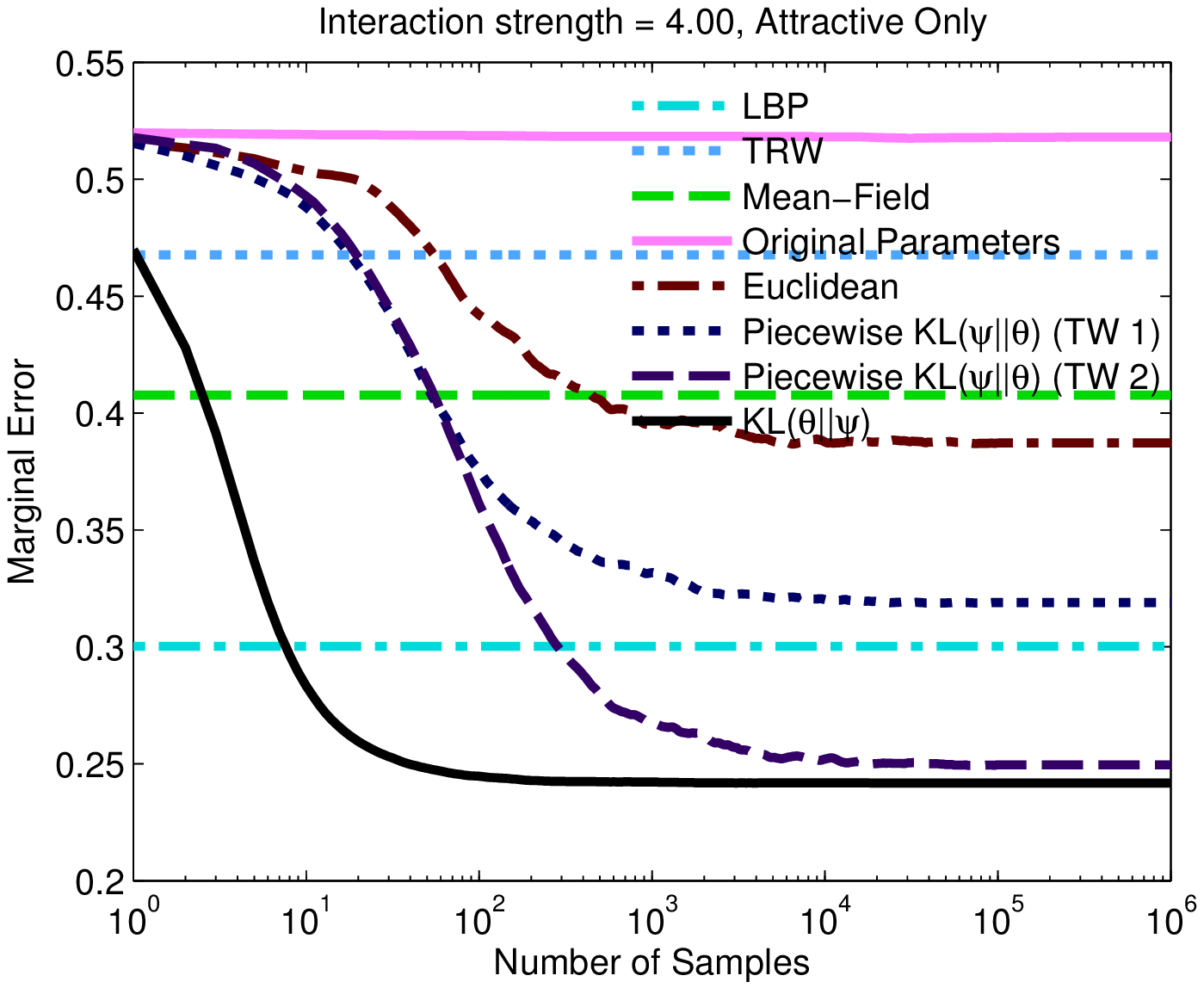}

\protect\caption{Marginal error v.s. number of samples for Ising models on grids}
\end{figure}

\newpage{}
\begin{figure}
\centering \includegraphics[scale=0.4]{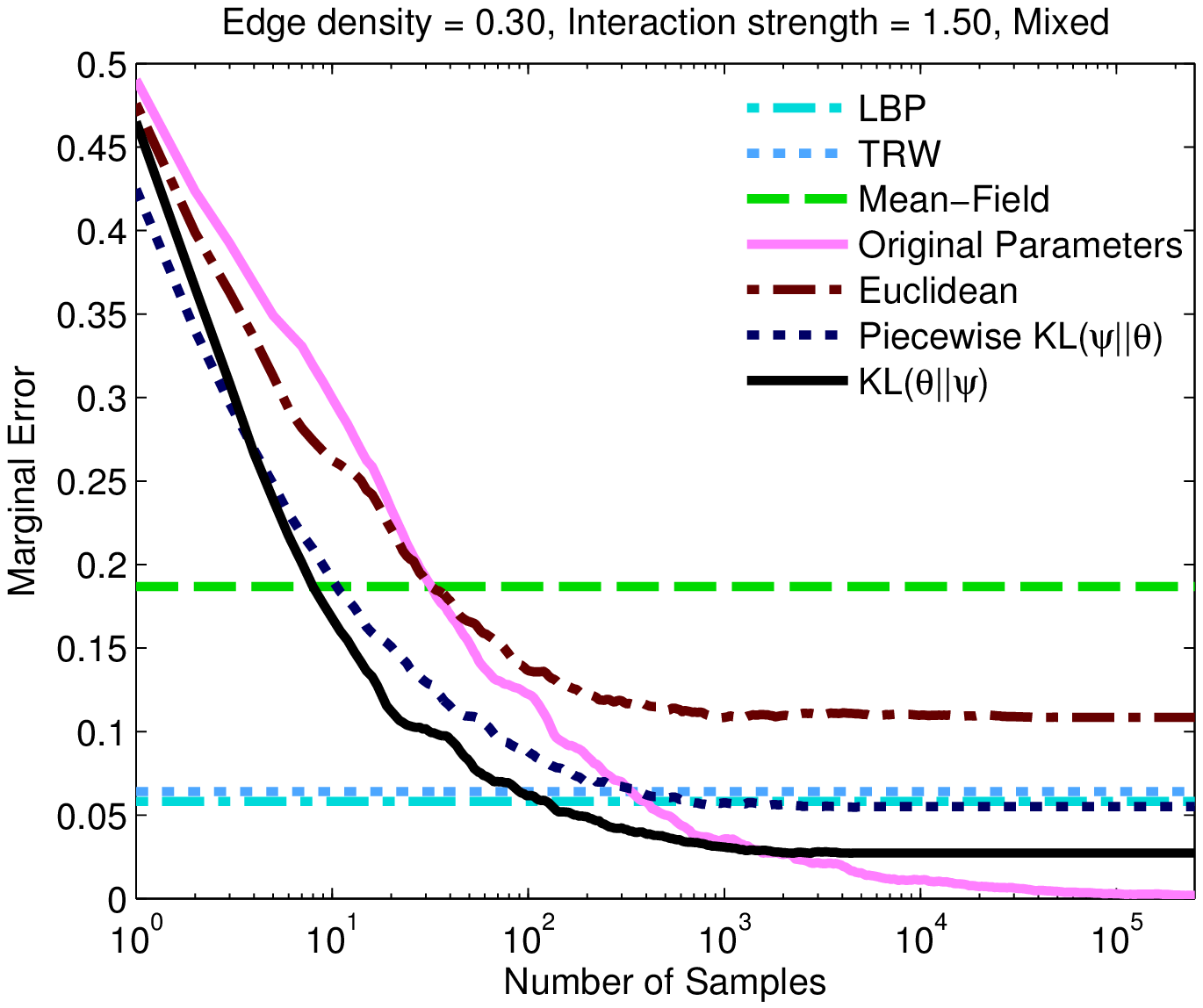}
\includegraphics[scale=0.4]{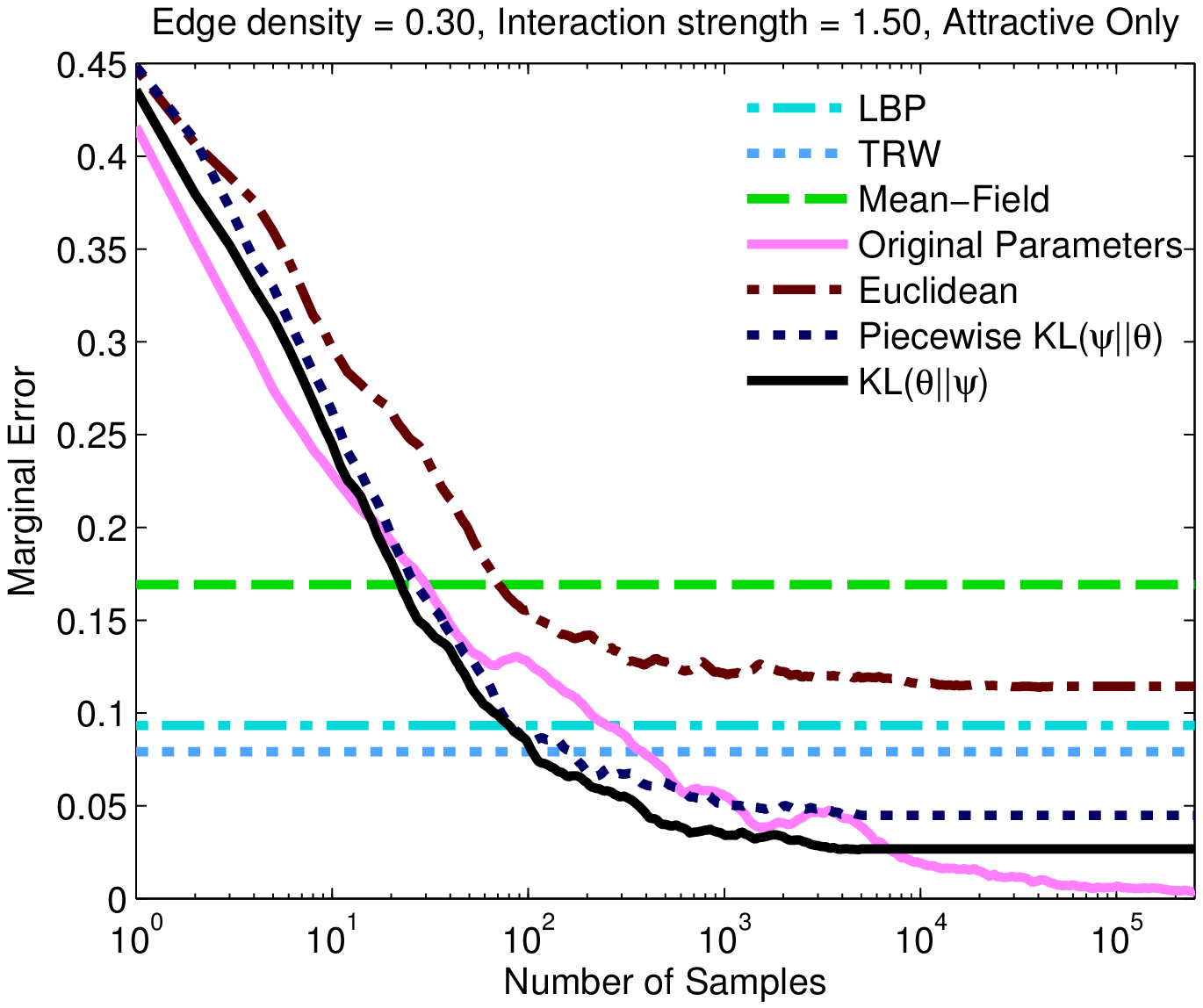}

\includegraphics[scale=0.4]{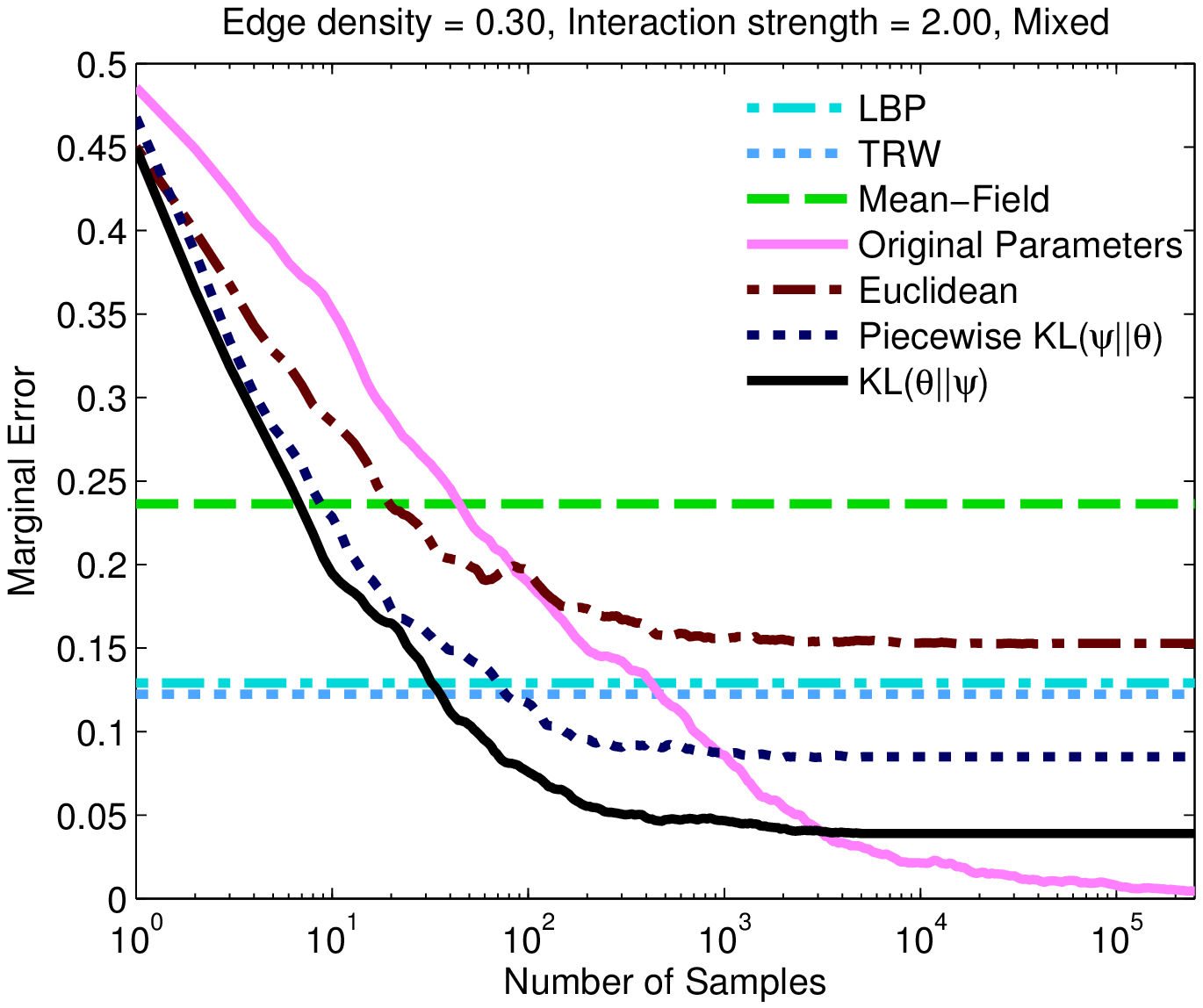}
\includegraphics[scale=0.4]{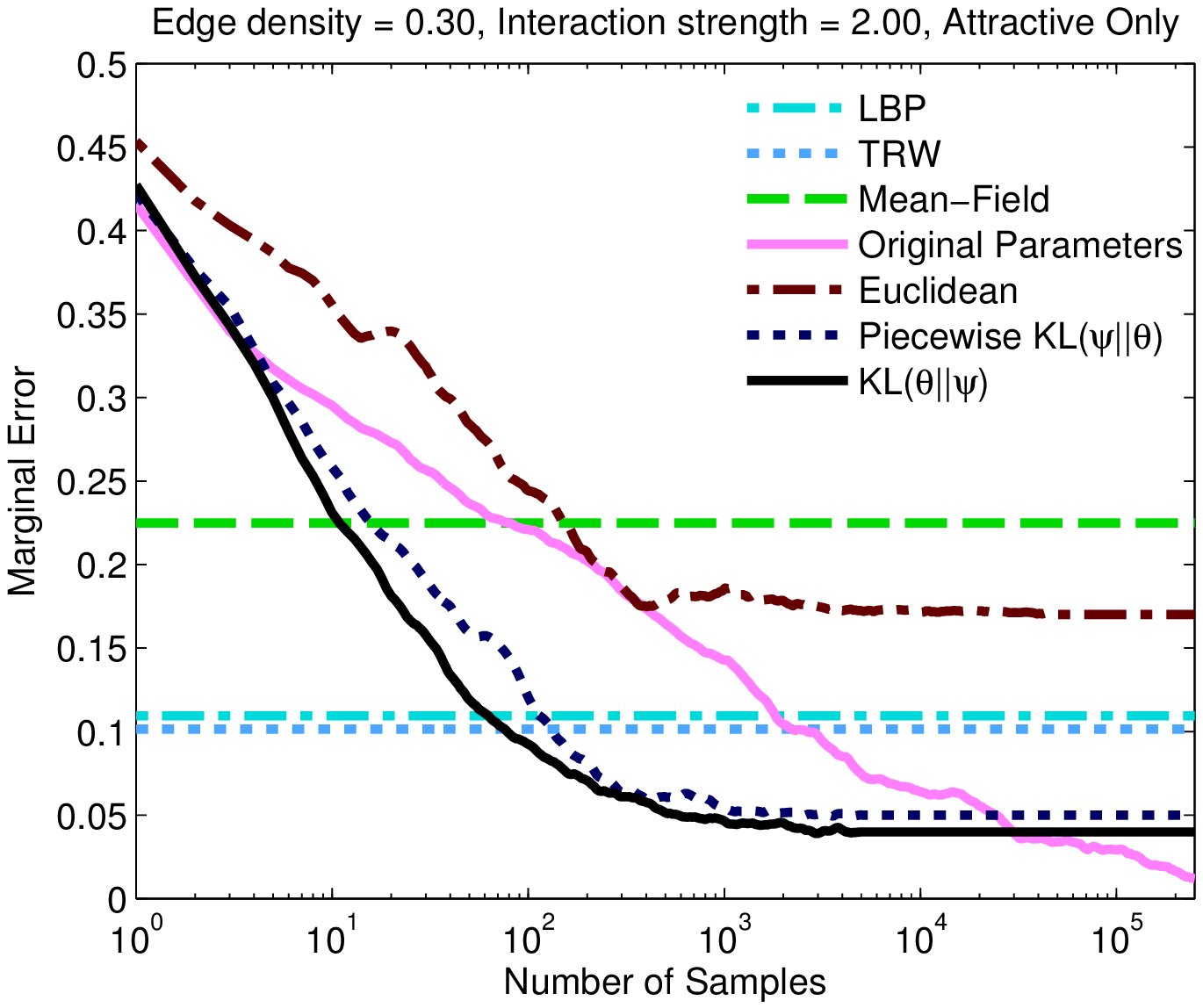}

\includegraphics[scale=0.4]{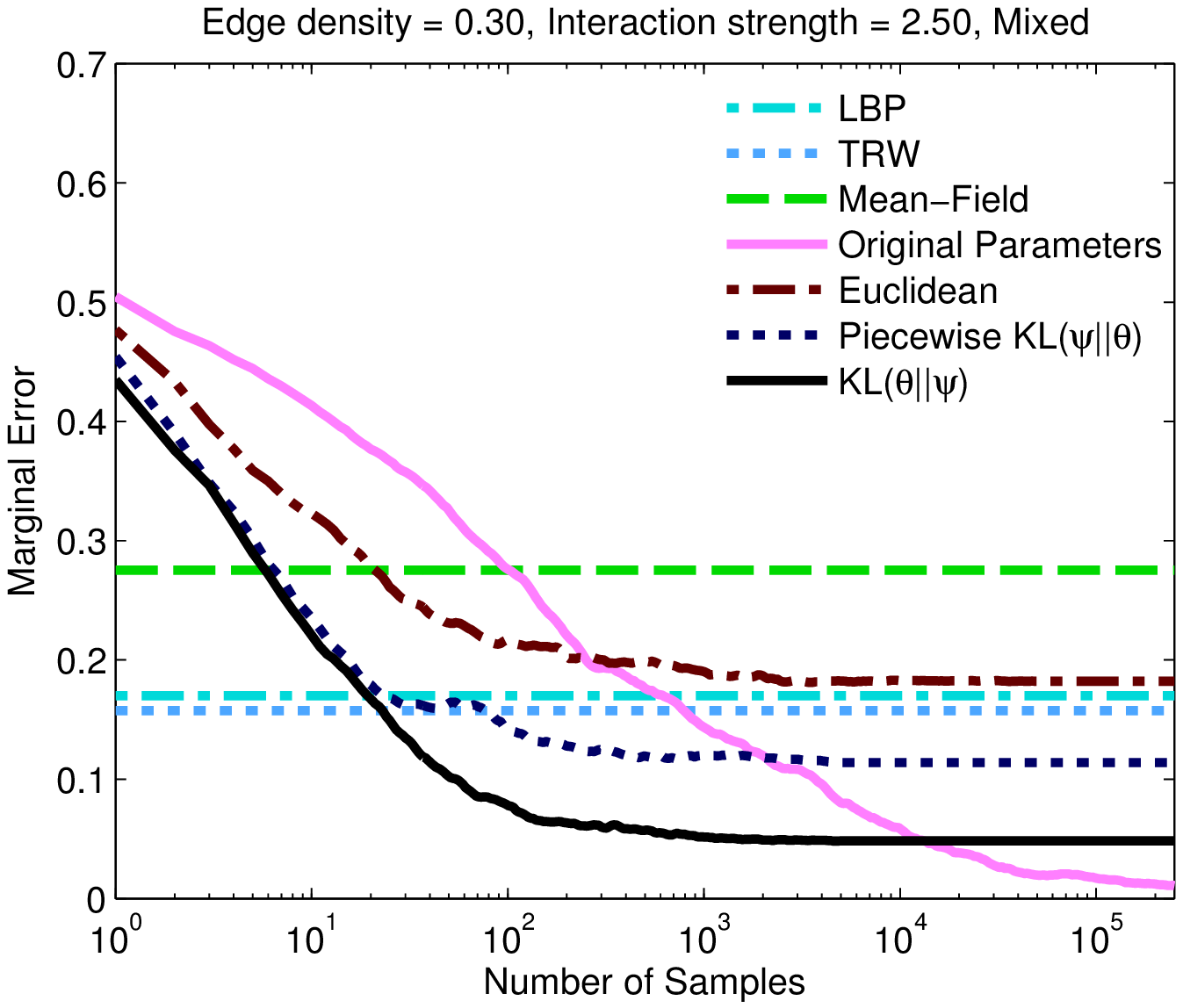}
\includegraphics[scale=0.4]{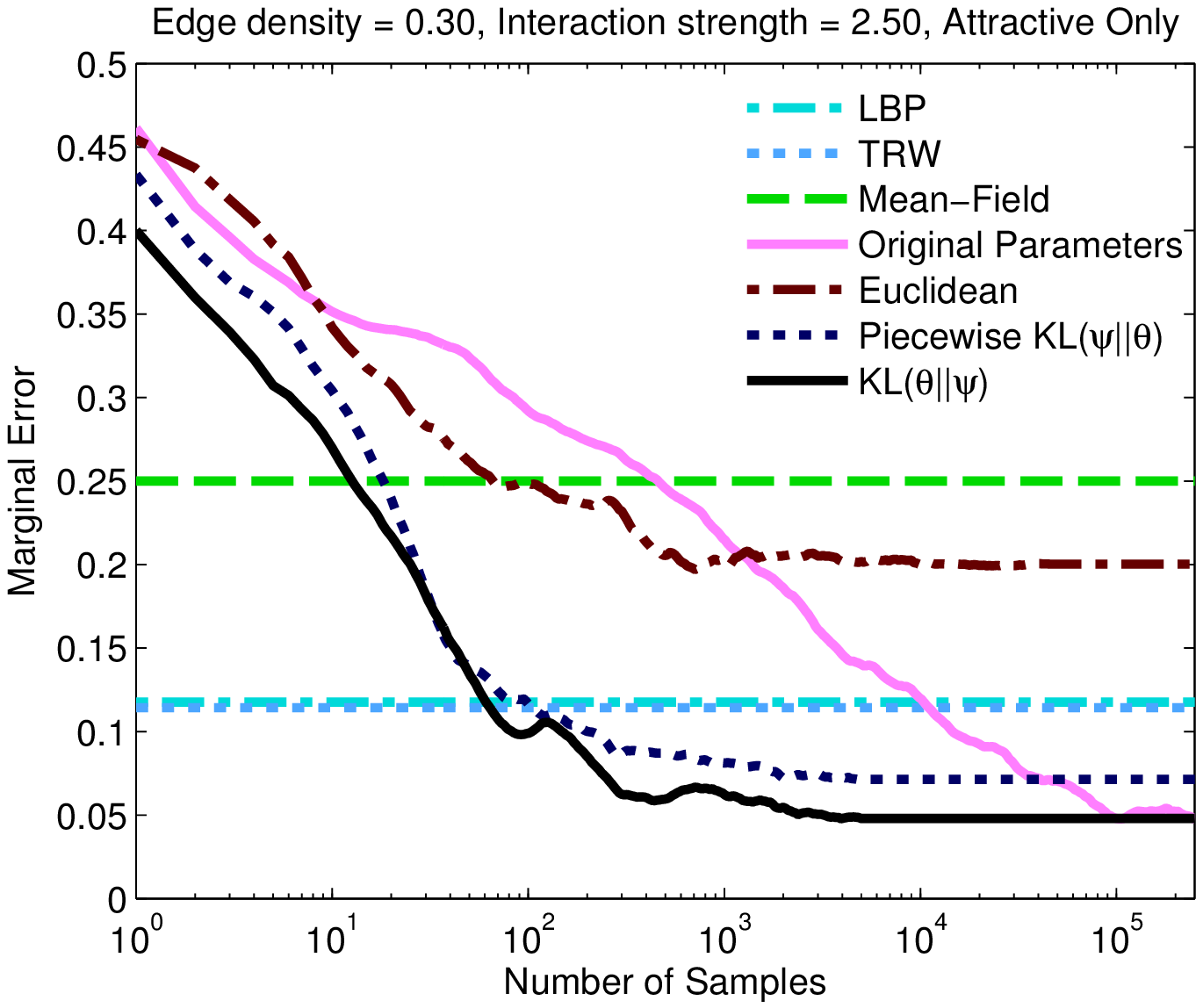}

\includegraphics[scale=0.4]{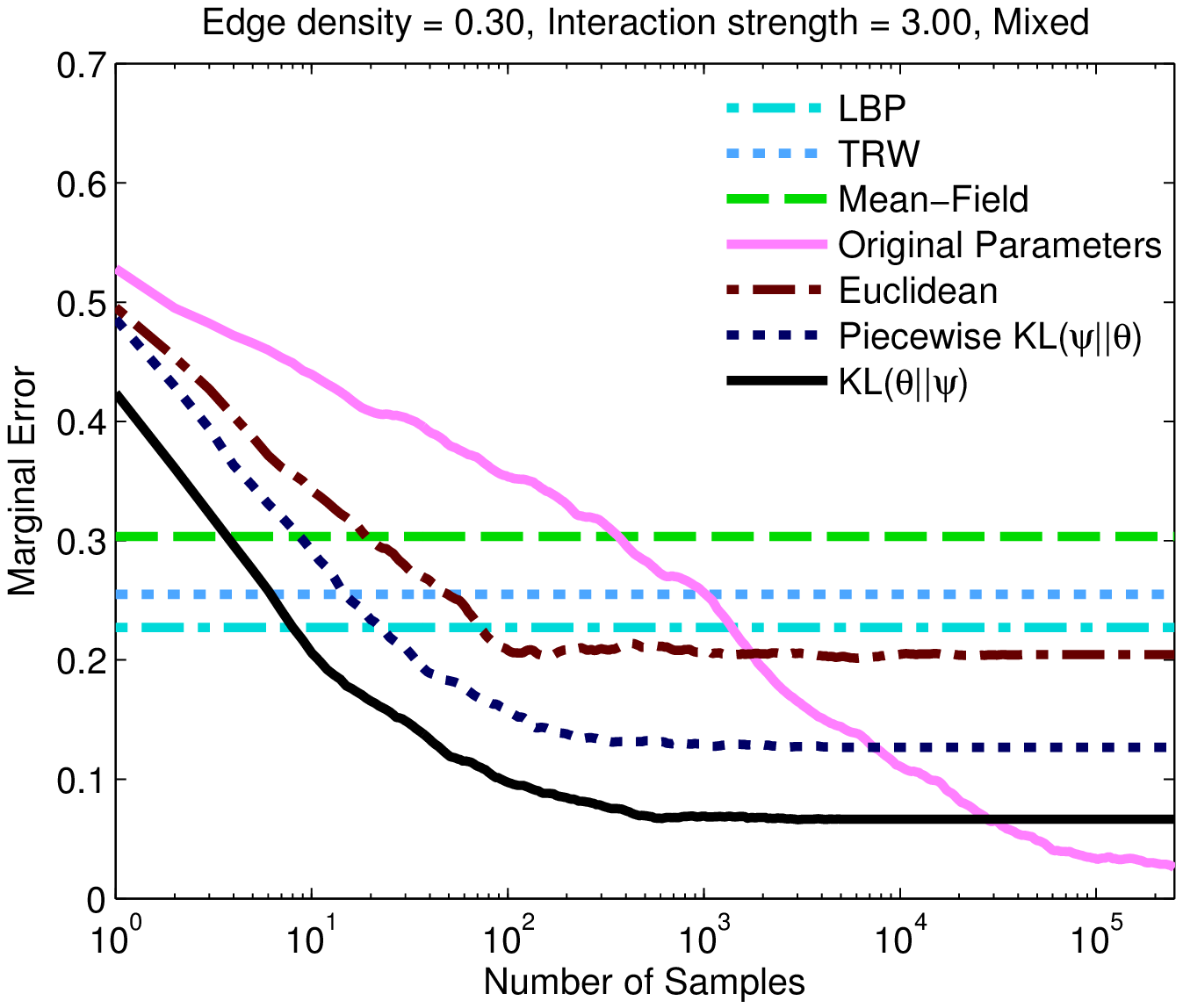}
\includegraphics[scale=0.4]{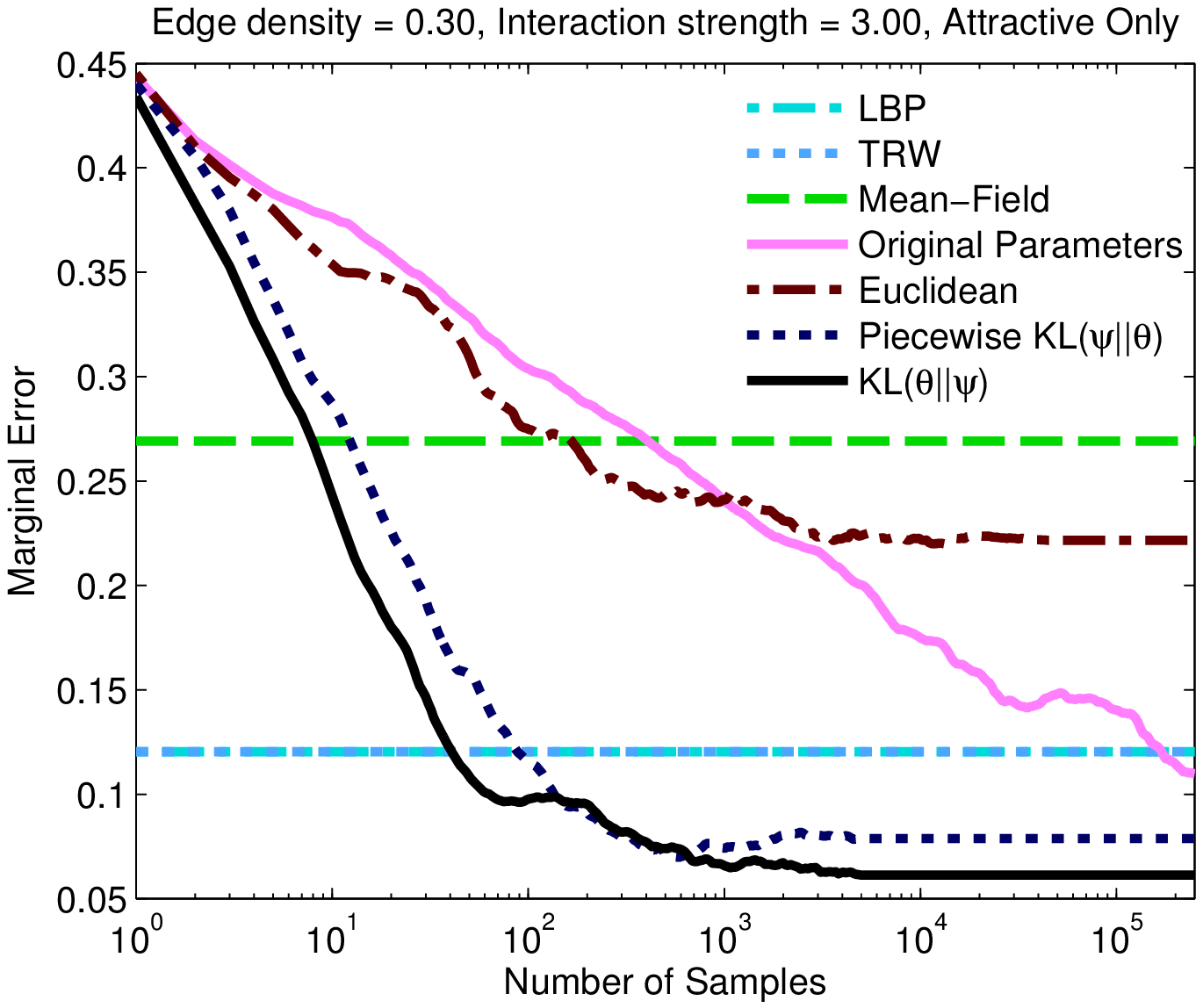}

\includegraphics[scale=0.4]{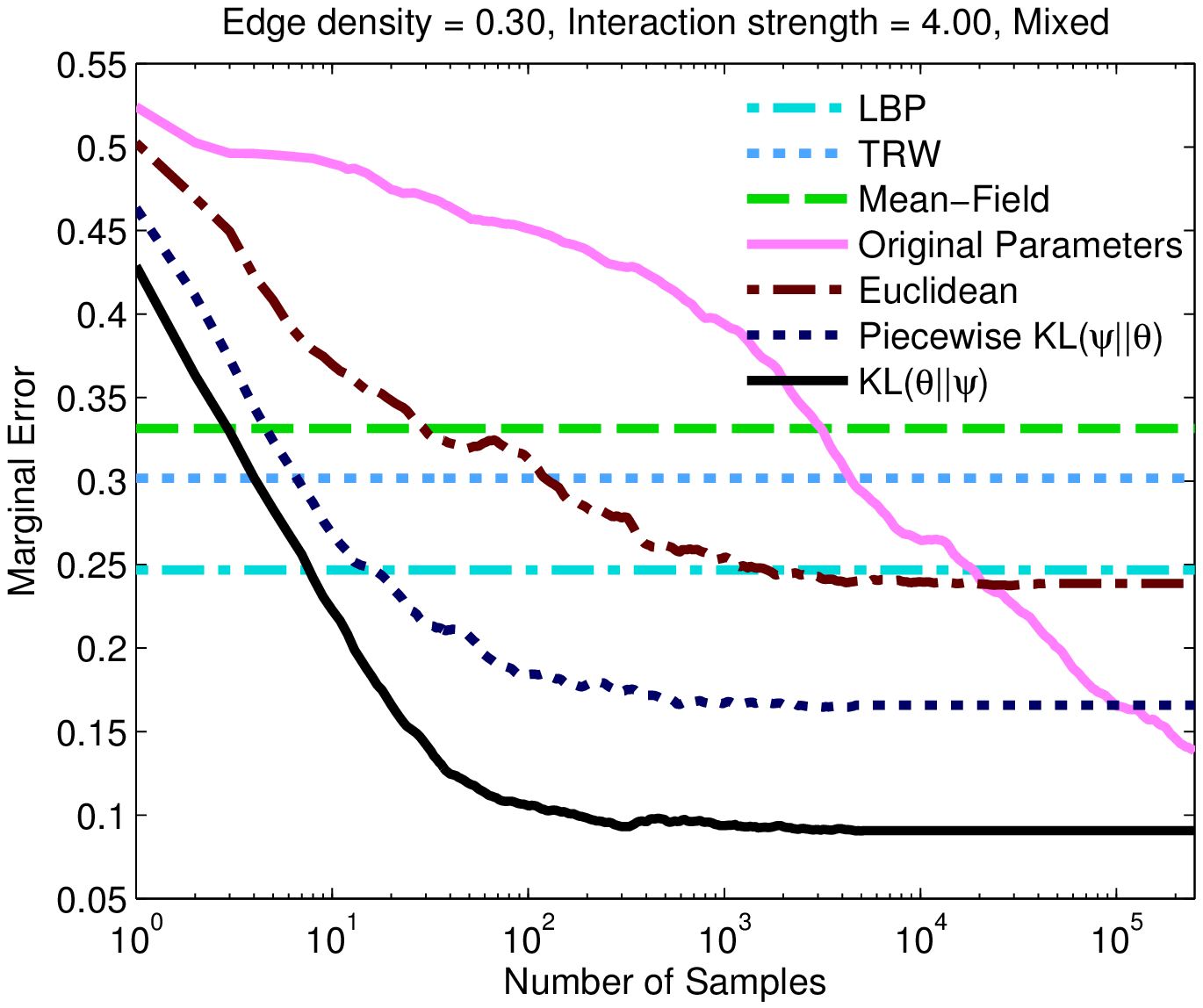}
\includegraphics[scale=0.4]{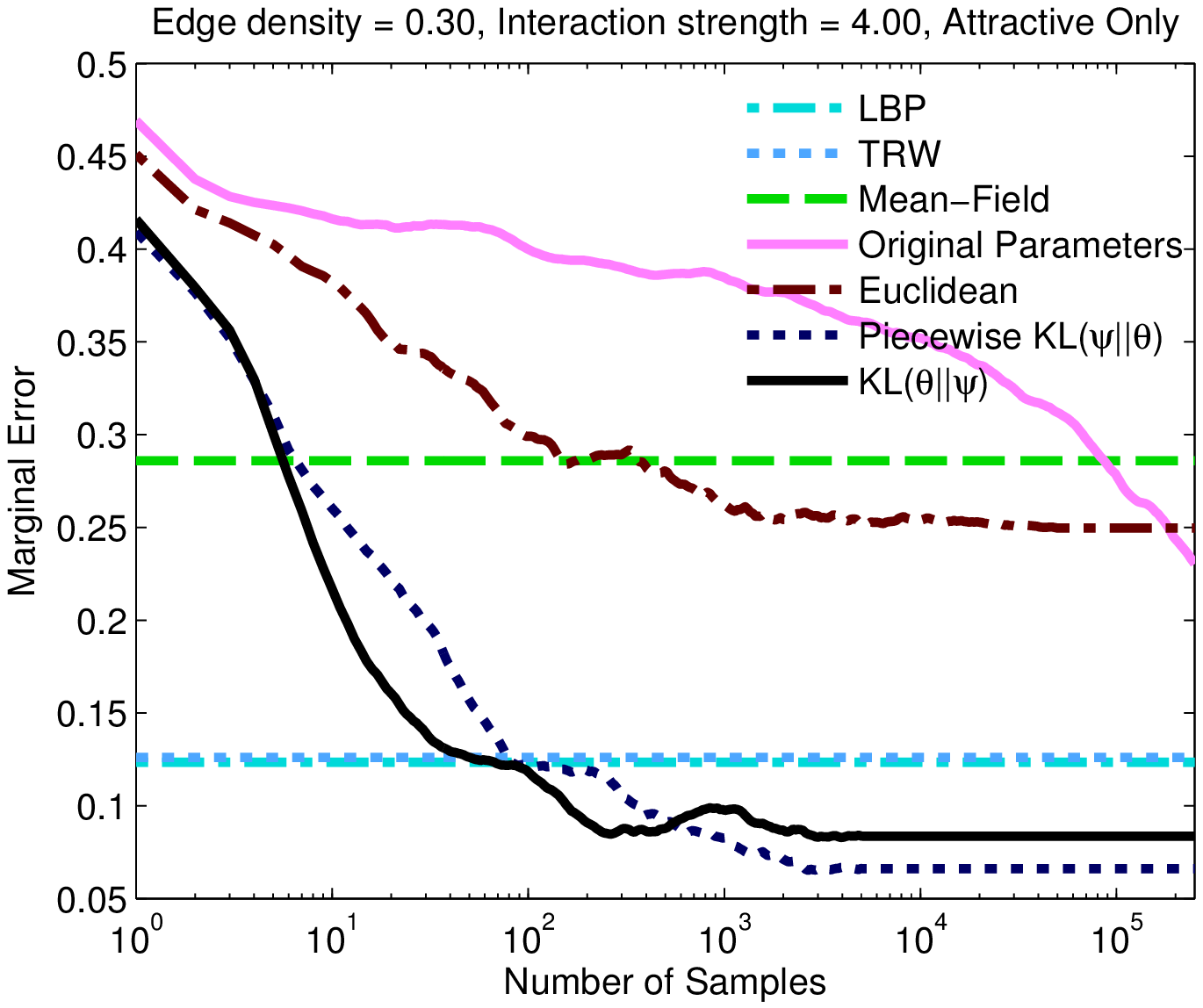}
\protect\caption{Marginal error v.s. number of samples for Ising models on random graphs
with edge density 0.3}
\end{figure}

\newpage{}
\begin{figure}
\centering \includegraphics[scale=0.4]{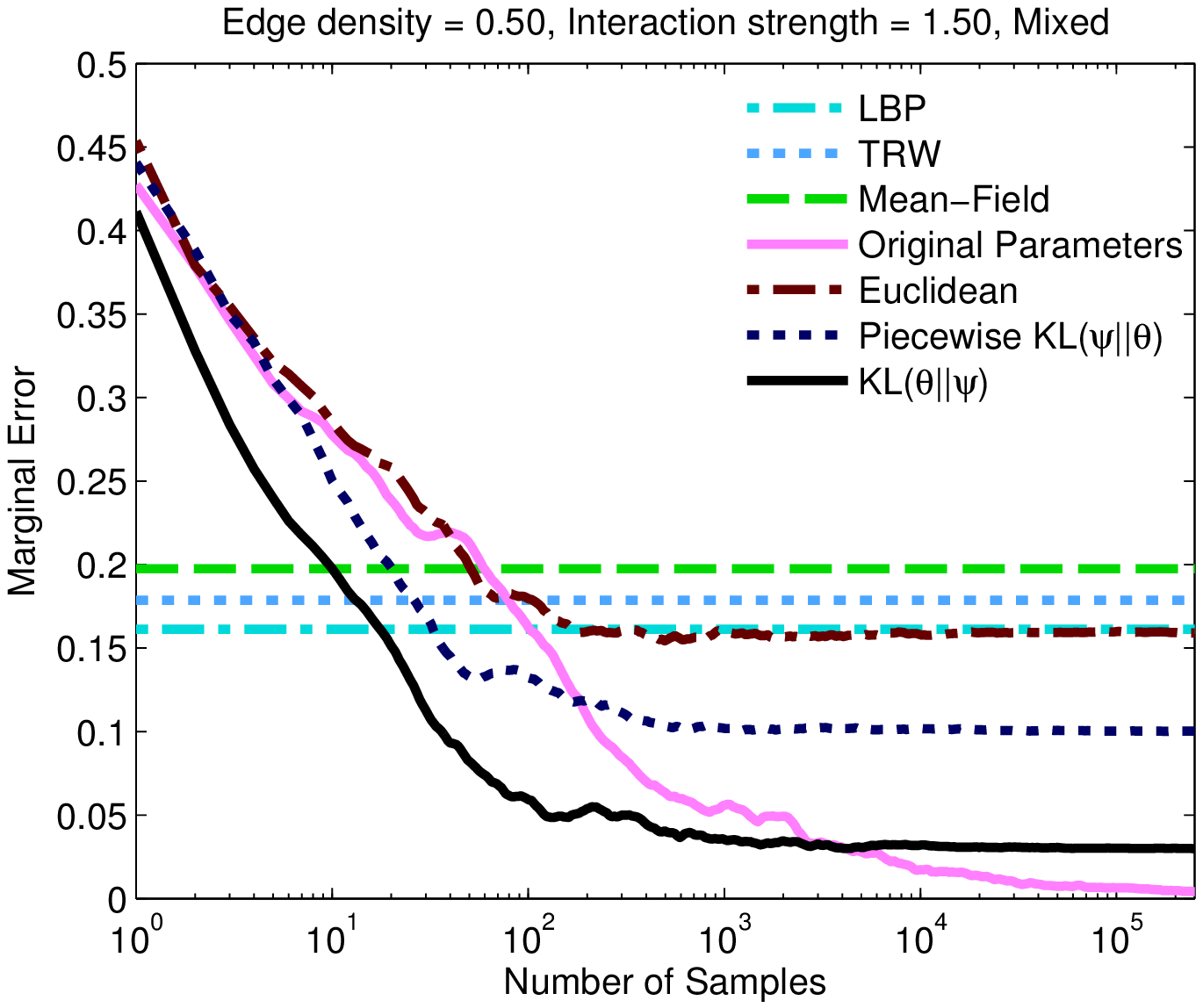}
\includegraphics[scale=0.4]{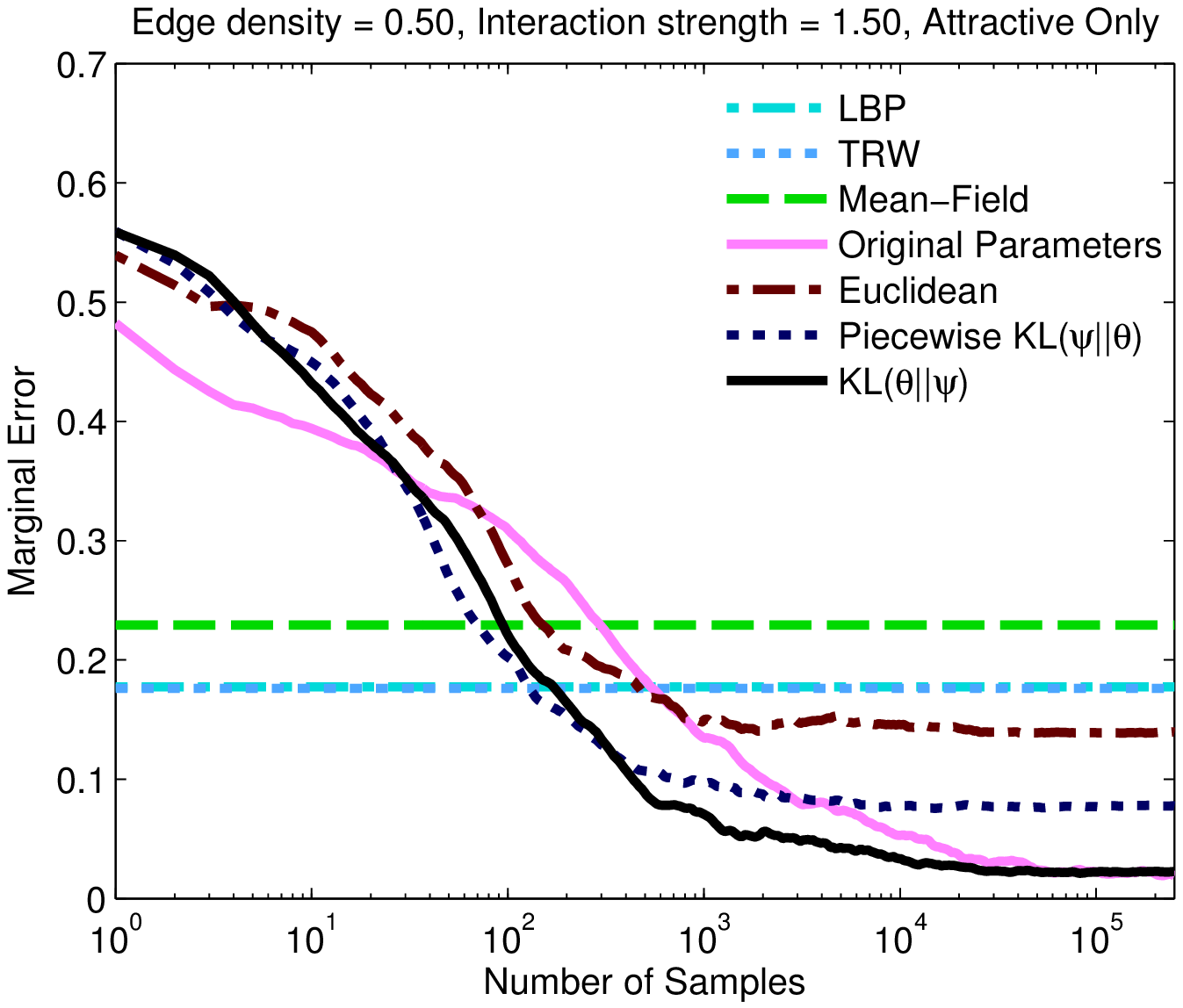}

\includegraphics[scale=0.4]{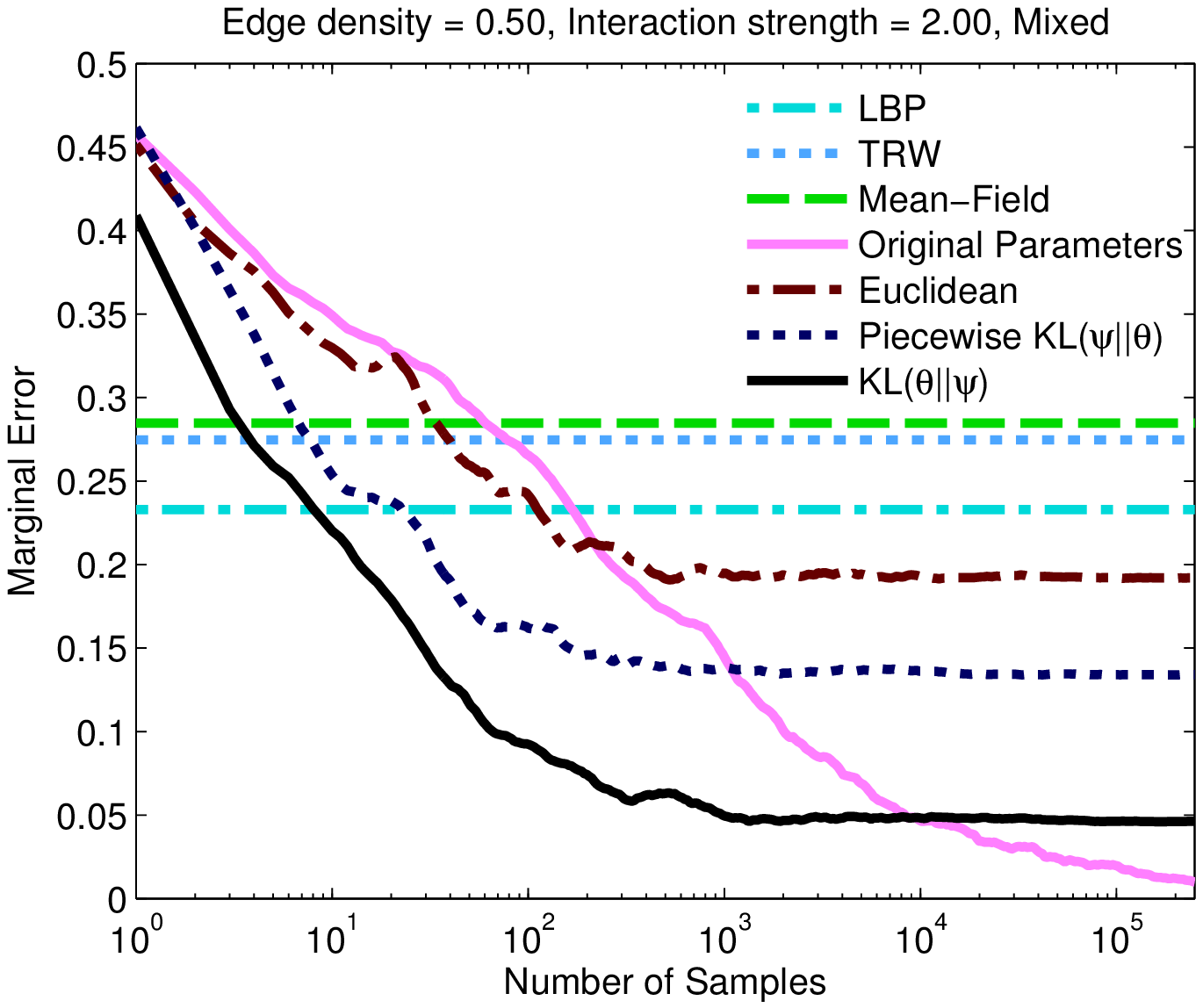}
\includegraphics[scale=0.4]{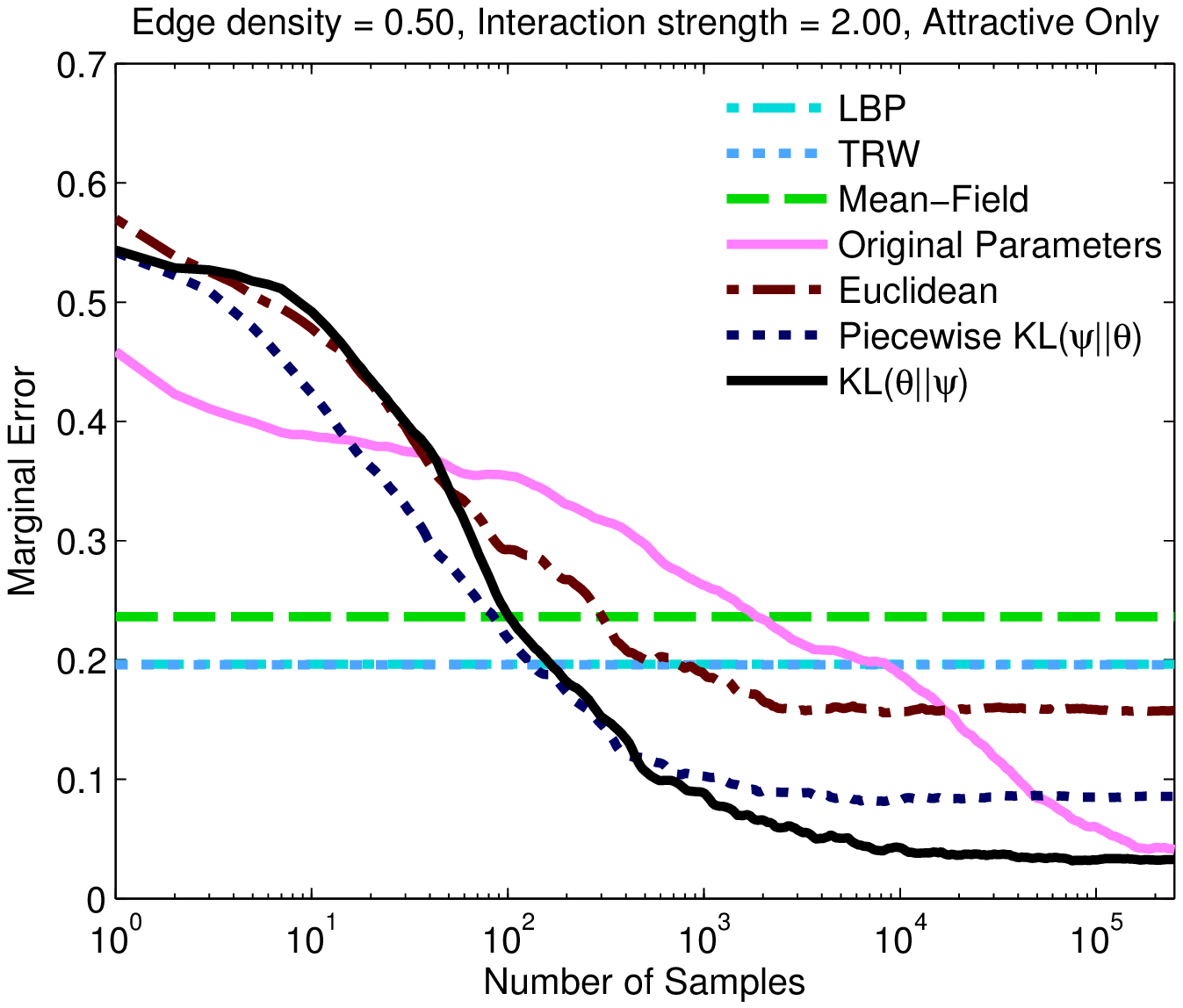}

\includegraphics[scale=0.4]{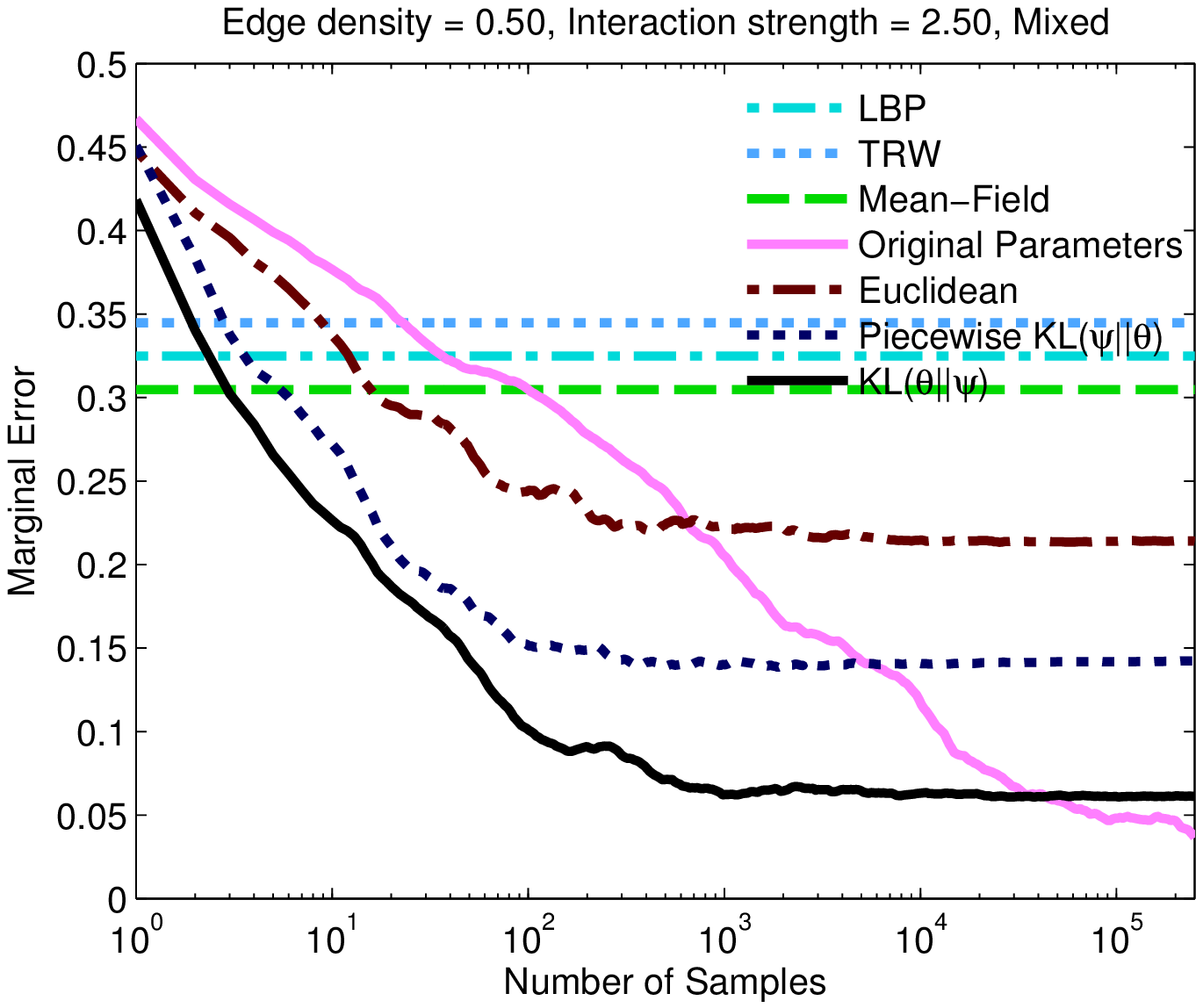}
\includegraphics[scale=0.4]{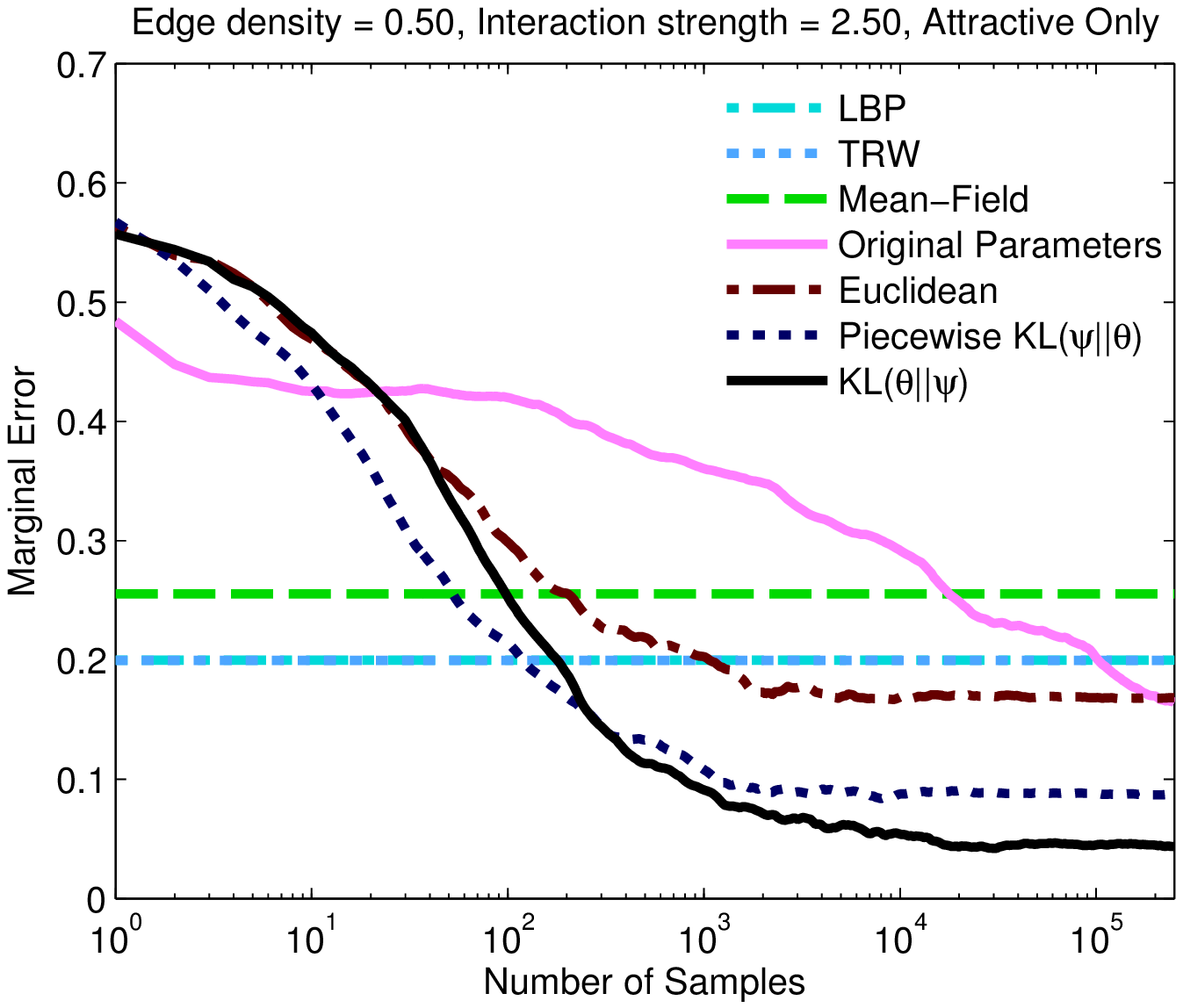}

\includegraphics[scale=0.4]{figures/rnd_time_0\lyxdot 50_3\lyxdot 00}
\includegraphics[scale=0.4]{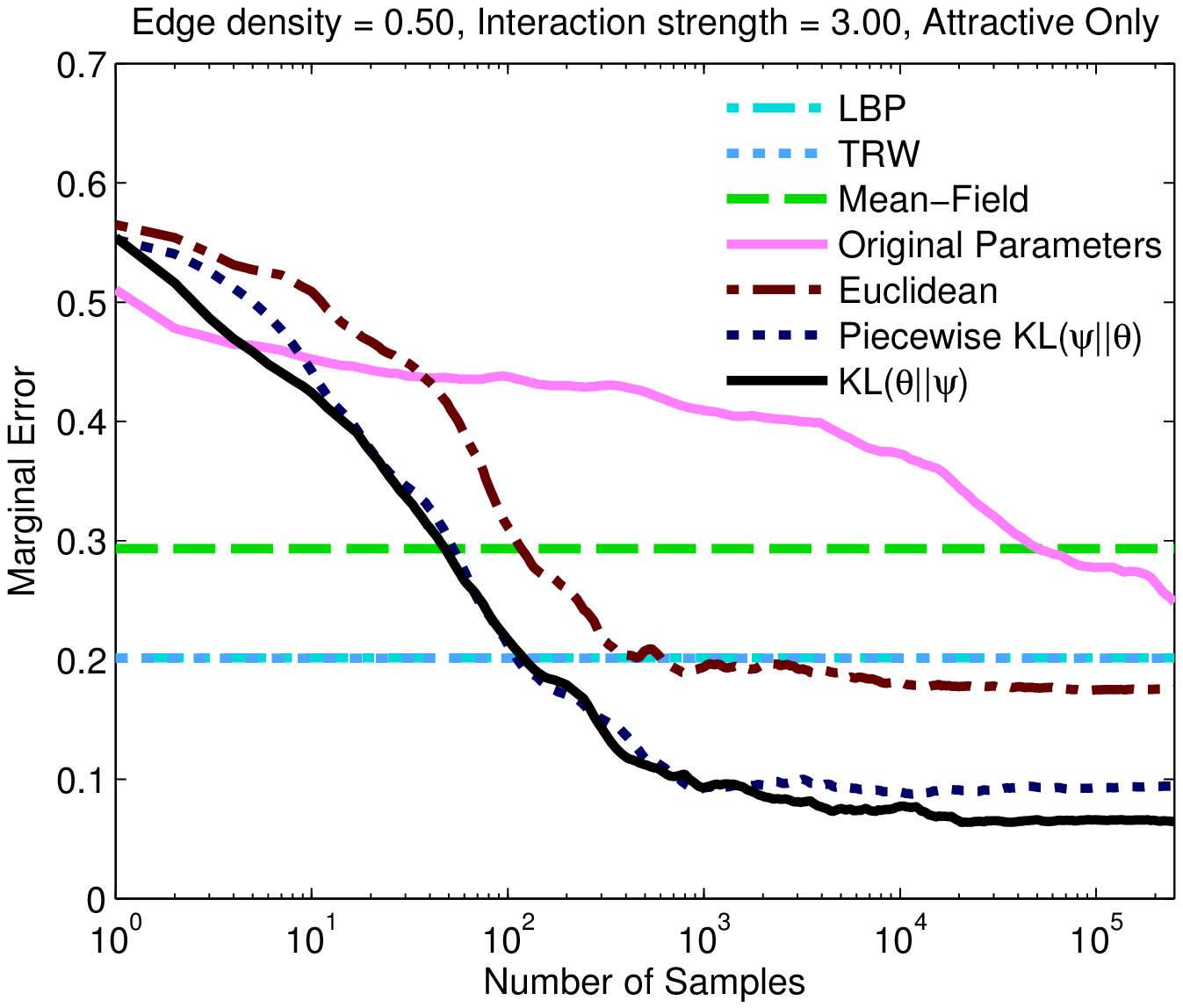}

\includegraphics[scale=0.4]{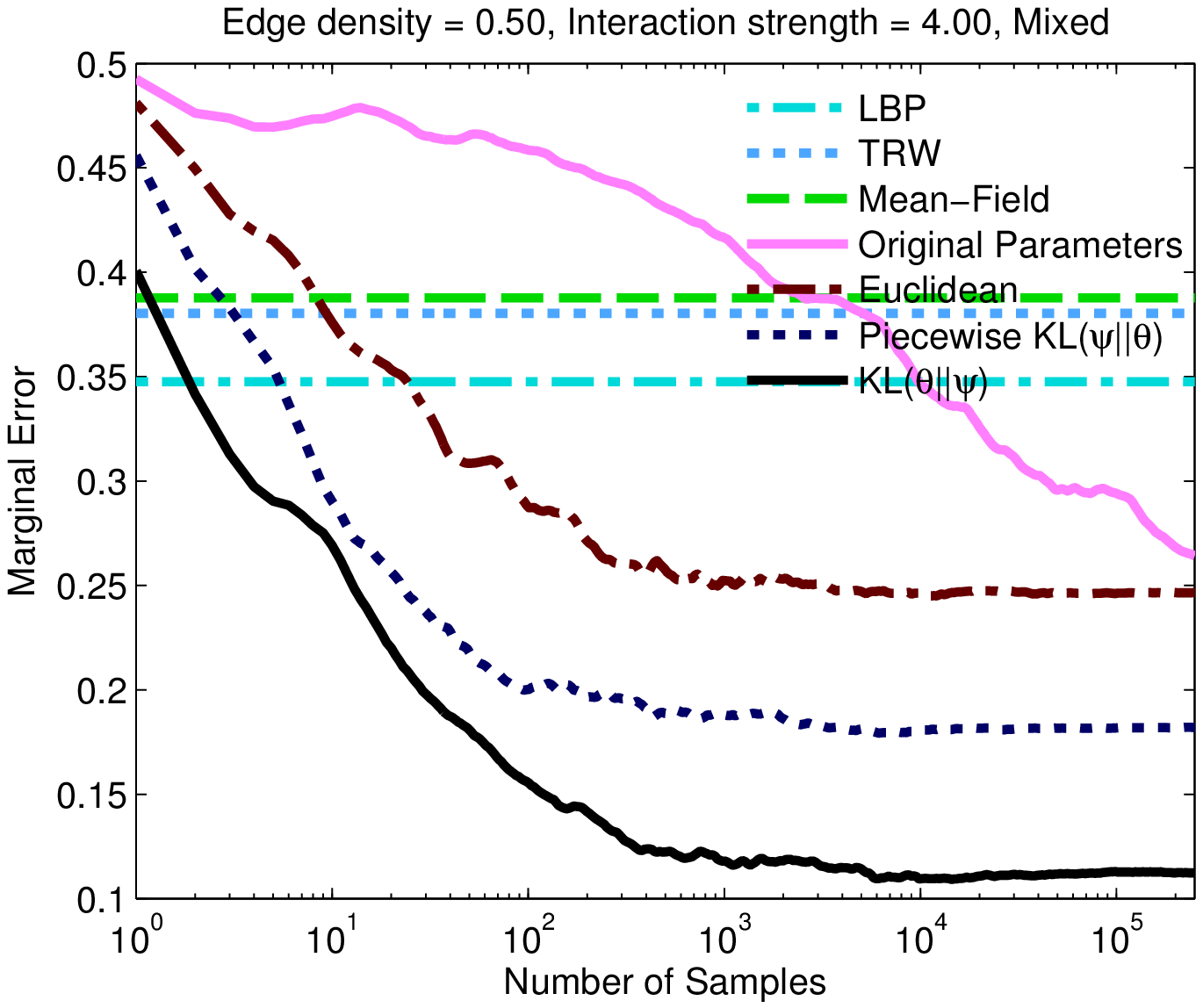}
\includegraphics[scale=0.4]{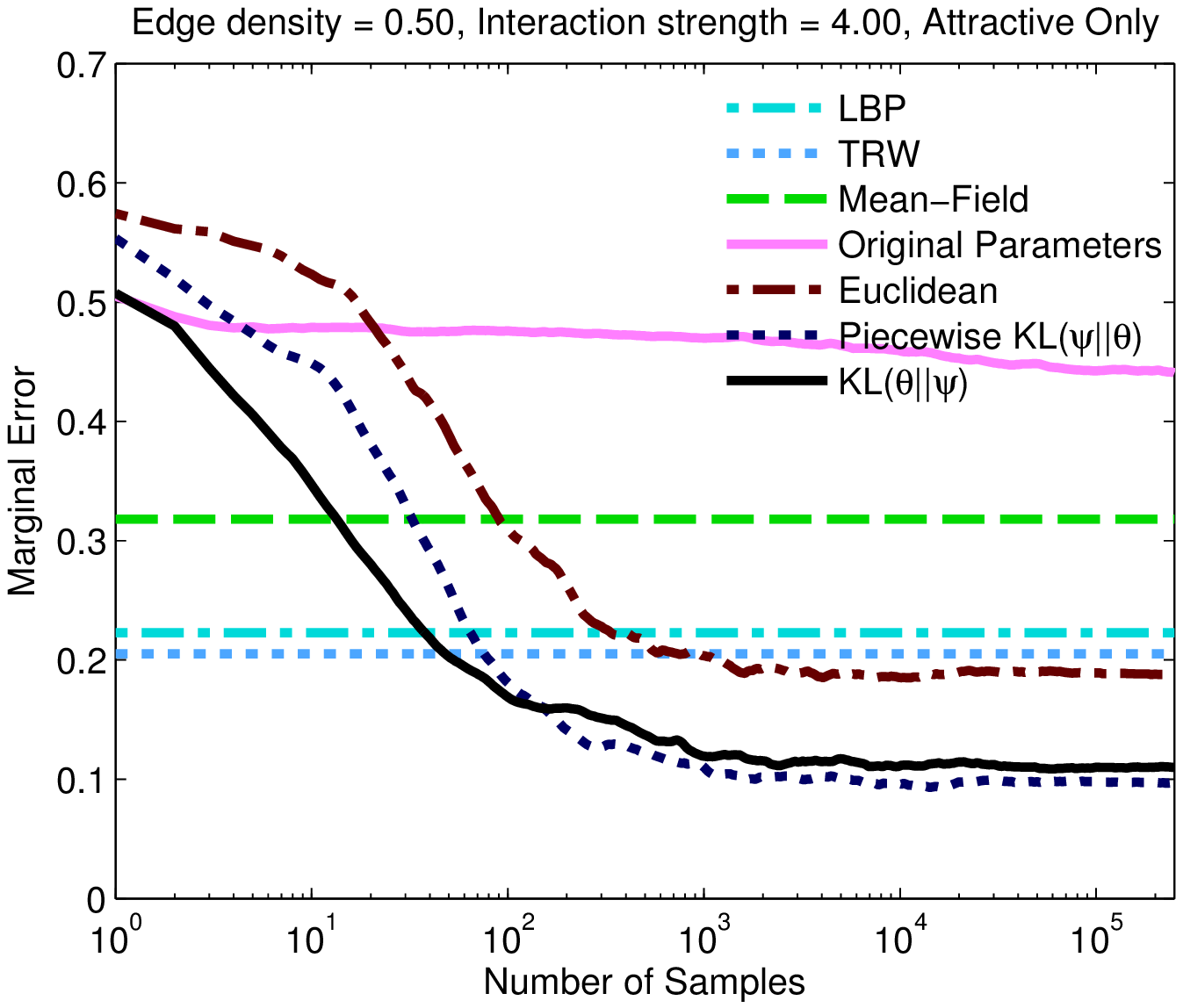}

\protect\caption{Marginal error v.s. number of samples for Ising models on random graphs
with edge density 0.5}
\end{figure}

\newpage{}
\begin{figure}
\centering \includegraphics[scale=0.4]{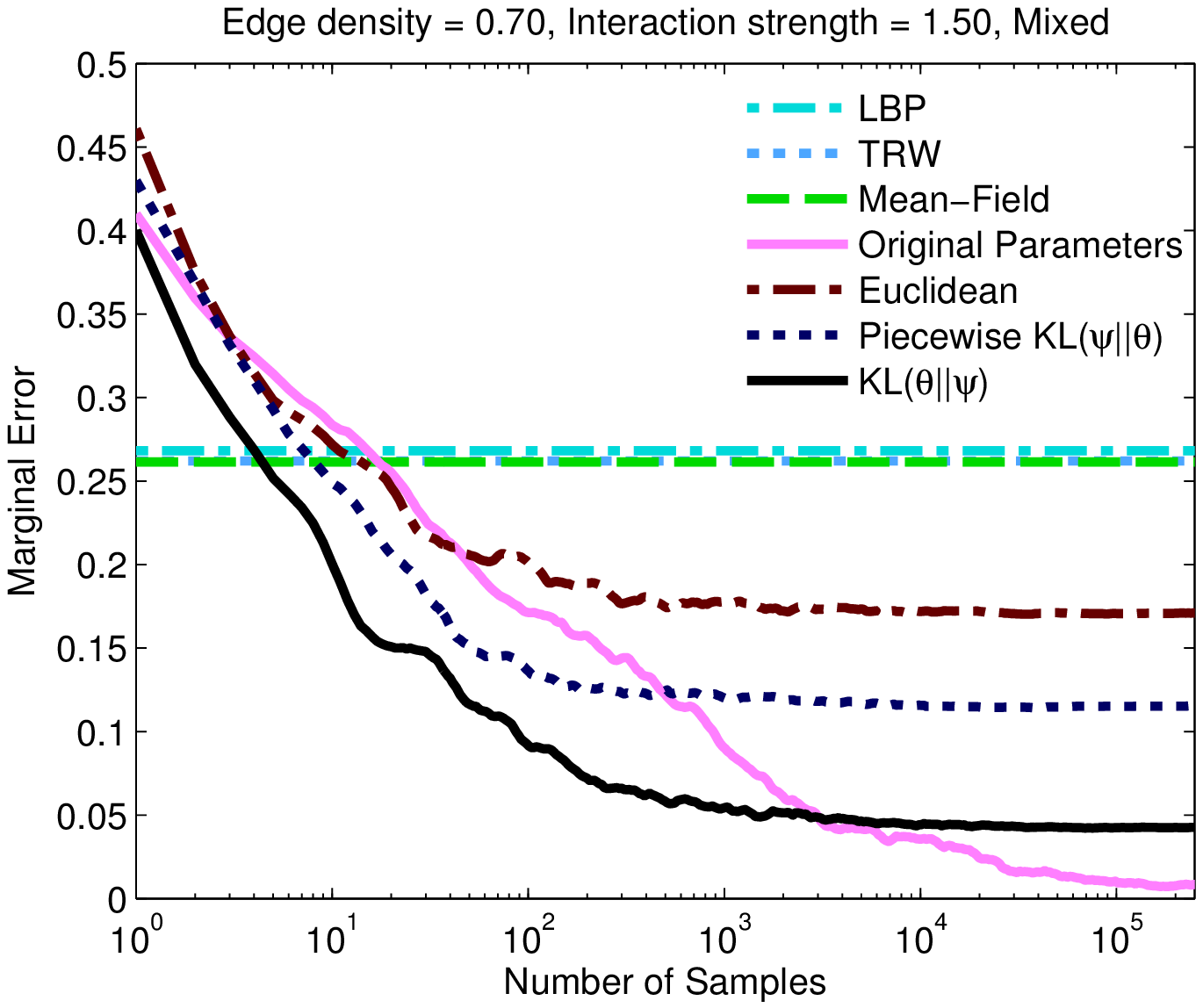}
\includegraphics[scale=0.4]{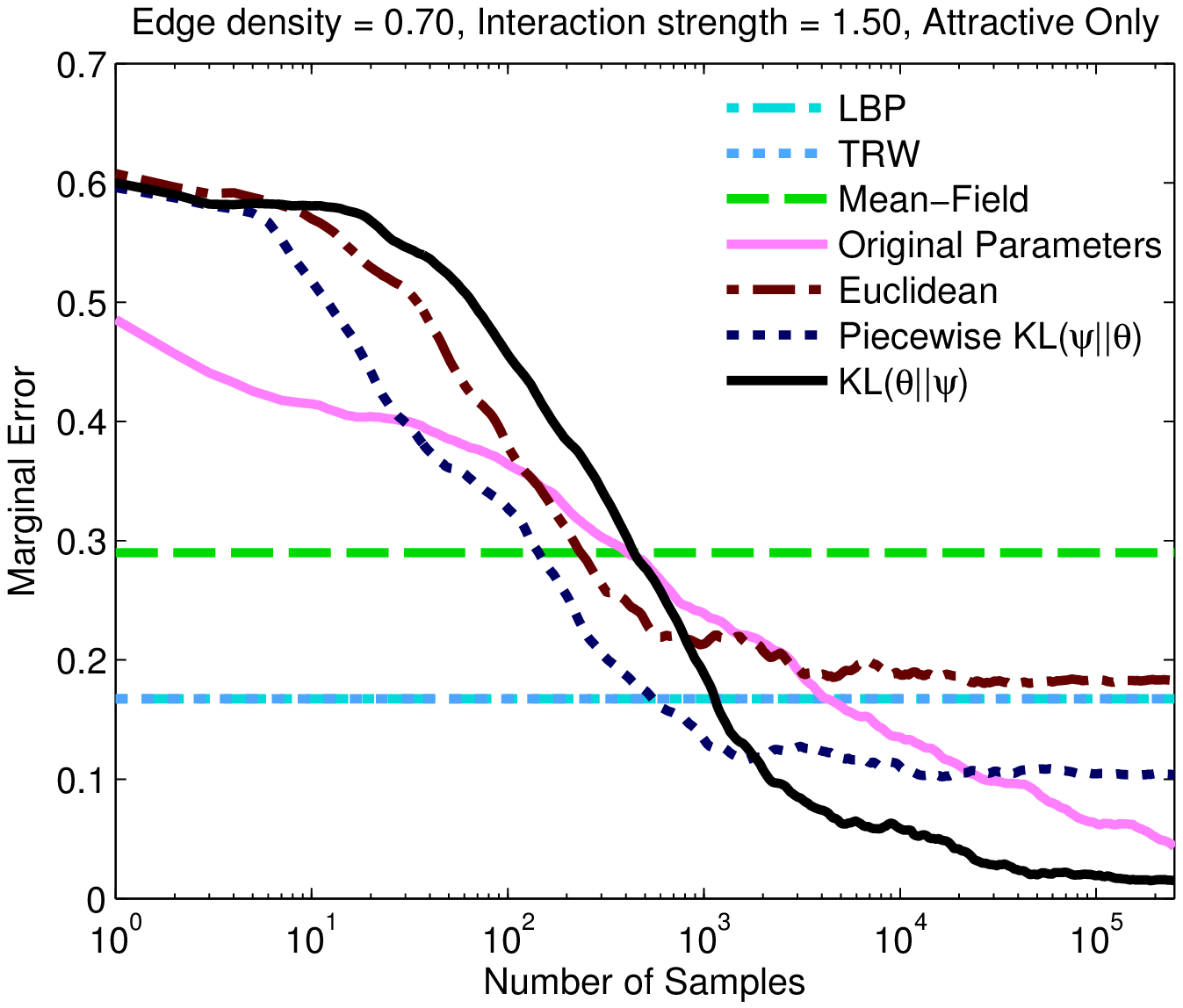}

\includegraphics[scale=0.4]{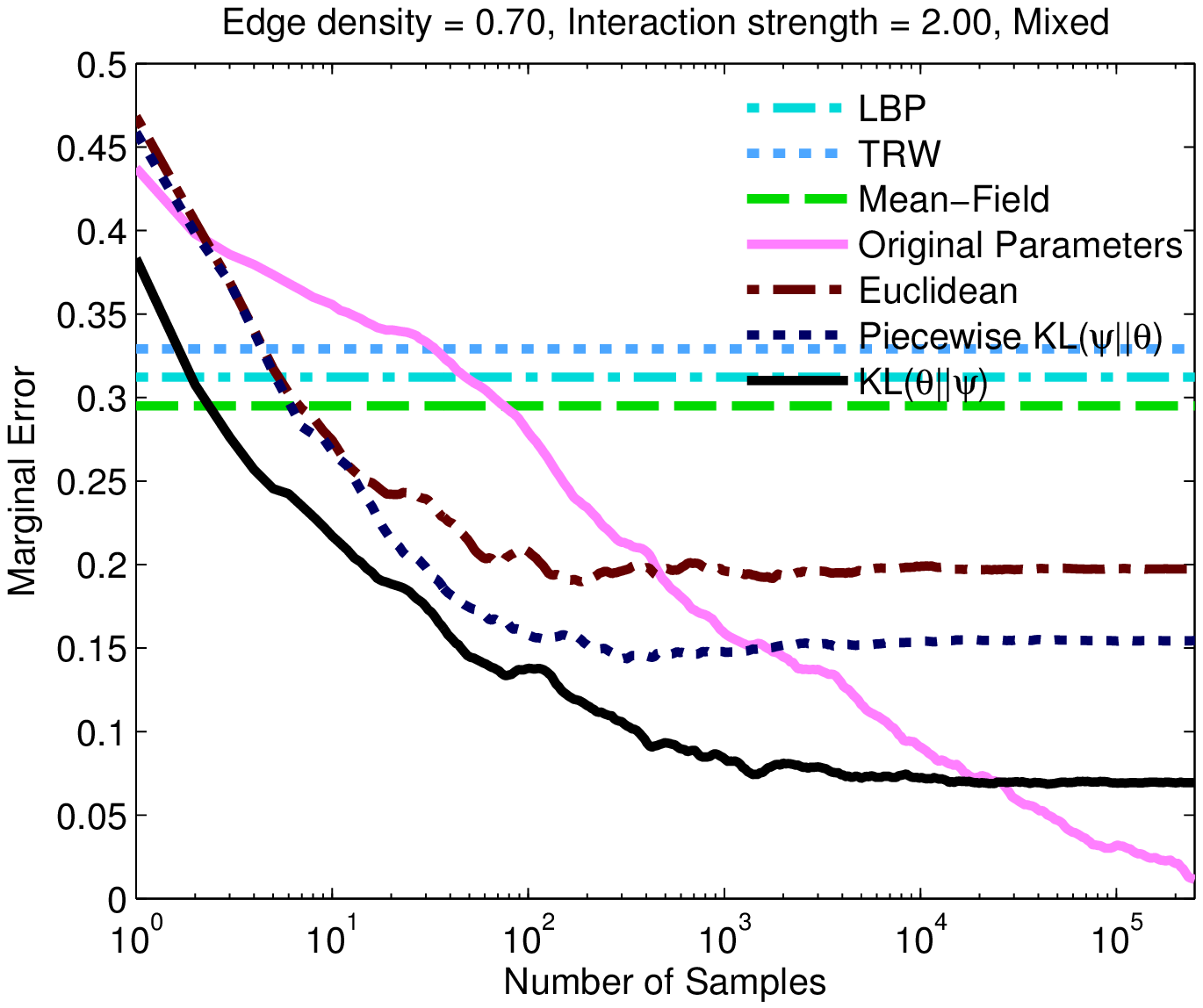}
\includegraphics[scale=0.4]{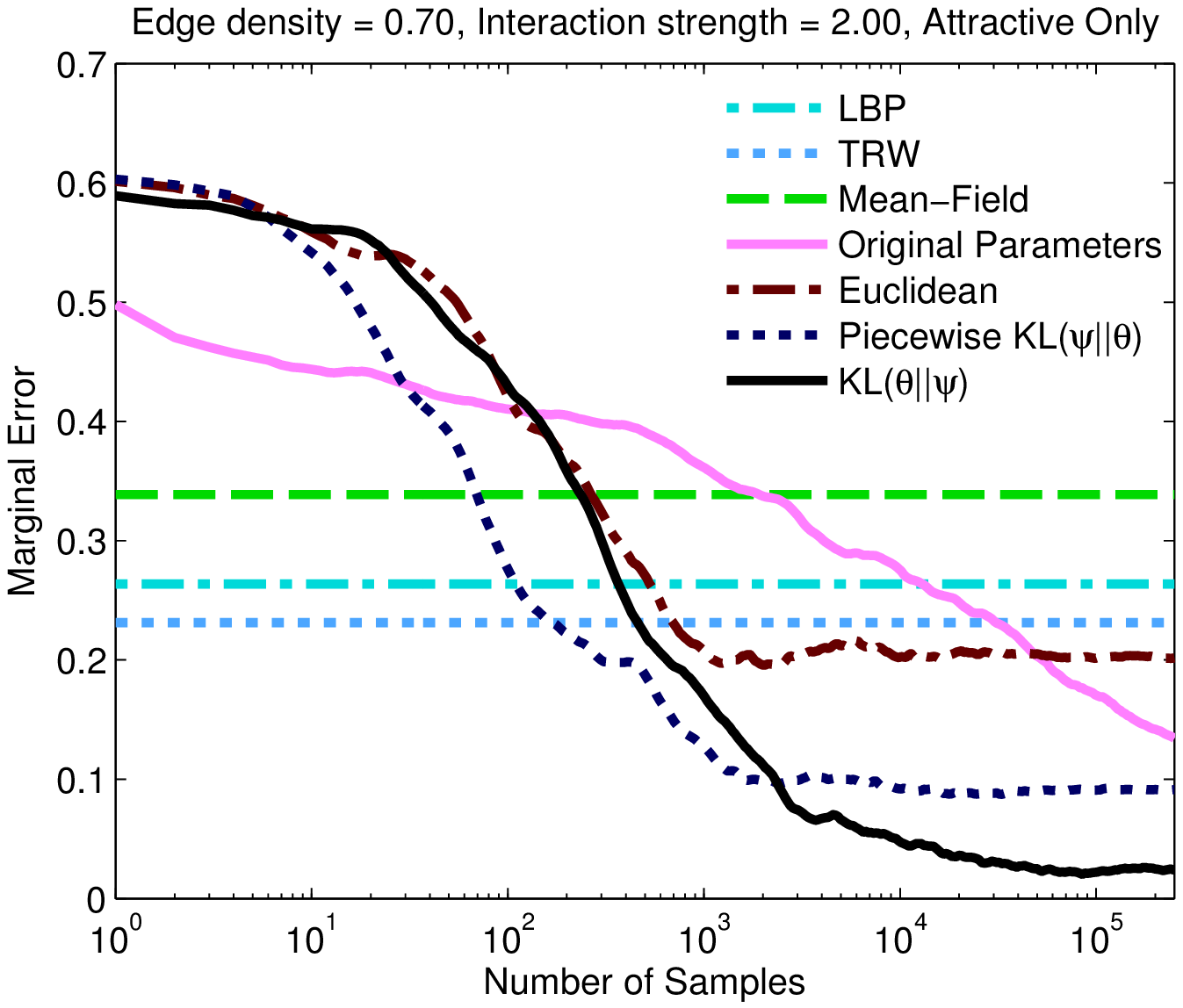}

\includegraphics[scale=0.4]{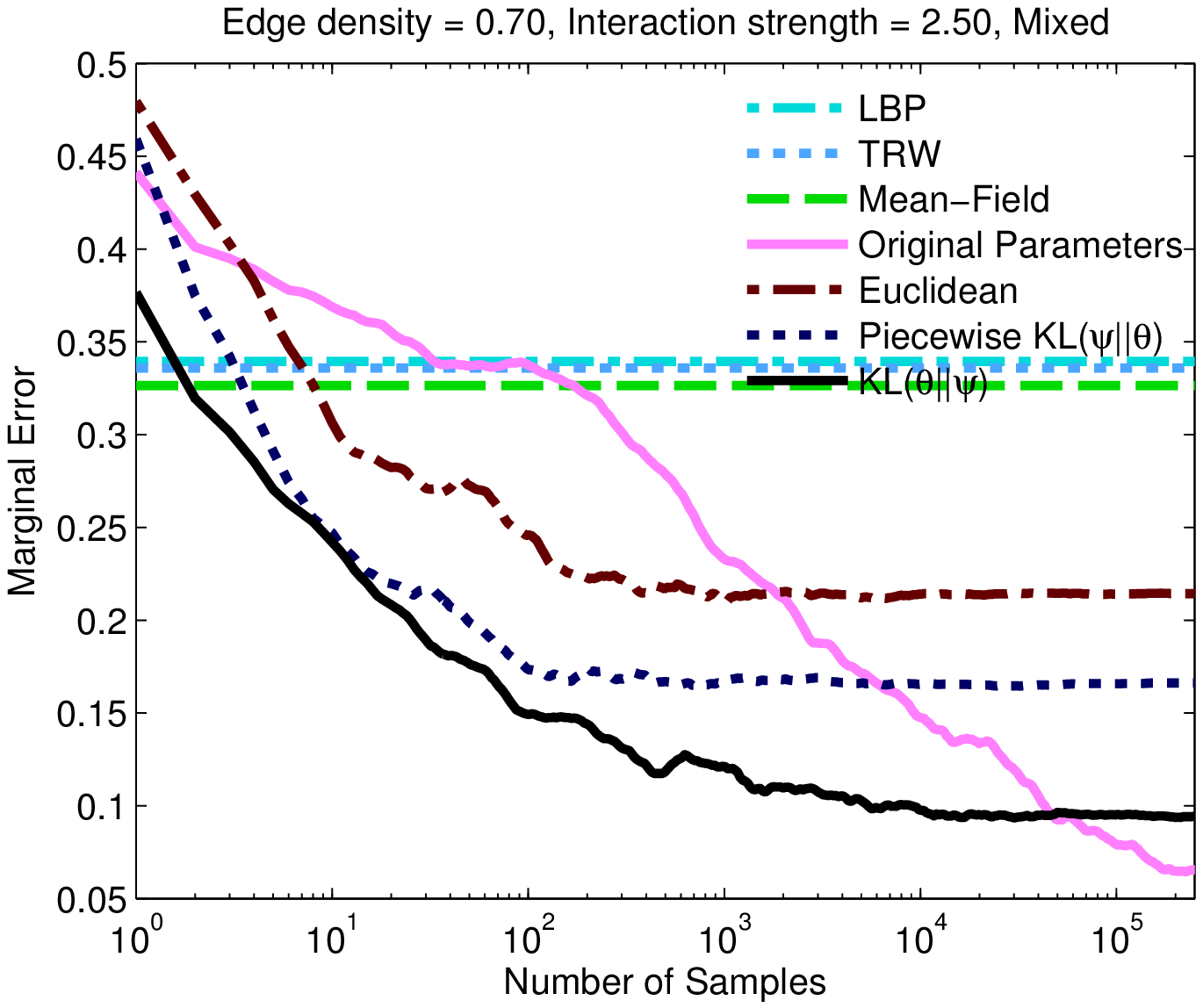}
\includegraphics[scale=0.4]{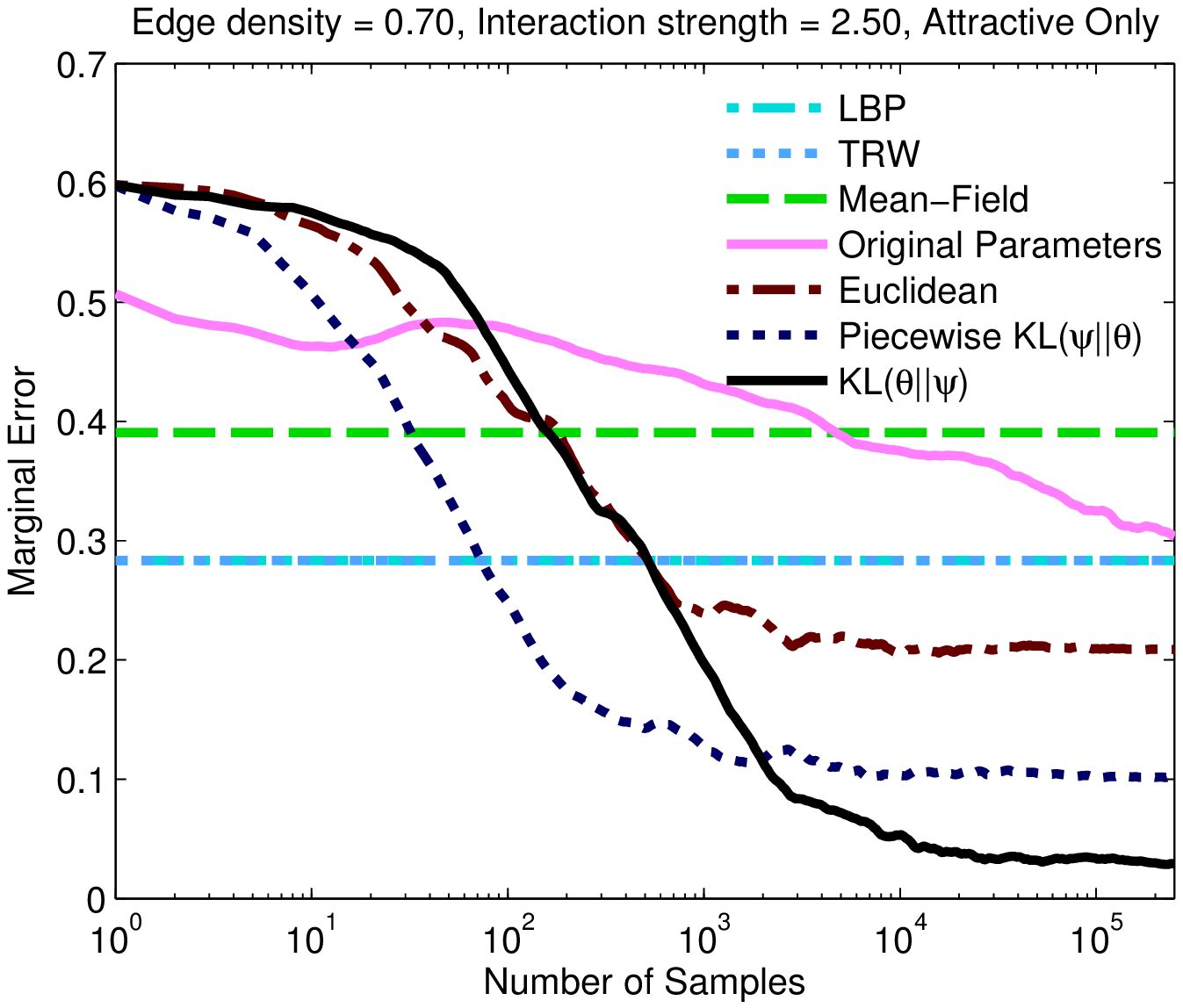}

\includegraphics[scale=0.4]{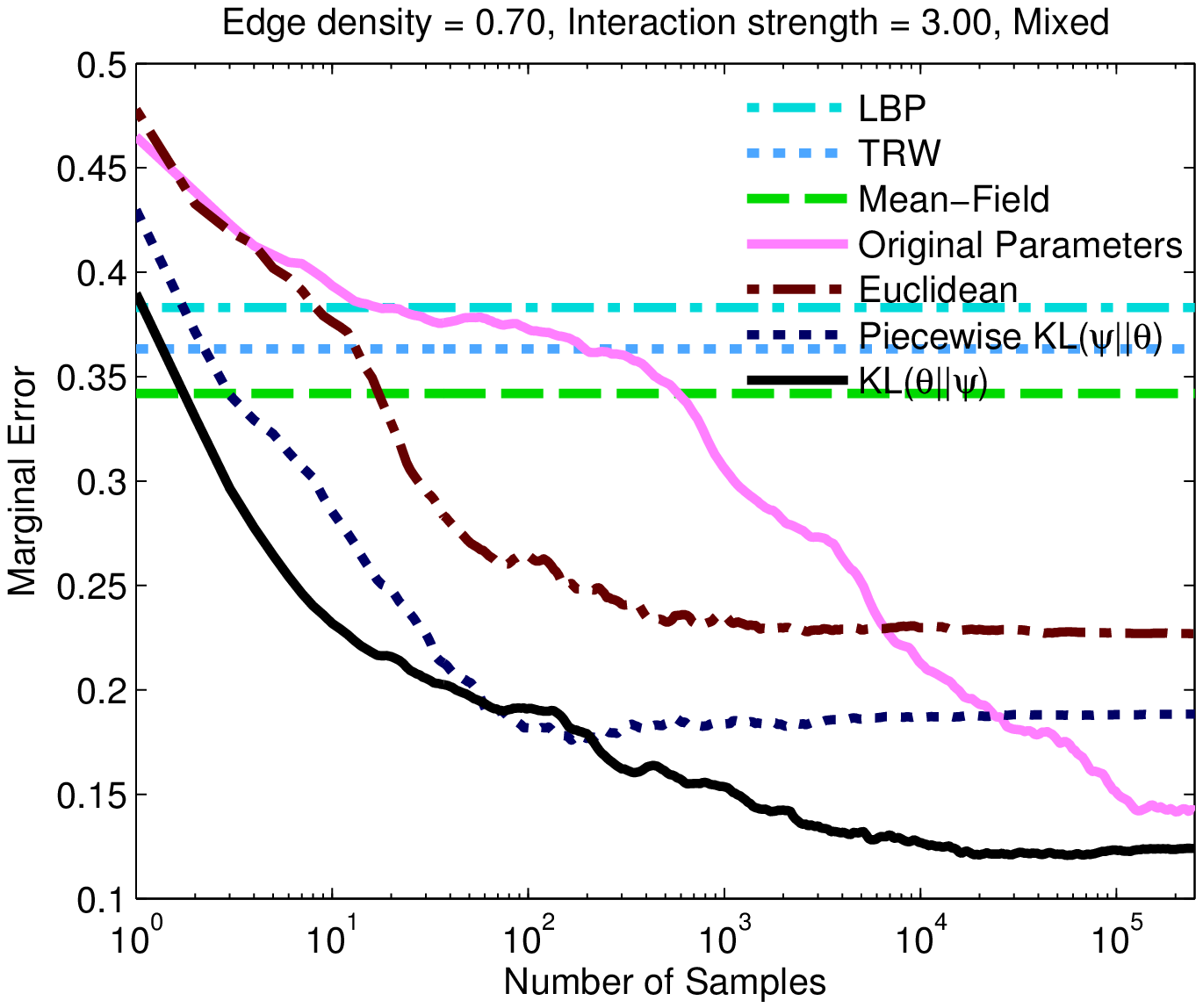}
\includegraphics[scale=0.4]{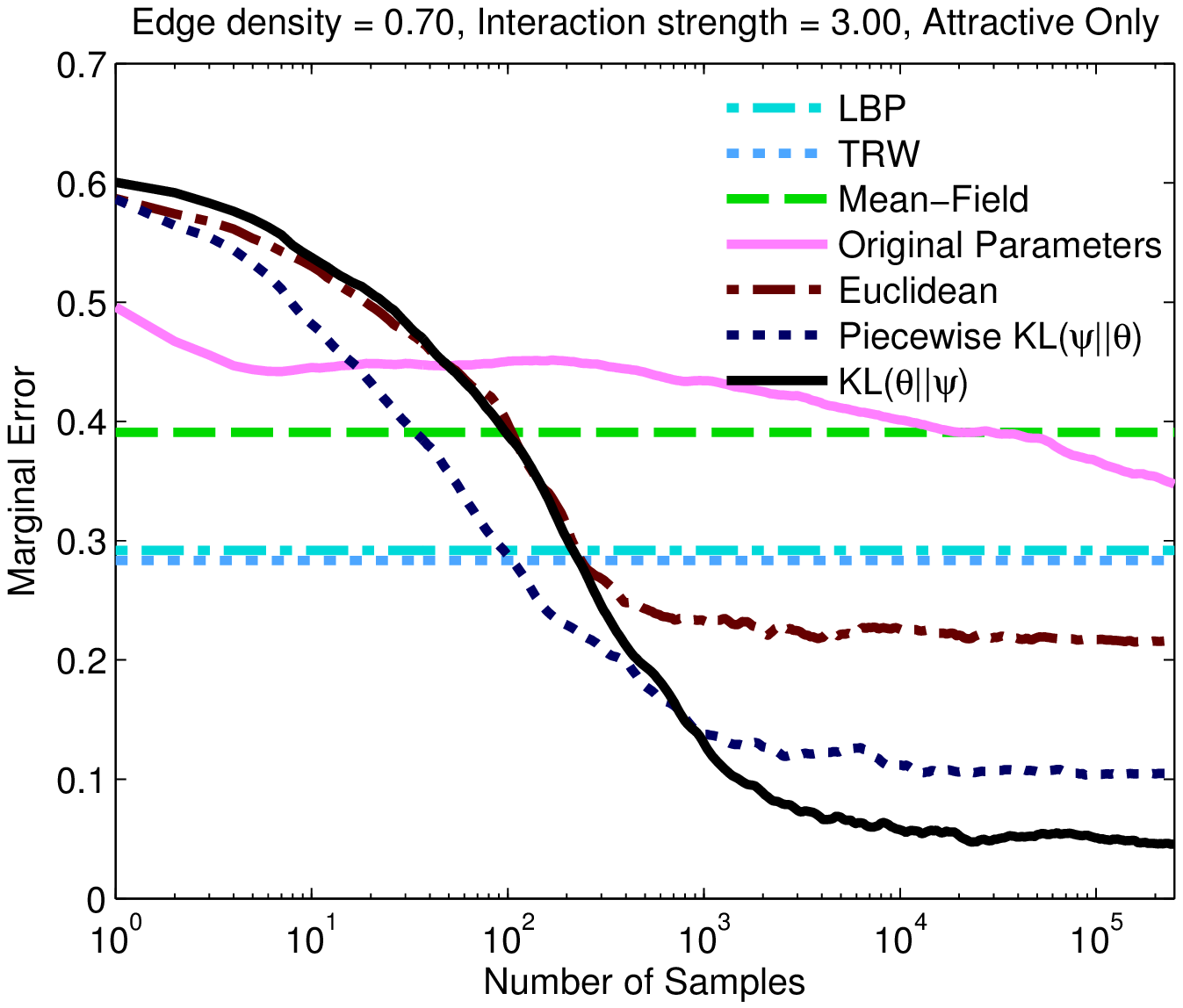}

\includegraphics[scale=0.4]{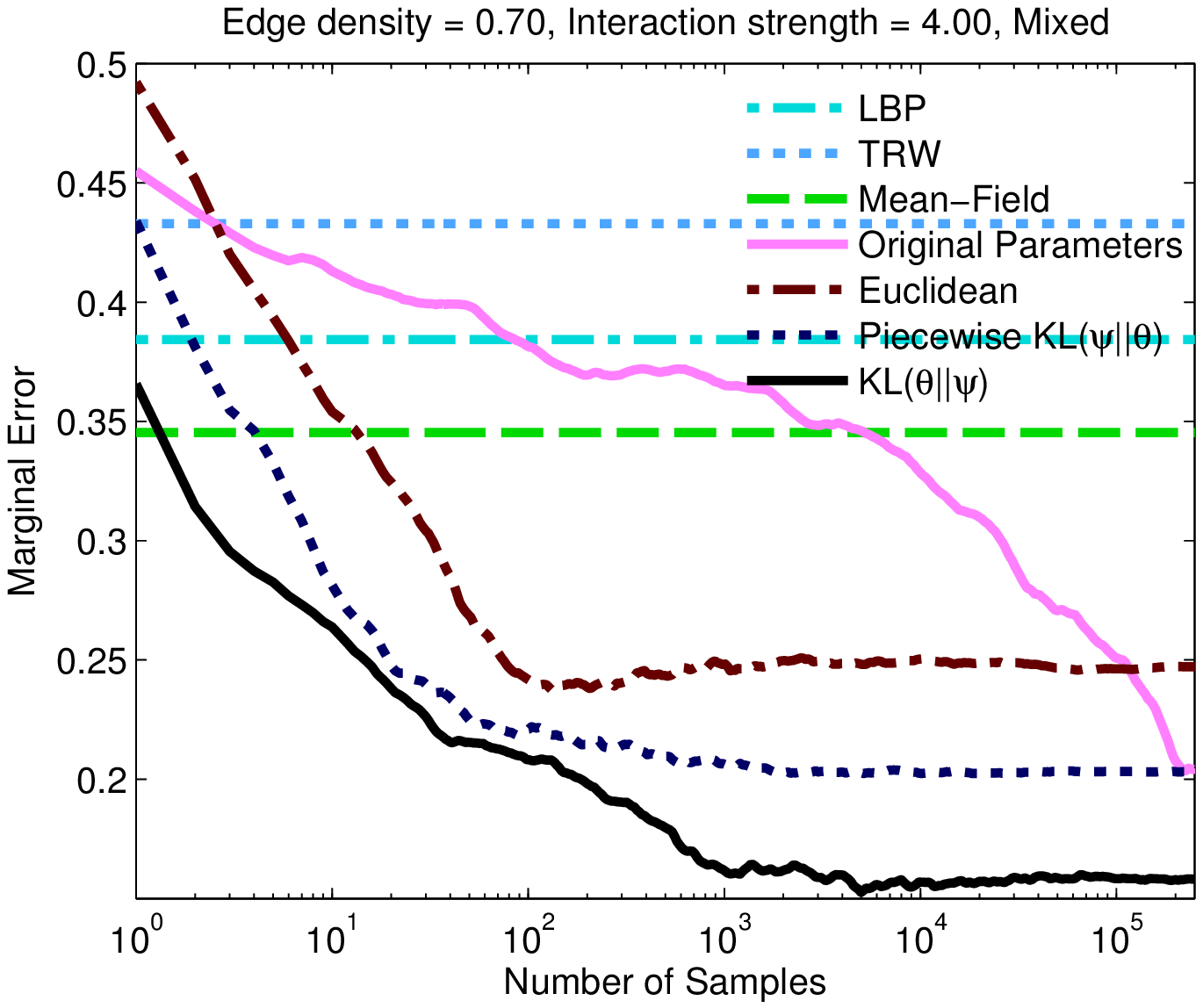}
\includegraphics[scale=0.4]{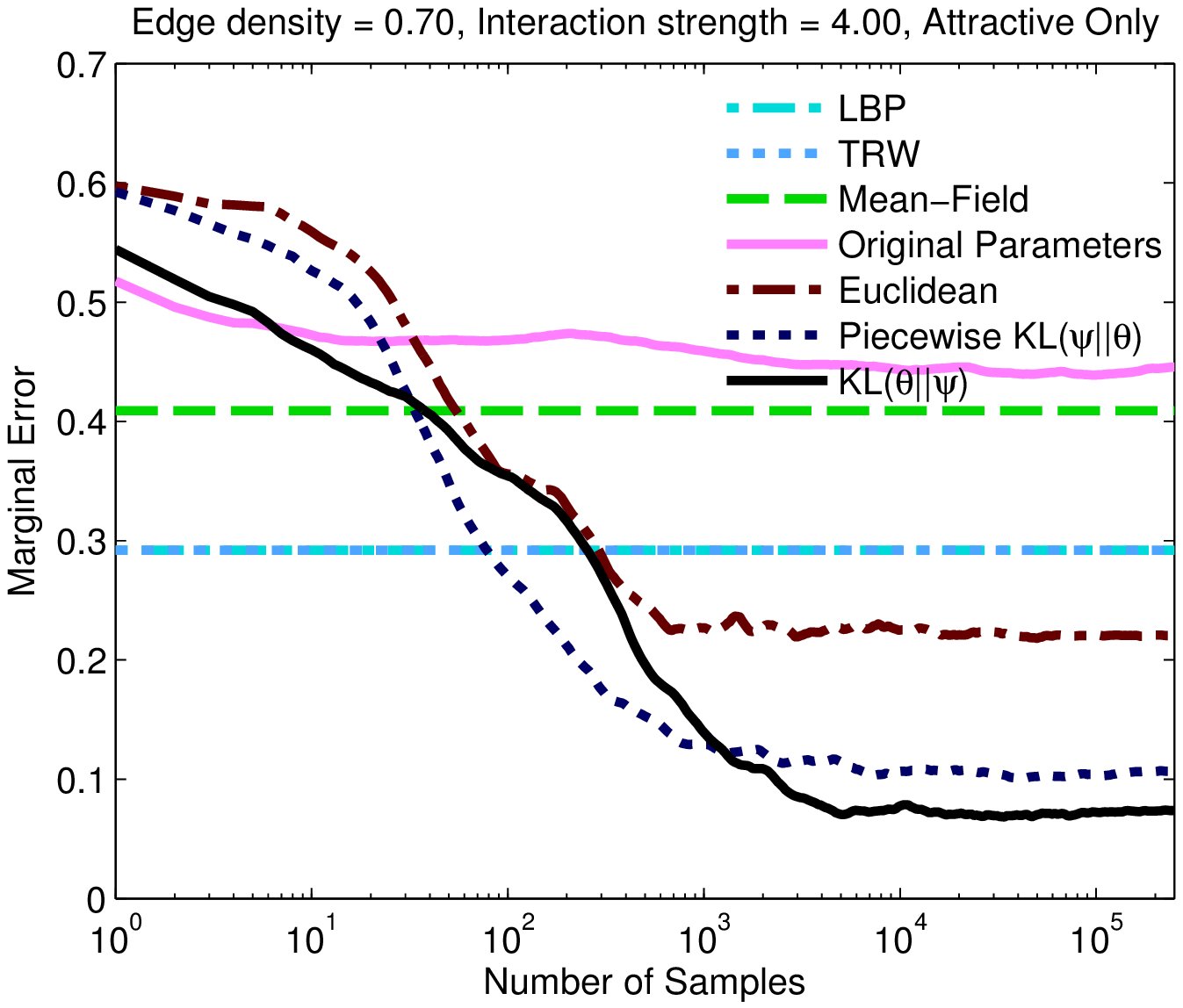}

\protect\caption{Marginal error v.s. number of samples for Ising models on random graphs
with edge density 0.7}
\end{figure}

\end{document}